\documentclass{article}
\usepackage{amsmath}
\usepackage{amssymb}
\usepackage{array}
\usepackage{bbm}
\usepackage{hyperref}
\usepackage{amsthm}
\usepackage{color}
\usepackage{mathrsfs}
\usepackage[]{algorithm2e}
\usepackage{xcolor}
\usepackage[title]{appendix}
\hypersetup{
    colorlinks,
    linkcolor={blue!80!black},
    citecolor={green!50!black},
}
\usepackage{accents}
\usepackage{apptools}
\usepackage{comment}
\usepackage[shortlabels]{enumitem}

\usepackage[mathscr]{euscript}

\DeclareSymbolFont{rsfs}{U}{rsfs}{m}{n}
\DeclareSymbolFontAlphabet{\mathscrsfs}{rsfs}

\AtAppendix{\counterwithin{lemma}{section}}

\usepackage{mathtools}

\def\calB{\mathcal{B}}

\def\calE{\mathcal{E}}
\def\calF{\mathcal{F}}

\def\calL{\mathcal{L}}

\def\calT{\mathcal{T}}

\def\calV{\mathcal{V}}

\def\bA{\mathbf{A}}
\def\bB{\mathbf{B}}

\def\bD{\mathbf{D}}

\def\bG{\mathbf{G}}

\def\bI{\mathbf{I}}

\def\bQ{\mathbf{Q}}

\def\bT{\mathbf{T}}
\def\bU{\mathbf{U}}
\def\bV{\mathbf{V}}
\def\bW{\mathbf{W}}
\def\bX{\mathbf{X}}
\def\bY{\mathbf{Y}}
\def\bZ{\mathbf{Z}}

\def\ba{\boldsymbol{a}}
\def\bb{\boldsymbol{b}}
\def\bc{\boldsymbol{c}}
\def\bd{\boldsymbol{d}}

\def\bg{\boldsymbol{g}}
\def\bh{\boldsymbol{h}}

\def\bq{\boldsymbol{q}}
\def\br{\boldsymbol{r}}

\def\bu{\boldsymbol{u}}
\def\bv{\boldsymbol{v}}
\def\bw{\boldsymbol{w}}
\def\bx{\boldsymbol{x}}
\def\by{\boldsymbol{y}}
\def\bz{\boldsymbol{z}}

\def\bA{\boldsymbol{A}}
\def\bB{\boldsymbol{B}}

\def\bD{\boldsymbol{D}}

\def\bG{\boldsymbol{G}}

\def\bI{\boldsymbol{I}}

\def\bQ{\boldsymbol{Q}}

\def\bT{\boldsymbol{T}}
\def\bU{\boldsymbol{U}}
\def\bV{\boldsymbol{V}}
\def\bW{\boldsymbol{W}}
\def\bX{\boldsymbol{X}}
\def\bY{\boldsymbol{Y}}
\def\bZ{\boldsymbol{Z}}
\def\hbtheta{\hat{\boldsymbol \theta}}
\def\normal{{\mathsf{N}}}

\def\btheta{\boldsymbol{\theta}}
\def\bvartheta{\boldsymbol{\vartheta}}
\def\bSigma{\boldsymbol{\Sigma}}

\def\blambda{\boldsymbol{\lambda}}
\def\bxi{\boldsymbol{\xi}}
\def\bgamma{\boldsymbol{\gamma}}

\def\bzero{\boldsymbol{0}}
\def\balpha{\boldsymbol{\alpha}}
\def\bzeta{\boldsymbol{\zeta}}
\def\bLambda{\boldsymbol{\Lambda}}

\def\bTheta{\boldsymbol{\Theta}}

\def\hbQ{\hat{\boldsymbol Q}}
\def\hq{\hat{q}}

\def\id{{\boldsymbol I}}

\newcommand\tB{\tilde B}

\newcommand\tbA{\tilde{\boldsymbol{A}}}
\newcommand\tbB{\tilde{\boldsymbol{B}}}

\newcommand\ta{\tilde a}
\newcommand\tb{\tilde b}
\newcommand\tr{\tilde r}
\newcommand\tq{\tilde q}
\newcommand\tba{\tilde {\boldsymbol{a}}}
\newcommand\tbb{\tilde {\boldsymbol{b}}}

\newcommand\tbq{\tilde {\boldsymbol{q}}}
\newcommand\tf{\tilde f}
\newcommand\tg{\tilde g}

\newcommand\tz{\tilde z}
\newcommand\tbw{\tilde {\boldsymbol{w}}}

\renewcommand{\P}{\mathbb{P}}
\newcommand{\E}{\mathbb{E}}

\newcommand{\eps}{\varepsilon}
\newcommand{\Var}{\operatorname{Var}}

\newcommand{\sign}{\operatorname{sign}}

\newcommand{\Tr}{\operatorname{Tr}}

\newcommand{\reals}{\mathbb{R}}
\newcommand{\indic}[1]{\mathbf{1}\{#1\}}

\def\hell{\hat{\ell}}
\def\<{\langle}
\def\>{\rangle}
\def\hby{\hat{\boldsymbol{y}}}
\def\Prox{\textsf{Prox}}
\def\sT{{\mathsf{T} }}

\newcommand{\mah}[1]{\tilde{a}_{v\rightarrow f}^{#1}}

\newcommand{\mbh}[1]{\tilde{b}_{f\rightarrow v}^{#1}}

\def\obtheta{\overline{\boldsymbol \theta}}
\def\obv{\overline{\boldsymbol v}}
\def\tbX{\tilde{\boldsymbol X}}
\def\tbx{\tilde{\boldsymbol x}}
\def\tbZ{\tilde{\boldsymbol Z}}

\def\Ball{B}
\def\talpha{\tilde{\alpha}}
\def\tg{\tilde{g}}
\def\tY{\tilde{Y}}
\def\tbg{\tilde{\boldsymbol g}}
\def\cuT{\mathscrsfs{T}}
\def\oq{\overline{q}}
\def\de{{\mathrm{d}}}
\def\tbW{\tilde{\boldsymbol W}}
\def\tU{\tilde U}

\def\cuP{\mathscrsfs{P}}
\def\mmse{\textsf{mmse}}
\def\htheta{\hat{\theta}}
\def\ttheta{\tilde{\theta}}
\def\tv{\tilde{v}}

\def\ttau{\tilde\tau}
\def\tV{\tilde V}
\def\tbtheta{\tilde{\boldsymbol \theta}}
\def\tTheta{\tilde{\Theta}}

\def\sp{\mbox{\tiny\textrm{sp}}}
\def\tbx{\tilde{\boldsymbol x}}

\def\tbv{\tilde{\boldsymbol{v}}}
\def\tby{\tilde{\boldsymbol{y}}}
\def\tbu{\tilde{\boldsymbol{u}}}
\def\tbw{\tilde{\boldsymbol{w}}}
\def\tu{\tilde{\boldsymbol{u}}}
\def\th{\tilde{h}}

\newlength{\dhatheight}

\newtheorem{theorem}{Theorem}

\newtheorem{lemma}{Lemma}
\newtheorem{corollary}{Corollary}

\allowdisplaybreaks
\usepackage[top=1in,bottom=1in,left=1in,right=1in]{geometry}
\usepackage[utf8]{inputenc}

\title{The estimation error of general first order methods}
\author{
	Michael Celentano\thanks{Department of Statistics, Stanford University}
	\and 
	Andrea Montanari\footnotemark[1]
	\thanks{Department of Electrical Engineering, Stanford University} 
	\and 
	Yuchen Wu\footnotemark[1]
}
\date{}

\begin{document}

\maketitle

\begin{abstract}%
Modern large-scale statistical models require to estimate thousands to millions of parameters. This is often accomplished by iterative algorithms such as gradient descent, projected gradient descent or their accelerated versions. What are the fundamental limits to these approaches? This question is well understood from an optimization viewpoint when the underlying objective is convex. Work in this area characterizes the gap to global optimality as a function of the number of iterations. However, these results have only indirect implications in terms of the gap to \emph{statistical} optimality.
    
Here we consider two families of high-dimensional estimation problems: high-dimensional regression and low-rank matrix estimation, and introduce a class of `general first order methods' that aim at efficiently estimating the underlying parameters. This class of algorithms is broad enough to include classical first order optimization (for convex and non-convex objectives), but also other types of algorithms. Under a random design assumption,  we derive lower bounds on the estimation error that hold in the high-dimensional
asymptotics in which both the number of observations and the number of parameters diverge. 
These lower bounds are optimal in the sense that there exist algorithms whose estimation  error matches the lower bounds up to asymptotically negligible
terms. We illustrate our general results through applications to sparse phase retrieval and sparse principal component analysis.
\end{abstract}

\section{Introduction}

High-dimensional statistical estimation problems are often addressed by constructing a suitable  data-dependent cost function
$\calL(\bvartheta)$, which encodes the statistician's knowledge of the problem. 
This cost is then minimized using  an algorithm which scales well to large dimension.
The most popular algorithms for high-dimensional statistical applications are first order methods,
i.e., algorithms that query the cost $\calL(\bvartheta)$ by computing its gradient (or a subgradient)
at a sequence of points $\btheta^1$,\dots $\btheta^t$.
Examples include (projected) gradient descent, mirror descent,
and accelerated gradient descent. 

This raises a fundamental question: \emph{What is the minimal statistical error achieved by first order methods?} In particular,
we would like to understand in which cases these methods are significantly sub-optimal (in terms of estimation)
with respect to statistically optimal but potentially intractable estimators, and what is the optimal tradeoff between number of iterations and estimation error.

These questions are relatively well understood only from the point of view of convex optimization, namely if estimation is performed by minimizing a convex cost function $\calL(\bvartheta)$, see  e.g. \cite{Dantzig,bickel2009simultaneous}. 
The seminal work of Nemirovsy and Yudin \cite{nemirovsky1983problem} characterizes the minimum gap to global optimality  $\calL(\btheta^t)-\min_{\bvartheta}\calL(\bvartheta)$,  where $\btheta^t$ is the algorithm's output $\btheta^t$ after $t$ iterations
(i.e., after $t$ gradient evaluations).
For instance, if $\calL(\btheta)$ is a smooth convex function, there exists a first order algorithm  which achieves $\calL(\btheta^t)\le \min_{\bvartheta}\calL(\bvartheta)+ O(t^{-2})$.
At the same time, no algorithm can be guaranteed to achieve a better convergence rate over all functions in this class.

In contrast, if the cost $\calL(\bvartheta)$ is nonconvex, there cannot be 
general guarantees of global optimality.
Substantial effort has been devoted to showing that --under suitable assumptions about the data distribution--
certain nonconvex costs $\calL(\btheta)$ can be minimized efficiently, e.g. by gradient descent
\cite{keshavan2010matrix,loh2011high,chen2015solving}. 
This line of work resulted in upper bounds on the estimation error of first order methods.
Unlike in the convex case, worst case lower bounds are typically overly pessimistic since non-convex optimization is NP-hard. Our work aims at developing 
precise average-case lower bounds for a restricted class of algorithms,
which are applicable both to convex and nonconvex problems.

We are particularly interested in problems that exhibit an information-computation gap: we know
that the optimal statistical estimator has high accuracy, but existing upper bounds on first order methods are substantially sub-optimal (see examples below).
Is this a limitation of our analysis, of the specific algorithm under consideration, or of first order algorithms in general?
The main result of this paper is a tight asymptotic characterization of the minimum estimation
error achieved by first order algorithms for two families of problems. This characterization can be
used --in particular-- to delineate information-computation gaps.

Our results are novel even in the case of a convex cost function $\calL(\bvartheta)$, for two reasons.
First, classical theory \cite{nesterov2018lectures} lower bounds 
the objective value $\calL(\btheta^t)-\min_{\bvartheta}\calL(\bvartheta)$ after $t$ iterations. This  has  only indirect  implications on estimation error,
e.g., $\|\btheta^t-\btheta_*\|_2$ (here $\btheta_*$ is the true value of
the parameters, not the minimizer of the cost $\calL(\bvartheta)$). Second, the classical lower bounds on the objective value are worst case with respect to the function $\calL(\bvartheta)$
and do not take into account the data distribution.

Concretely, we consider two families of estimation problems:
\begin{description}
    \item[High-dimensional regression.] Data are i.i.d. pairs $\{(y_i,\bx_i)\}_{i\le n}$, where $y_i\in\reals$ is a label and $\bx_i\in\reals^p$ is a feature vector.
    We assume $\bx_i\sim\normal(\bzero,\bI_p/n)$ and $y_i|\bx_i \sim \P(y_i\in\, \cdot\, |\bx_i^\sT \btheta)$ for a vector $\btheta\in\reals^p$. 
    Our objective is to estimate the coefficients $\theta_j$ from data $\bX\in\reals^{n\times p}$ (the matrix whose $i$-th row is vector $\bx_i$) and $\by\in\reals^n$ (the vector whose $i$-th entry is label $y_i$).
    \item[Low-rank matrix estimation.] Data consist of a matrix $\bX \in \reals^{n \times p}$ where $x_{ij} = \frac1n \blambda_i^\mathsf{T} \btheta_j + z_{ij}$ with $\blambda_i,\btheta_j \in \reals^r$ and $z_{ij} \stackrel{\mathrm{iid}}\sim \normal(0,1/n)$.
    We denote by $\blambda \in \reals^{n\times r}$ and $\btheta \in \reals^{p \times r}$ the matrices whose rows are $\blambda_i^\mathsf{T}$ and $\btheta_j^\mathsf{T}$ respectively.
    Our objective is to to estimate  $\blambda,\btheta$ from data $\bX$. 
\end{description}
In order to discuss these two examples in a unified fashion, we will introduce a dummy vector $\by$
(e.g., the all-zeros vector) as part of the data in the low-rank matrix estimation problem. 
Let us point out that our normalizations are somewhat different from, but completely equivalent to, the traditional ones in statistics.

The first question to address is how to properly define `first order methods.' A moment of thought reveals that the above discussion in terms
of a cost function $\calL(\btheta)$ needs to be revised. Indeed, given either of the above statistical models, there is no simple way to construct
a `statistically optimal' cost function.\footnote{In particular, maximum likelihood is not statistically optimal in high dimension \cite{Bean2013}.}
Further, it is not clear that using a faster optimization algorithm for that cost
will result in faster decrease of the estimation error.

We follow instead a different strategy and introduce the class of  \emph{general first order methods} (GFOM).
In words, these include all algorithms that keep as state sequences of matrices $\bu^{1},\dots, \bu^{t}\in \reals^{n\times r}$, and
$\bv^{1},\dots,\bv^t\in \reals^{p \times r}$,  which are updated by two types of operations: row-wise application of a function, or multiplication by $\bX$
or $\bX^\top$. 
We will then show that standard first order methods, for common choices of the cost $\calL(\btheta)$, are in fact special examples of GFOMs.

Formally, a GFOM  is defined by sequences of functions 
$F^{(1)}_t,G^{(2)}_t:\reals^{r(t+1) + 1}\to\reals^r$, $F^{(2)}_t,G^{(1)}_t:\reals^{r(t+1)}\to\reals^r$, with the $F$'s indexed by $t\ge 0$ and the $G$'s indexed by $t \geq 0$.
In the high-dimensional regression problem, we set $r = 1$.
The algorithm produces two sequences of matrices (vectors for $r = 1$) $(\bu^{t})_{t\ge 1}$,
$\bu^t\in \reals^{n\times r}$, and $(\bv^{t})_{t\ge 1}$, $\bv^t\in \reals^{p \times r}$, 
\begin{subequations}\label{gfom}
\begin{align}
    \bv^{t+1} & = \bX^\top F^{(1)}_t(\bu^1,\dots,\bu^t;\by,\bu)+F^{(2)}_t(\bv^1,\dots,\bv^{t};\bv)\\
    \bu^t &= \bX G^{(1)}_t(\bv^1,\dots,\bv^t;\bv)+G^{(2)}_t(\bu^1,\dots,\bu^{t-1};\by,\bu)\, ,
\end{align}
\end{subequations}
where it is understood that each function is applied row-wise. For instance
\begin{align*}
F^{(1)}_t(\bu^1,\dots,\bu^t;\bu) =(F^{(1)}_t( \bu_i^1,\dots, \bu_i^t;\bu_i))_{i\le n}\in\reals^{n\times r}\, ,
\end{align*}
 where $(\bu_i^s)^{\sT}$ is the $i^\text{th}$ row of $\bu^s$.
Here $\bu,\bv$ are either deterministic or random and independent of everything else.
In particular, the iteration is initialized with $\bv^1 = \bX^\mathsf{T} F_0^{(1)}(\by,\bu) + F_0^{(2)}(\bv)$.
The unknown matrices (or vectors) $\btheta$ and $\blambda$ are estimated after $t_*$ iterations by 
$\hbtheta = G_*(\bv^1,\cdots,\bv^{t_*};\bv)$ and $\hat \blambda = F_*(\bu^1,\ldots,\bu^{t_*};\by,\bu)$, where the latter only applies in the low-rank matrix estimation problem. 
Let us point out that the update also depend on additional information encoded in the two vectors $\bu\in\reals^n$, $\bv\in\reals^p$.
This enables us to model side information provided to the statistician (e.g., an `initialization' correlated with the true signal) or auxiliary randomness.

We study the regime in which $n,p\to \infty$
with $n/p\to \delta\in (0,\infty)$ and $r$ is fixed. 
We assume the number of iterations $t_*$ is fixed, or potentially
$t_*\to\infty$ after $n\to\infty$. 
In other words, we are interested in linear-time or nearly linear-time algorithms
(complexity being measured relative to the input size $np$).
As mentioned above, our main result is a general lower bound on the minimum estimation error that is achieved by any GFOM in this regime.

The paper is organized as follows: Section \ref{sec:Examples} illustrates the setting introduced above in two examples;
Section \ref{sec:Main} contains the statement of our general lower bounds; 
Section \ref{sec:Application} applies these lower bounds to the two examples;
Section \ref{sec:proof-of-main-results} presents an outline of the proof, deferring technical details to appendices.

\section{Two examples}
\label{sec:Examples}

\subsection*{Example $\# 1$: M-estimation in high-dimensional regression and phase retrieval}

Consider the high-dimensional regression problem.
Regularized M-estimators minimize a cost 
\begin{align}
  \calL_{n}(\bvartheta) :=\sum_{i=1}^n\ell(y_i;\<\bx_i,\bvartheta\>)+ \Omega_n(\bvartheta) =\hell_n(\by,\bX\bvartheta)+
  \Omega_n(\bvartheta)\, , \label{eq:Mestimation}
\end{align}
Here  $\ell:\reals\times \reals\to\reals$ is a loss function, $\hell_n(\by,\hby) :=\sum_{i=1}^n\ell(y_i,\hat{y}_i)$ is its empirical average,
and $\Omega_n:\reals^p\to \reals$ is a regularizer. 
It is often the case that $\ell$ is smooth and $\Omega_n$ is separable, i.e.,
$\Omega_n(\bvartheta)=\sum_{i=1}^p\Omega_1(\vartheta_i)$. 
We will assume this to be the case in our discussion.

  The prototypical first order method is proximal gradient \cite{Parikh2013ProximalAlgorithms}: 
  \begin{gather*}
    \btheta^{t+1} = \Prox_{\gamma_t\Omega_1}\big(\btheta^t-\gamma_t\nabla_{\bvartheta} \hell_n(\by,\bX\btheta^t)\big)\, ,\\
    \Prox_{\gamma\Omega_1}(y)
    :=\arg\min_{\theta\in\reals}\left\{\frac{1}{2}(y-\theta)^2+\gamma\Omega_1(\theta)\right\}\, .
  \end{gather*}
  Here $(\gamma_t)_{t\ge 0}$ is a sequence of step sizes and $\Prox_{\gamma\Omega_1}$ acts on a vector coordinate-wise.
  Notice that 
  \begin{align}
  \nabla_{\bvartheta}\hell_n(\by,\bX\btheta^t) = \bX^{\sT}s(\by,\bX\btheta^t)\, ,\;\;\;\; s(\bh,\hby)_i \equiv \frac{\partial \ell}{\partial \hat y_i}(y_,\hat y_i)\, .
  \end{align}
  Therefore proximal gradient --for the cost function \eqref{eq:Mestimation}--  is an example of a GFOM. Similarly, mirror
  descent with a separable Bregman divergence and accelerated proximal gradient methods are easily shown to fit in the same framework.

  Among the countless applications of regularized M-estimation, we will focus on the
  sparse phase retrieval problem. We want to reconstruct a sparse  signal
  $\btheta\in\reals^p$ but only have noisy measurements of the modulus $|\<\btheta,\bx_i\>|$; that is, we lose the `phase' of these projections.
    (We will consider for simplicity the case of a real-valued signal, but the generalization of our results to the complex case should be immediate.)
  
  As a concrete model, we will assume that number of non-zero entries of $\btheta$ is $\|\btheta\|_0\le s_0$. 
  From an information-theoretic viewpoint, it is known that $\btheta$ can be reconstructed accurately as soon as the number
  of measurements satisfies $n\ge C s_0\log(p/s_0)$, with  $C$ a sufficiently large constant \cite{li2013sparse}.
  Several groups have investigated practical reconstruction algorithms by exploiting either semidefinite programming relaxations \cite{li2013sparse}
  or first order methods \cite{schniter2014compressive,candes2015phase,cai2016optimal}. A standard approach would be
  to apply a proximal gradient algorithm to the cost function \eqref{eq:Mestimation} with 
  $\Omega_n(\bvartheta) =\lambda\|\bvartheta\|_1$.
  However, all existing global convergence guarantees for these methods require $n\ge Cs_0^2\log p$. Is the dependence on $s_0^2$
  due to a fundamental computational barrier or an artifact of the theoretical analysis? Recently \cite{soltanolkotabi2019structured} presented partial evidence towards the possibility of `breaking' this barrier, by proving that a first order method can accurately reconstruct the signal for $n\ge Cs_0\log (p/s_0)$, if it is  initialized close enough to the true signal $\btheta$.

  \subsection*{Example $\# 2$: Sparse PCA}
\label{sec:IntroSPCA}

  In a simple model for sparse principal component analysis (PCA), we observe a matrix
  $\bX = \frac{1}{n}\blambda\btheta^{\sT}+\bZ\in\reals^{n\times p}$,
  where $\blambda\in\reals^{n}$ has entries $(\lambda_i)_{i\le n}\stackrel{\mathrm{iid}}\sim\normal(0,1)$, $\btheta\in\reals^p$ is a sparse vector
  with $s_0\ll p$ non-zero entries, and $\bZ$ is a noise matrix with entries $(z_{ij})_{i\le n,j\le p}\stackrel{\mathrm{iid}}\sim\normal(0,1/n)$.
  Given data $\bX$, we would like to reconstruct the signal $\btheta$.
  From an information-theoretic viewpoint, it is known that accurate reconstruction of $\btheta$ is possible if $n\ge Cs_0\log(p/s_0)$,
  with $C$ a sufficiently large constant \cite{amini2008high}.

  A number of polynomial time algorithms have been studied, ranging from simple thresholding algorithms
  \cite{johnstone2009consistency,deshpande2016sparse} to sophisticated convex relaxations \cite{amini2008high,ma2015sum}.
  Among other approaches, one natural idea is to modify the power iteration algorithm of standard PCA
  by computing
  \begin{align}
    \btheta^{t+1} = c_t \, \bX^{\sT}\bX \eta(\btheta^t;\gamma_t)\, .
    \end{align}
  Here $(c_t)_{t\ge 0}$ is a deterministic normalization, and $\eta(\;\cdot\;;\gamma)$ is a thresholding function at level $\gamma$,
  e.g., soft thresholding $\eta(x;\gamma) = \sign(x) (|x|-\gamma)_+$. 
  It is immediate to see that this algorithm is a GFOM.
  More elaborate versions of non-linear power iteration were developed, for example, by \cite{journee2010generalized,ma2013sparse}, and are typically equivalent to suitable GFOMs.

  Despite these efforts, no algorithm is known to succeed unless $n\ge Cs_0^2$. Is this a fundamental barrier or a limitation
  of present algorithms or analysis? Evidence towards intractability was provided by
  \cite{berthet2013optimal,brennan2018reducibility} via reduction from the planted clique problem. 
  Our analysis provides
  new evidence towards the same conclusion.
  
\section{Main results}
\label{sec:Main}

In this section we state formally our general results about high-dimensional regression
and low-rank matrix estimation. 
The next section will apply these general results to concrete instances.
Throughout we make the following assumptions:
\begin{itemize}
\item[\textsf{A1.}] The functions $F^{(1)}_t,G^{(2)}_t,F_*:\reals^{r(t+1)+1}\to\reals$, $F^{(2)}_t,G^{(1)}_t,G_*:\reals^{r(t+1)}\to\reals$, are Lipschitz continuous, with the $F$'s indexed by $t\ge 0$ and the $G$'s indexed by $t \geq 0$.
\item[\textsf{A2.}] The covariates matrix $\bX$ (for high-dimensional regression) or the noise matrix $\bZ$ (for low-rank estimation)
  have entries $x_{ij}\stackrel{\mathrm{iid}}\sim\normal(0,1/n)$, $z_{ij}\stackrel{\mathrm{iid}}\sim\normal(0,1/n)$.
\end{itemize}
Also, we denote by $\cuP_q(\reals^k)$ the set of probability distributions  with finite $q$-th moment on $\reals^k$ and $\cuP_{\mathrm{c}}(\reals^k)$ those with compact support. 
We say a function $f: \reals^k \rightarrow \reals$ is \emph{pseudo-Lipschitz of order 2} if there exists constant $C$ such that $|f(\bx) - f(\bx')| \leq C(1 + \|\bx\| + \|\bx'\|)\|\bx - \bx'\|$ for all $\bx,\bx' \in \reals^k$.
We call a function $\ell:(\reals^k)^2 \rightarrow \reals$ a \emph{quadratically-bounded loss} if it is non-negative and pseudo-Lipschitz of order 2 and there exists $C > 0$ such that for all $\bx,\bx',\bd \in \reals^k$ we have $|\ell(\bx,\bd) - \ell(\bx',\bd)| \leq C(1 + \sqrt{\ell(\bx,\bd)} + \sqrt{\ell(\bx',\bd)})\|\bx - \bx'\|$.

\subsection{High-dimensional regression}\label{sec:main-results-hd-reg}

We make the following additional assumptions for the regression problem:
\begin{itemize}
  \item[\textsf{R1.}] We sample $\{(w_i,u_i)\}_{i\le n}\stackrel{\mathrm{iid}}\sim\mu_{W,U}$, $\{(\theta_i,v_i)\}_{i\le p}\stackrel{\mathrm{iid}}\sim\mu_{\Theta,V}$ for $\mu_{\Theta,V},\mu_{W,U}\in \cuP_2(\reals^2)$. 
  \item[\textsf{R2.}] 
  There exists a measurable function $h: \reals^2 \rightarrow \reals$ such that $y_i= h(\bx_i^{\sT}\btheta,w_i)$.
  Moreover, there exists constant $C$ such that $|h(x,w)| \leq C(1 + |x| + |w|)$ for all $x,w$.
\end{itemize}
Notice that the description in terms of a probability kernel $\P(y_i\in \,\cdot\, |\bx_i^{\sT}\btheta)$ is equivalent to the one
in terms of a `noisy' function $y_i= h(\bx_i^{\sT}\btheta,w_i)$ in most cases of interest.

Our lower bound is defined in terms of a one-dimensional recursion.
Let $(\Theta,V)\sim \mu_{\Theta,V}$.
Let $\mmse_{\Theta,V}(\tau^2)$ be the minimum mean square error for estimation of $\Theta$ given observations $V$ and $\Theta + \tau G$ where $G \sim \normal(0,1)$ independent of $\Theta$.
Set $\tau_\Theta^2 = \E[\Theta^2]$ and $\tau_0^2 = \infty$, and define recursively
\begin{equation}\label{bamp-se-hd-reg}
    \begin{gathered}
      \tilde \tau_s^2 = \frac{1}{\delta} \,\mmse_{\Theta,V}(\tau_s^2),\;\;\;\;\;\;\sigma_s^2  = \frac1\delta(\tau_\Theta^2-\mmse_{\Theta,V}(\tau_s^2)) \, ,\\
      \frac{1}{\tau_{s+1}^2} = \frac1{\ttau_s^2} \E\left[\E[G_1 | Y , G_0 , U]^2\right],\\
    \end{gathered}
\end{equation}
where $Y = h(\sigma_s G_0 + \ttau_s G_1,W)$ and the expectation is with respect to
$G_0,G_1\stackrel{\mathrm{iid}}\sim \normal(0,1)$ and $(W,U) \sim \mu_{W,U}$ independent.
\begin{theorem}\label{thm:hd-reg-lower-bound}
    Under assumptions \textsf{A1}, \textsf{A2}, \textsf{R1}, \textsf{R2} in the high-dimensional regression model and under the asymptotics $n,p \rightarrow \infty$, $n/p \rightarrow \delta \in (0,\infty)$,
    let $\hbtheta^t$ be output of any GFOM after $t$ iterations ($2t-1$ matrix-vector multiplications).
    Then
    \begin{align*}
       \lim_{n \rightarrow \infty} \frac1p\|\hbtheta^t-\btheta\|_2^2 \ge \mmse_{\Theta,V}(\tau_t^2)\, .
    \end{align*}
    More generally, for any quadratically-bounded loss $\ell: \reals^2 \rightarrow \reals_{\geq 0}$,
    \begin{gather}\label{eq:hd-reg-lb}
      \lim_{n \rightarrow \infty} \frac1p \sum_{j=1}^p \ell(\theta_j, \htheta_j^t ) \geq \inf_{\htheta(\,\cdot\,)}
      \E\big\{\ell(\Theta,\htheta(\Theta+\tau_t G,V))\big\}\, ,
    \end{gather}
    where $(\Theta,V) \sim \mu_{\Theta,V}$ independent of $G \sim \normal(0,1)$, 
    and the infimum on the right-hand side is over measurable functions $\htheta:\reals^2\to\reals$. 
    The limits are in probability and to a constant, and they are guaranteed to exist.
    For all $\epsilon > 0$, there exist GFOMs which satisfy these bounds to within tolerance $\epsilon$.
\end{theorem}

\subsection{Low-rank matrix estimation}\label{sec:main-results-lr-mat}
\def\VE{{\boldsymbol V}}

We make the following additional assumption:
\begin{itemize}
\item[\textsf{M1.}] We sample $\{(\blambda_i,\bu_i)\}_{i\le n}\stackrel{\mathrm{iid}}\sim\mu_{\bLambda,\bU}$
  and $\{(\btheta_j,\bv_j)\}_{j\le p}\stackrel{\mathrm{iid}}\sim\mu_{\bTheta,\bV}$ for  $\mu_{\bLambda,\bU},\mu_{\bTheta,\bV}\in \cuP_2(\reals^{2r})$.
  \end{itemize}

Again, our lower bound is defined in terms of recursion, which this time is defined over positive semidefinite matrices $\bQ_t,\hbQ_t\in \reals^{r\times r}$,
$\bQ_t,\hbQ_t\succeq \bzero$. 
Set $\hbQ_0=\bzero$, and define recursively
\begin{align}
 \bQ_{t+1}  = \VE_{\bLambda,\bU}(\hbQ_t)\, ,\;\;\;\;\;\;\;\;\;\;\;
 \hbQ_{t}  = \frac{1}{\delta}\VE_{\bTheta,\bV}(\bQ_t)\, ,\label{eq:SE_Matrix}
\end{align}
where we define the second moment of the conditional expectation  $\VE_{\bTheta,\bV}:\reals^{r\times r}\to\reals^{r\times r}$ by
\begin{align*}
  \VE_{\bTheta,\bV}(\bQ):=\E\Big\{ \E[\bTheta|\bQ^{1/2}\bTheta+\bG=\bY;\bV] \E[\bTheta|\bQ^{1/2}\bTheta+\bG=\bY;\bV]^{\sT} \Big\},
\end{align*}
and analogously for $\VE_{\bLambda,\bU}(\hbQ)$.
Here the expectation is with respect to $(\bTheta,\bV)\sim \mu_{\bTheta,\bV}$ and an independent Gaussian vector 
$\bG\sim\normal(\bzero,\id_r)$.
Notice in particular that $\E\{\bTheta\bTheta^{\sT}\}- \VE_{\bTheta,\bV}(\bQ)$ is
the vector minimum mean square error when $\bTheta$ is observed in Gaussian noise with covariance  $\bQ^{-1}$.
For $r=1$, Eq.~\eqref{eq:SE_Matrix} is a simple scalar recursion.
\begin{theorem}\label{thm:lr-mat-lower-bound}
  Under assumptions $\textsf{A1}$, $\textsf{A2}$, $\textsf{M1}$ in the low-rank matrix estimation model and under the under the asymptotics $n,p\rightarrow \infty,n/p \rightarrow \delta \in (0,\infty)$,
  let $\hbtheta^t$ be output of any GFOM after $t$ iterations ($2t-1$ matrix-vector multiplications). Then
    \begin{align*}
       \lim_{n \rightarrow \infty} \frac1p\|\hbtheta^t-\btheta\|_{\mathsf{F}}^2 \ge \E\{\|\bTheta\|^2\}-\Tr \VE_{\bTheta,\bV}(\bQ_t) \, .
    \end{align*}
    More generally, for any quadratically-bounded loss $\ell: \reals^{2r} \rightarrow \reals_{\geq 0}$,
    \begin{gather}\label{eq:lr-mat-lb}
      \lim_{n \rightarrow \infty} \frac1p \sum_{j=1}^p \ell(\btheta_j, \hbtheta_j^t ) \geq \inf_{\hbtheta(\,\cdot\,)}
      \E\big\{\ell(\bTheta,\hbtheta(\bQ_t^{1/2}\bTheta+\bG, \bV))\big\}\, ,
    \end{gather}
    where the infimum on the right-hand side is over functions $\hbtheta:\reals^r\to\reals^r$.
    The limits are in probability and to a constant, and they are guaranteed to exist.
    As above, for all $\epsilon > 0$ there exist GFOMs which satisfy these bounds to within tolerance $\epsilon$.
\end{theorem}

\subsection{Discussion}

Our motivations are similar to the ones for statistical query (SQ) lower bounds
\cite{feldman2017statistical,feldman2017statistical_b}:
we want to provide estimation lower bounds under a restricted computational model,
that are sensitive to the data distribution. However the scope of our approach is significantly different from SQ algorithms:  the latter can query data 
distributions and compute approximate expectations with respect to that distribution.
In contrast, our algorithms work with a fixed sample (the data matrix $\bX$ and responses $\by$), which is queried multiple times. These queries can be thought
as weighted averages of \emph{both rows and columns} of $\bX$ and, as such,
cannot be simulated by the SQ oracle.
For instance, the proximal gradient method or
the nonlinear power iteration of Section
\ref{sec:IntroSPCA} cannot be framed as a SQ algorithms.

The lower bounds of Theorems \ref{thm:hd-reg-lower-bound} and \ref{thm:lr-mat-lower-bound} are satisfied with equality by a specific first order method that is an approximate message passing (AMP) algorithm, with Bayes updates. This can be regarded as a version of belief propagation (BP) for densely connected graphs \cite{koller2009probabilistic}, or an iterative implementation of the TAP equations from spin glass theory \cite{mezard1987spin}. 

Our proof builds on the asymptotically exact analysis of AMP algorithms developed in \cite{bolthausen2014iterative,bayati2011dynamics,javanmard2018,Berthier2017StateFunctions}. However we need to overcome three technical obstacles:
$(1)$~Show that any GFOM can be reduced (in a suitable sense) to a certain AMP algorithms, whose behavior can be exactly tracked. $(2)$~Show that Bayes-AMP is optimal among all AMP algorithms. We achieve this goal by considering an estimation 
problem on trees and showing that, in a suitable large degree limit, it has the same asymptotic behavior as AMP on the complete graph. On trees it is immediate to see that 
BP is the optimal local algorithm. $(3)$~We need to prove that the asymptotic behavior of BP for trees of large degree is equivalent to the one of Bayes-AMP on the original problem.
This amounts to proving a Gaussian approximation theorem for BP. While similar results were obtained in the past for discrete models 
\cite{sly2009reconstruction,mossel2016local}, the current setting
is technically more challenging because the underlying variables $\theta_i$ are continuous and unbounded.

While the line of argument above is --in hindsight-- very natural, the conclusion is broadly useful. For instance, \cite{antenucci2019glassy} study a class of 
of message passing algorithms inspired to replica symmetry breaking and survey propagation \cite{mezard2002analytic}, and observe that they 
do not perform better than Bayes AMP.
These algorithms are within the scope of our Theorem \ref{thm:lr-mat-lower-bound}, which implies that indeed 
they cannot outperform Bayes AMP, for any constant number of iterations.

Finally, a sequence of recent papers characterize the asymptotics of the Bayes-optimal 
estimation error in the two models described above \cite{lelarge2019fundamental,barbier2019optimal}.
It was conjectured that, in this context, no polynomial-time algorithm can outperform Bayes AMP, provided these algorithms have access to an arbitrarily small amount of side information.\footnote{Concretely, side information can take the form $\bv = \eta\btheta+\bg$ for $\eta>0$ arbitrarily small, $\bg\sim\normal(0,\id_p)$}
Theorems \ref{thm:hd-reg-lower-bound} and \ref{thm:lr-mat-lower-bound} 
establish this result within the restricted class of GFOMs.

\section{Applying the general lower bounds}
\label{sec:Application}

In our two examples, we will refer to the sets $\Ball^p_0(k) \subset \reals^p$ of $k$-sparse vectors and $\Ball^p_2(R) \subset \reals^p$ of vectors with $\ell_2$-norm bounded by $R$.

\subsection*{Example $\# 1$: Sparse phase retrieval}

For the reader's convenience, we follow the standard normalization in phase retrieval, whereby the `sensing vectors'
(i.e.\ the rows of the design matrix) have norm concentrated around one. In other words, we observe
$y_i\sim p(\,\cdot\, |\tbx_i^{\sT}\btheta)\de y$, where $\tbx_i \sim\normal(0,\id_p/p)$. 

In order to model the phase retrieval problem, we assume that the conditional density $p(\, \cdot\,|\, \cdot\,\,)$ satisfies the symmetry condition
$p(y|x) = p(y|-x)$. In words: we only observe a noisy version of the absolute value $|\<\tbx_i,\btheta\>|$.
An important role is played by the following critical value of the number of observations per dimension
\begin{align}
  \delta_{\sp}:=\left(\int_{\reals}\frac{\E_G[p(y|G)(G^2-1)]}{\E_G[p(y|G)]}
  \,\de y\right)^{-1}\, .
\end{align}
Here expectation is with respect to $G\sim\normal(0,1)$. 
It was proved in \cite{mondelli2019fundamental} that,
if $\|\btheta\|_2=\sqrt{p}$ and $n>(\delta_{\sp}+\eta) p$, for some $\eta$ bounded away from zero,
then there exists a simple spectral estimator $\hbtheta_{\sp}$ that achieves
weak recovery, i.e., a positive correlation with the true signal. Namely, 
$\frac{|\<\hbtheta_{\sp},\btheta\>|}{\|\hbtheta_{\sp}\|_2\|\btheta\|_2}$ is bounded away from zero as $p,n\to\infty$.

In the case of a dense signal $\btheta$ and observation model $y_i = |\tbx_i^{\sT}\btheta| + w_i,\, w_i \sim \normal(0,\sigma^2)$, the oversampling ratio $\delta_{\sp}$ is known to be information-theoretically optimal:
for $n<(\delta_{\sp}-\eta)p$ no estimator can achieve a correlation that is bounded away from $0$ \cite{mondelli2019fundamental}. On the other hand,
if $\btheta$ has at most $p\eps$ nonzero entries, it is information-theoretically possible to reconstruct it from $\delta>C\eps\log(1/\eps)$ phaseless measurements per dimension \cite{li2013sparse}.

Our next result implies that no GFOM can achieve reconstruction from 
$O(\eps\log(1/\eps))$ measurements per dimension, unless it is initialized close enough
to the true signal. In order to model the additional information provided by the initialization we
assume to be given 
  \begin{align}
    \obv = \sqrt{\alpha}\,\btheta/\|\btheta\|_2+\sqrt{1-\alpha}\tbg, \;\;\;\;\;\;(\tg_i)_{i\le p}\stackrel{\mathrm{iid}}\sim\normal (0,1/p), .\label{eq:Initialization}
  \end{align}
 Notice that with this normalization $\|\obv\|_2$ concentrates tightly around $1$, and $\sqrt{\alpha}$ can be interpreted as the cosine
 of the angle between $\btheta$ and $\obv$.

 \begin{corollary}\label{coro:PhaseRetrieval}
   Consider the phase retrieval model, for a sequence of deterministic signals $\btheta\in\reals^p$, and let  $\cuT(\eps,R) :=\Ball^p_0(p\eps)\cap \Ball^p_2(R)$.  Assume the noise kernel $p(\,\cdot\,|x)$ to satisfy the conditions of
   Theorem \ref{thm:hd-reg-lower-bound}  and to be be twice differentiable with respect to $x$.
   
   Then, for any \emph{$\delta<\delta_{\sp}$},  there exists $\alpha_*=\alpha_*(\delta,\eps)>0$
   and $C_*=C_* (\delta,\eps)$ such that, if $\alpha\le \alpha_*$, then
   \begin{align}
     \sup_{t\ge 0} \lim_{n,p\to\infty}\inf_{\btheta\in \cuT(\eps,\sqrt{p})}\E
    \frac{\<\btheta,\hbtheta^t\>}{\|\btheta\|_2\|\hbtheta^t\|_2}\le C_*\sqrt{\alpha}\, .
   \end{align}
   The same conclusion holds if $\btheta$ is drawn randomly with i.i.d.\ entries
   $\theta_{i}\sim \mu_{\theta}:= (1-\eps)\delta_0+(\eps/2)(\delta_{\mu}+\delta_{-\mu})$, $\mu = 1/\sqrt{\eps}$.
  \end{corollary}

  \subsection*{Example $\# 2$: Sparse PCA}

  For ease of interpretation, we assume the observation model $\tbX = \blambda\obtheta^{\sT}+\tbZ$,
  where $(\tz_{ij})_{i\le n,j\le p}\sim\normal(0,1)$ and $(\lambda_i)_{i\le n}\sim\normal(0,1)$. Equivalently, conditional on $\obtheta$,
  the rows of $\tbX$ are i.i.d. samples $\tbx_i\sim\normal(0,\bSigma)$, $\bSigma = \id_p+ \obtheta\obtheta^{\sT}$.
  We also assume to have access to an initialization $\obv$ correlated with $\obtheta$, as per Eq.~\eqref{eq:Initialization}.
  In order to apply Theorem \ref{thm:lr-mat-lower-bound}, we choose a specific distribution for the spike. 
  Defining $\btheta =\obtheta\sqrt{p}$, we assume that the entries of $\btheta$ follow a  three-points sparse distribution $(\theta_i)_{i\le p}\sim \mu _{\theta}:=(1-\eps)\delta_0+(\eps/2)(\delta_{+\mu}+\delta_{-\mu})$. The next lemma specializes Theorem \ref{thm:lr-mat-lower-bound}.
  \begin{lemma}\label{lemma:SPCA}
        Assume the sparse PCA model with the distribution of $\obtheta$ given above. Define $(q_t)_{t\ge 0}$  by   
  \begin{align}
    q_{t+1}  &= \frac{V_{\pm}(q_t+\talpha)}{1+V_{\pm}(q_t+\talpha)}\, ,\;\;\;\;\;\;
    q_0=0\, ,\label{eq:SE-SPCA-1}\\
    V_{\pm}(q)  &:= e^{-\delta q\mu^2}\mu^2\eps^2\E\left\{\frac{\sinh(\mu\sqrt{\delta q} G )^2}{1-\eps+\eps e^{-\delta q\mu^2/2}
      \cosh(\mu\sqrt{\delta q} G )}\right\}\, , \label{eq:SE-SPCA-2}
  \end{align}
  
  where $\talpha=\alpha/(\mu^2\eps(1-\alpha))$.
  Then, for any GFOM
  \begin{align}
    \lim_{n,p\to\infty}\frac{\<\obtheta,\hbtheta^t\>}{\|\obtheta\|_2\|\hbtheta^t\|_2}\le
    \sqrt{\frac{V_{\pm}(q_t + \talpha)}{\mu^2\eps}}\, .\label{eq:StatementPCA}
  \end{align}
  \end{lemma}
  The  bound in the last lemma holds for random vectors $\obtheta$ with i.i.d.\ entries from the three-points distribution.
  As a consequence, it implies a minimax bound for non-random vectors $\obtheta$ with given $\ell_2$-norm and sparsity.
  We state this bound in the corollary below. In order to develop explicit expressions, we analyze the recursion of Eqs.~\eqref{eq:SE-SPCA-1}, \eqref{eq:SE-SPCA-2}. 
  \begin{corollary}\label{coro:SPCA}
    Assume the sparse PCA model, for $\obtheta\in\reals^p$ a deterministic vector and $\blambda$, $\tbZ$ random,
    and consider the parameter space $\cuT(\eps,R) :=\Ball^p_0(p\eps)\cap \Ball^p_2(R)$. 
\begin{enumerate}
 \item[$(a)$] If $R^2<1/\sqrt{\delta}$, then there exists $\alpha_*=\alpha_*(R,\delta,\eps), C_* = C_*(R,\delta,\eps)$
  such that, for $\alpha<\alpha_*$, and any GFOM
  \begin{align}
    \sup_{t\ge 0} \lim_{n,p\to\infty}\inf_{\obtheta\in \cuT(\eps,R)}\E
    \frac{\<\obtheta,\hbtheta^t\>}{\|\obtheta\|_2\|\hbtheta^t\|_2}\le C_*\sqrt{\alpha}\, .
  \end{align}
\item[$(b)$]  If $R^2< \sqrt{(1-\eps)/4\delta}$, then the above statement holds with $\alpha_* =\left(\frac{\eps}{4\delta}\wedge \frac{1}{2}\right)$,  $C_* = 3/R^2$.
  \end{enumerate}
  \end{corollary}
In words, the last corollary implies that for $R^2\delta<1$, no estimator 
achieves a non-vanishing correlation with the true signal $\obtheta$, unless sufficient side information about $\obtheta$ is available. 
Notice that for $R^2\delta=1$ is the threshold above which
the principal eigenvector of the empirical covariance $\tbX^{\sT}\tbX/n$
becomes correlated with $\obtheta$. Hence, our result implies that, simple PCA fails, then every GFOM will 
fail. 

Viceversa, if simple PCA succeed, then it can be implemented via a GFOM,
provided arbitrarily weak side information if available.
Indeed, assume side information $\bv=\eta \btheta+\bg$,
with $\bg\sim\normal(0,\id_p)$, and an $\eta$ arbitrarily small constant.
Then the power method initialized at $\bv$ converges to an estimate
that has  correlation with $\btheta$ bounded away from zero in $O(\log(1/\eta))$
iterations.



\section{Proof of main results}
\label{sec:proof-of-main-results}

In this section, we prove Theorems \ref{thm:hd-reg-lower-bound} and \ref{thm:lr-mat-lower-bound} under stronger assumptions than in their statements.
In the high-dimensional regression model, these assumptions are as follows.
\begin{itemize}
  \item[\textsf{R3.}] Given  $\mu_{\Theta,V}\in\cuP_{\mathrm{c}}(\reals^2)$ and $\mu_{\bW,U} \in \cuP_4(\reals^k \times \reals)$ for some $k \geq 1$, we sample $\{(\theta_i,v_i)\}_{i\le p}\stackrel{\mathrm{iid}}\sim\mu_{\Theta,V}$, $\{(\bw_i,u_i)\}_{i\le n}\stackrel{\mathrm{iid}}\sim\mu_{\bW,U}$.
  \item[\textsf{R4.}] There exists Lipschitz function $h: \reals \times \reals^k \rightarrow \reals$ such that $y_i= h(\bx_i^{\sT}\btheta,\bw_i)$.
  Measure $\mu_{\bW,U}$ has regular conditional probability distribution $\mu_{\bW|U}(u,\cdot)$ such that, for all fixed $x,u$, the distribution of $h(x,\bW)$ when $\bW \sim \mu_{\bW|u}(u,\cdot)$ has positive and bounded density $p(y|x,u)$ with respect Lebesgue measure. Further, $\partial_x^k \log p(y|x,u)$ for $1 \leq k \leq 5$ exists and is bounded.
\end{itemize}
In the low-rank matrix estimation model, this assumption is as follows.
\begin{itemize}
  \item[\textsf{M2.}] Given  $\mu_{\bLambda,\bU},\mu_{\bTheta,\bV} \in \cuP_{\mathrm{c}}(\reals^{2r})$, we sample $\{(\blambda_i,\bu_i)\}_{i\le n}\stackrel{\mathrm{iid}}\sim\mu_{\bLambda,\bU}$, $\{(\btheta_j,\bv_j)\}_{j\le p}\stackrel{\mathrm{iid}}\sim\mu_{\bTheta,\bV}$.
\end{itemize}
In Appendix \ref{app:strong-to-weak-ass}, we show that Theorem \ref{thm:hd-reg-lower-bound} (resp.\ Theorem \ref{thm:lr-mat-lower-bound}) under assumptions \textsf{R3} and \textsf{R4} (resp.\ \textsf{M2}) implies the theorem under the weaker assumptions \textsf{R1} and \textsf{R2} (resp.\ \textsf{M1}).

\subsection{Reduction of GFOMs to approximate message passing algorithms}

Approximate message passing (AMP) algorithms are a special class of GFOMs that admit an asymptotic characterization called \emph{state evolution}
\cite{bayati2011dynamics}.
We show that, in both models we consider, any GFOM is equivalent to an AMP algorithm after a change of variables.

An AMP algorithm is defined by sequences of Lipschitz functions $(f_t:\reals^{r(t+1) + 1} \rightarrow \reals^r)_{t\geq 0}$, $(g_t: \reals^{r(t+1)} \rightarrow \reals^r)_{t \geq 1}$.
It generates sequences $(\ba^t)_{t \geq 1}$, $(\bb^t)_{t \geq 1}$ of matrices in $\reals^{p \times r}$ and $\reals^{n \times r}$, respectively, according to
\begin{equation}\label{eq:amp}
  \begin{split}
      \ba^{t+1} = \bX^\mathsf{T} f_t(\bb^1,\ldots,\bb^t;\by,\bu) - \sum_{s = 1}^t g_s(\ba^1,\ldots,\ba^s;\bv)\bxi_{t,s}^\sT,\\
      \bb^t = \bX g_t(\ba^1,\ldots,\ba^t;\bv) - \sum_{s = 0}^{t-1}  f_s(\bb^1,\ldots,\bb^s;\by,\bu)\bzeta_{t,s}^\sT,
  \end{split}
\end{equation}
with initialization $\ba^1 = \bX^\mathsf{T} f_0(\by,\bu)$. 
Here $(\bxi_{t,s})_{1\le s \le t}$, $(\bzeta_{t,s})_{0 \leq s < t}$ are deterministic $r\times r$ matrices.
The we refer to the recursion \eqref{eq:amp} as to an AMP algorithm
if only if the matrices $(\bxi_{t,s})_{1\le s \le t}$, $(\bzeta_{t,s})_{0 \leq s < t}$ are determined by the functions $(f_t)_{t\geq 0}$, $(g_t)_{t \geq 1}$ 
in a  specific way, which depends on the model under consideration,
and we describe in Appendix \ref{app:proofs-gfom-to-amp}.
For this special choice of the matrices $(\bxi_{t,s})_{1\le s \le t}$, $(\bzeta_{t,s})_{0 \leq s < t}$, the iterates $\ba^t,\bb^t$ are asymptotically Gaussian, with a covariance that can be determined via the state evolution recursion.

The next lemma, proved in Appendix \ref{app:proofs-gfom-to-amp}, makes this precise and describes the state evolution of the resulting AMP algorithm.
\begin{lemma}\label{lem:gfom-to-amp}
  Under assumptions \textsf{A1}, \textsf{A2}, \textsf{R3}, \textsf{R4} (for high-dimensional regression) or assumptions \textsf{A1}, \textsf{A2}, \textsf{M2} (for low-rank matrix estimation), there exist Lipschitz functions $(f_t)_{t\geq 0},(g_t)_{t \geq 1}$ as above and $(\varphi_t: \reals^{r(t+1)}\rightarrow \reals)_{t \geq 1},(\phi_t:\reals^{r(t+1)+1}\rightarrow\reals)_{t\geq 1}$, such that the following holds.
  Let $(\bxi_{t,s})_{1\le s \le t}, (\bzeta_{t,s})_{0 \leq s < t}$ be $r\times r$ matrices determined by the general AMP prescription 
  (see Appendix  \ref{app:proofs-gfom-to-amp}), and define 
  $\{\ba^s,\bb^s\}_{s\ge 0}$ via the AMP algorithm \eqref{eq:amp}.
  Then we have
      \begin{gather*}
          \bv^t = \varphi_t(\ba^1,\ldots,\ba^t;\bv),\quad t \geq 1,\\
          \bu^t = \phi_t(\bb^1,\ldots,\bb^t;\by,\bu),\quad t \geq 1.
        \end{gather*}
  Further, state evolution determines two collections of of $r\times r$ matrices $(\bT_{s,t})_{s,t \geq 1}, (\balpha_t)_{t \geq 1}$ such that for all pseudo-Lipschitz functions $\psi:\reals^{r(t+2)}\rightarrow \reals$ of order 2,
  \begin{equation}\label{eq:se}
        \frac1p \sum_{j=1}^p \psi(\ba^1_j,\ldots,\ba^t_j,\bv_j,\btheta_j) \stackrel{\mathrm{p}}\rightarrow \E[\psi(\balpha_1\bTheta + \bZ^1,\ldots,\balpha_t\bTheta + \bZ^t,\bV,\bTheta)],\\
    \end{equation}
    where $(\bTheta,\bV) \sim \mu_{\bTheta,\bV}$ independent of $(\bZ^1,\ldots,\bZ^t)\sim \mathsf{N}(\bzero,\bT_{[1:t]})$.
    Here $\bT_{[1:t]} \in \reals^{tr\times tr}$ is a positive semi-definite block matrix with block $(s,s')$ given by $\bT_{s,s'}$.\footnote{We emphasize that the construction of all relevant functions and matrices depend on the model. We describe these constructions and prove Lemma \ref{lem:gfom-to-amp} in Appendix \ref{app:proofs-gfom-to-amp}.}
\end{lemma}

Lemma \ref{lem:gfom-to-amp} implies that the estimator $\hat \btheta^t$ in Theorem \ref{thm:hd-reg-lower-bound} and \ref{thm:lr-mat-lower-bound} can alternatively be viewed as a Lipschitz function $g_*:\reals^{r(t+1)} \rightarrow \reals^r$ of the AMP iterates $(\ba^s)_{s \leq t}$ and side information $\bv$, applied row-wise.
Thus, $\ell(\btheta_j,\hat \btheta_j^t)$ can be viewed as a pseudo-Lipschitz function of order 2 applied to $(\ba_j^s)_{s\leq t},\bv_j,\btheta_j$; namely, $\ell(\btheta_j,g_*((\ba_j^s)_{s\leq t},\bv_j))$.
Then, Lemma \ref{lem:gfom-to-amp} implies that the limits in Theorems \ref{thm:hd-reg-lower-bound} and \ref{thm:lr-mat-lower-bound} exist 
and have lower bound 
\begin{equation}\label{eq:test-fun-conv}
  \inf R_\ell(g_*,(\balpha_s),(\bT_{s,s'})) := \inf \E[\ell(\bTheta,g_*(\balpha_1\bTheta + \bZ^1,\ldots,\balpha_t\bTheta + \bZ^t,\bV))],
\end{equation}
where the infimum is taken over Lipschitz functions $g_*$ and matrices $(\balpha_s),(\bT_{s,s'})$ generated
by the state evolution of \emph{some} AMP algorithm.
This lower bound is characterized in the following sections.

\subsection{Models and message passing on the computation tree}

We introduce two statistical models on trees and a collection of algorithms which correspond, in a sense we make precise, to the high-dimensional regression and low-rank matrix estimation models, and AMP algorithms. 
We derive lower bounds on the estimation error in these models using information-theoretic, rather than algorithmic, techniques. 
We then transfer these to lower bounds on \eqref{eq:test-fun-conv}.
The models are defined using an infinite connected tree $\calT = (\calV, \calF, \calE)$ consisting of infinite collections of variable nodes $\calV$, factor nodes $\calF$, and edges $\calE$.
Factor nodes have degree $p$ and have only variables nodes as neighbors, and variable nodes have degree $n$ and have only factor nodes as neighbors. 
These properties define the tree uniquely up to isomorphism.
We denote the set of neighbors of a variable $v$ by $\partial v$, and similarly define $\partial f$.
We call $\calT$ the \emph{computation tree}.

The statistical models are joint distributions over random variables associated to the nodes and edges of the computation tree.
\begin{description}
    \item[High-dimensional regression on the computation tree.] 
    The random variables $\{(\theta_v,v_v)\}_{v \in \calV} \stackrel{\mathrm{iid}}\sim \mu_{\Theta,V}$, 
    $\{(\bw_f,u_f)\}_{f \in \calF} \stackrel{\mathrm{iid}}\sim \mu_{\bW,U}$, and $\{x_{fv}\}_{(f,v)\in \calE} \stackrel{\mathrm{iid}}\sim \normal(0,1/n)$ are generated independently.
    We assume $\mu_{\Theta,V}$, $\mu_{\bW,U}$ are as in assumption \textsf{R3}.
    We define $y_f = h(\sum_{v \in \partial f} x_{fv} \theta_v , \bw_f)$ for $h$ as in assumption \textsf{R4}.
    For each $v \in \calV$, our objective is to estimate the coefficient $\theta_v$ from data $(y_f,u_f)_{f \in \calF}$, $(v_v)_{v \in \calV}$, and $(x_{fv})_{(f,v) \in \calE}$.
    \item[Low-rank matrix estimation on the computation tree.] 
    The random variables $\{(\btheta_v,\bv_v)\}_{v \in \calV} \stackrel{\mathrm{iid}}\sim \mu_{\bTheta,\bV}$, $\{(\blambda_f,\bu_f)\}_{f \in \calF}$, and $\{z_{fv}\}_{(f,v) \in \calE} \stackrel{\mathrm{iid}}\sim \normal(0,1/n)$ are generated independently.
    We assume $\mu_{\bLambda,\bU}$, $\mu_{\bTheta,\bV}$ are as in assumption \textsf{M2}.
    For each $v \in \calV$, our objective is to estimate $\btheta_v$ from data $(x_{fv})_{(f,v) \in \calE}$, $(\bv_v)_{v \in \calV}$, and $(\bu_f)_{f \in \calF}$.
\end{description}
When ambiguity will result, we will refer to the models of Section \ref{sec:Main} as high-dimensional regression and low-rank matrix estimation \emph{on the graph.}\footnote{This terminology is motivated by viewing the models of Section \ref{sec:Main} as equivalent to the tree-based models except that they are  defined with respect to a finite complete bipartite graph between factor and variable nodes.}
As on the graph, we introduce dummy variables $(y_f)_{f \in \calF}$ in the low-rank matrix estimation problem on the computation tree.

To estimate $\btheta_v$, we introduce the class of \emph{message passing algorithms}. 
A message passing algorithm is defined by sequences of Lipschitz functions $(f_t:\reals^{r(t+1)+1}\rightarrow \reals^r)_{t \geq 0}$, $(g_t:\reals^{r(t+1)}\rightarrow \reals^r)_{t \geq 1}$.
For each edge $(f,v) \in \calE$, it generates sequences $(\ba_{v\rightarrow f}^t)_{t \geq 1}$, $(\bq_{v \rightarrow f}^t)_{t \geq 1}$, $(\bb_{f\rightarrow v}^t)_{t \geq 1}$, and $(\br_{f \rightarrow v}^t)_{t \geq 0}$ of vectors in $\reals^r$, called \emph{messages}, according to
\begin{equation}\label{mp-on-tree}
  \begin{gathered}
      \ba_{v\rightarrow f}^{t+1} = \sum_{f' \in \partial v \setminus f} x_{f'v} \br_{f'\rightarrow v}^t, \qquad  \br_{f\rightarrow v}^t = f_t( \bb_{f\rightarrow v}^1,\ldots, \bb_{f\rightarrow v}^t;y_f,\bu_f), \\
      \bb_{f\rightarrow v}^t = \sum_{v' \in \partial f \setminus v} x_{fv'}  \bq_{v'\rightarrow f}^t, \qquad  \bq_{v\rightarrow f}^t = g_t( \ba_{v\rightarrow f}^1, \ldots, \ba_{v\rightarrow f}^t;\bv_v),
  \end{gathered}
\end{equation}
with initialization $\br_{f \rightarrow v}^0 = f_0(y_f,\bu_f)$ and $\ba_{v\rightarrow f}^1 = \sum_{f' \in \partial v \setminus f} x_{f'v} \br_{f'\rightarrow v}^0$.
We also define for every variable and factor node the vectors
\begin{equation}\label{beliefs}
    \ba_v^{t+1} = \sum_{f \in \partial v} x_{fv} \br_{f\rightarrow v}^t,\qquad \bb_f^t = \sum_{v \in \partial f} x_{fv} \bq_{v \rightarrow f}^t.
\end{equation}
These are called \emph{beliefs}.
The vector $\btheta_v$ is estimated after $t$ iterations by $\hat \btheta_v^t = g_*(\ba_v^1,\ldots,\ba_v^t;\bv_v)$.

Message passing algorithms on the computation tree correspond to AMP algorithms on the graph in the sense that their iterates are asymptotically characterized by the same state evolution.
\begin{lemma}\label{lem:se-mp-tree}
  In both the high-dimensional regression and low-rank matrix estimation problems on the tree, the following is true.
  For any Lipschitz functions $(f_t)_{t \ge 0}$, $(g_t)_{t \ge 1}$, 
  there exist collections of $r\times r$ matrices $(\bT_{s,t})_{s,t \geq 1}, (\balpha_t)_{t \geq 1}$ such that for any node $v$ chosen independently of the randomness on the model, fixed $t \geq 1$, and under the asymptotics $n,p \rightarrow \infty$, $n/p \rightarrow \delta \in (0,\infty)$,
    the message passing algorithm \eqref{mp-on-tree} generates beliefs at $v$ satisfying
    \begin{gather*}
        (\ba^1_v,\ldots,\ba^t_v,\bv_v,\btheta_v) \stackrel{\mathrm{W}}\rightarrow (\balpha_1\bTheta + \bZ^1,\ldots,\balpha_t\bTheta + \bZ^t,\bV,\bTheta),
    \end{gather*}
    where $(\bTheta,\bV) \sim \mu_{\bTheta,\bV}$ independent of $(\bZ^1,\ldots,\bZ^t)\sim \mathsf{N}(\bzero,\bT_{[1:t]})$, and $\stackrel{\mathrm{W}}\rightarrow$ denotes convergence in the Wasserstein metric of order 2 (see Appendix \ref{app:technical-lemmas}).
    Moreover, the matrices $(\bT_{s,t})_{s,t \geq 1}, (\balpha_t)_{t \geq 1}$ agree with those in Lemma \ref{lem:gfom-to-amp} when the functions $(f_t)_{t \ge 0}$, $(g_t)_{t \ge 1}$ also agree. 
\end{lemma}
We prove Lemma \ref{lem:se-mp-tree} in Appendix \ref{app:mp-on-tree}. 
Lemma \ref{lem:se-mp-tree} and the properties of convergence in the Wasserstein metric of order 2 (see Lemma \ref{lem:wass-pseudo-lipschitz-convergence}, Appendix \ref{app:technical-lemmas}) imply that for any message passing estimator $\hat \btheta_v^t$ and loss $\ell$, the risk $\E[\ell(\btheta_v,\hat \btheta_v^t)] = \E[\ell(\btheta_v,g_*(\ba_v^1,\ldots,\ba_v^t;\bv_v)]$ converges to $R_\ell(g_*,(\balpha_s),(\bT_{s,s'}))$, in agreement with the asymptotic error of the corresponding AMP estimator on the graph.

On the computation tree, 
we may lower bound this limiting risk by information-theoretic techniques, 
as we now explain. 
By induction, 
the estimate $\hat \btheta_v^t$ is a function only of observations corresponding to edges and nodes in the ball of radius $2t-1$ centered at $v$ on the computation tree. 
We denote the observations in this local neighborhood by $\calT_{v,2t-1}$.
We lower bound the risk of $\hat \btheta_v^t$ by the optimal risk of any measurable estimator, possibly intractable, which depends only on $\calT_{v,2t-1}$; we call this the \emph{local Bayes risk}.
The following lemma characterizes the local Bayes risk.
\begin{lemma}\label{lem:local-info-theory-lb}
    Consider a quadratically-bounded loss $\ell:\reals^{2r} \rightarrow \reals_{\geq 0}$.
    In the high-dimensional regression (resp.\ low-rank matrix estimation) model on the computation tree and under the asymptotics $n,p \rightarrow \infty$, $n/p \rightarrow \delta \in (0,\infty)$,
    \begin{equation*}
        \liminf_{n \rightarrow \infty} \inf_{\hat \btheta(\cdot)} \E[\ell(\btheta_v,\hat \btheta(\calT_{v,2t-1}))] \geq R^*,
    \end{equation*}
    where the infimum is over all measurable functions of $\calT_{v,2t-1}$, and $R^*$ is equal to the right-hand side of Eq.~\eqref{eq:hd-reg-lb} (resp.\ Eq.~\eqref{eq:lr-mat-lb}).
\end{lemma}
We prove Lemma \ref{lem:local-info-theory-lb} in Appendix \ref{app:info-theory-lb}. Combining Lemma \ref{lem:local-info-theory-lb} with the preceding discussion, we conclude that $R_{\ell}(g_*,(\balpha_s),(\bT_{s,s'})) \geq R_*$ for all Lipschitz functions $g_*$ and matrices $(\balpha_s),(\bT_{s,s'})$ generated
by the state evolution of some message passing or, equivalently, by some AMP algorithm.
The bounds \eqref{eq:hd-reg-lb} and \eqref{eq:lr-mat-lb} now follow.
Moreover, as we show in Appendix \ref{app:achieving-the-bound}, the bounds \eqref{eq:hd-reg-lb} and \eqref{eq:lr-mat-lb} are achieved by a certain AMP algorithm. 
The proof is complete.

\section*{Acknowledgements}

MC is supported by the National Science Foundation Graduate Research Fellowship under Grant No.~DGE – 1656518. AM was partially supported by NSF grants CCF-1714305, IIS-1741162 and by the ONR grant N00014-18-1-2729.

\bibliographystyle{alpha}
\bibliography{bibliography}

\newpage
\begin{appendices}

\section{Technical definitions and lemmas}
\label{app:technical-lemmas}

We collect some useful technical definitions and lemmas,
some of which we state without proof.
First, we recall the definition of the Wasserstein metric of order 2 on the space $\cuP_2(\reals^k)$:
\begin{equation*}
  W_2(\mu,\mu')^2 = \inf_{\Pi}\E_{(\bA,\bA')\sim\Pi}[\|\bA-\bA'\|^2]\,,
\end{equation*}
where the infimum is over couplings $\Pi$ between $\mu$ and $\mu'$. 
That is, $\Pi \in \cuP_2(\reals^k \times \reals^k)$ whose first and second marginals are $\mu$ (where a marginal here involves a block of $k$ coordinates).
It is well known that $W_2(\mu,\mu')$ is a metric on $\cuP_2(\reals^k)$ \cite[pg.~94]{Villani2008}.
When a sequence of probability distributions $\mu_n$ converges to $\mu$ in the Wasserstein metric of order 2, we write $\mu_n \stackrel{\mathrm{W}}\rightarrow \mu$.
We also write $\bA_n \stackrel{\mathrm{W}}\rightarrow \bA$ when $\bA_n \sim \mu_n$, $\bA \sim \mu$ for such a sequence.

\begin{lemma}\label{lem:pseudo-lipschitz}
    If $f:\reals^{r}\rightarrow \reals$ and $g:\reals^{r}\rightarrow \reals$ are pseudo-Lipschitz of order $k_1$ and $k_2$, respectively, then their product is pseudo-Lipschitz of order $k_1 + k_2$.
\end{lemma}

\begin{lemma}\label{lem:wass-pseudo-lipschitz-convergence}
    If a sequence of random vectors $\bX_n \stackrel{\mathrm{W}}\rightarrow \bX$, then for any pseudo-Lipschitz function $f$ of order $2$ we have $\E[f(\bX_n)] \rightarrow \E[f(\bX)]$.
\end{lemma}

\begin{lemma}\label{lem:strong-dist-convergence}
  Consider a sequence of random variables $(A_n,\bB_n) \stackrel{\mathrm{d}}\rightarrow (A,\bB)$ with values in $\reals \times \reals^k$ such that $(A_n,\bB_n) \stackrel{\mathrm{d}}\rightarrow (A,\bB)$ and $A_n \stackrel{\mathrm{d}}= A$ for all $n$.
  Then, for any bounded measurable function $f:\reals \times \reals^k \rightarrow \reals$ for which $\bb \mapsto f(a,\bb)$ is continuous for all $a$, we have $\E[f(A_n,\bB_n)] \rightarrow \E[f(A,\bB)]$.

  Further, for any function $\phi:\reals \times \reals^k \rightarrow \reals^{k'}$ (possibly unbounded) which is continuous in all but the first coordinate, we have $\phi(A_n,\bB_n) \stackrel{\mathrm{d}}\rightarrow \phi(A,\bB)$.
\end{lemma}

\begin{proof}[Proof of Lemma \ref{lem:strong-dist-convergence}]
  Without loss of generality, $f$ takes values in $[0,1]$.
  First we show that for any set $S \times I$ where $S \subset \reals$ is measurable and $I\subset \reals^k$ is a rectangle whose boundary has probability 0 under $\bB$ that 
  \begin{equation}\label{eq:rect-conv}
    \mu_{A_n,\bB_n}(S \times I) \rightarrow \mu_{A,\bB}(S \times I)\,.
  \end{equation}
  First, we show this is true for $S = K$ a closed set.
  Fix $\epsilon > 0$.
  Let $\phi_K^\epsilon:\reals \rightarrow [0,1]$ be a continuous function which is 1 on $K$ and 0 for all points separted from $K$ by distance $\epsilon$.
  Similarly define $\phi_I^\epsilon:\reals^k \rightarrow \reals$.
  Then
  \begin{align*}
    \E[\phi_K^\epsilon(A_n)\phi_I^\epsilon(\bB_n)] \geq \mu_{A_n,\bB_n}(K \times I) \geq \E[\phi_K^\epsilon(A_n)\phi_I^\epsilon(\bB_n)] - \epsilon - \mu_{\bB_n}(\mathsf{spt}(\phi_I^\epsilon) \setminus I)\,.
  \end{align*}
  Because the boundary of $I$ has measure 0 under $\mu_{\bB}$, we have $\lim_{\epsilon \rightarrow 0} \limsup_{n \rightarrow \infty} \mu_{\bB_n}(\mathsf{spt}(\phi_I^\epsilon) \setminus I) = 0$.
  Also, $\lim_{\epsilon \rightarrow 0} \lim_{n \rightarrow \infty} \E[\phi_K^\epsilon(A_n)\phi_I^\epsilon(\bB_n)] = \lim_{\epsilon \rightarrow 0} = \E[\phi_K^\epsilon(A)\phi_I^\epsilon(\bB)] = \mu_{A,\bB}(K \times I)$.
  Thus, taking $\epsilon \rightarrow \infty$ after $n \rightarrow \infty$, the previous display gives
  $\mu_{A_n,\bB_n}(K \times I) \rightarrow \mu_{A,\bB}(K \times I)$. 
  For $S = G$ an open set, we can show $\mu_{A_n,\bB_n}(G \times I) \rightarrow \mu_{A,\bB}(G \times I)$ by a similar argument:
  take instead $\phi_K^\epsilon$ to be 0 outside of $G$ and $1$ for all points in $G$ separated from the boundary by at least $\epsilon$, and likewise for $\phi_I^\epsilon$.
    By Theorem 12.3 of \cite{Billingsley2012ProbabilityMeasure}, we can construct $K \subset S \subset G$ such that $K$ is closed and $G$ is open, and $\mu_A(K) > \mu_A(S) - \epsilon$, $\mu_A(G) < \mu_A(S) + \epsilon$.
  The previous paragraph implies that
  \begin{align*}
    \mu_{A,\bB}(S \times I) - \epsilon &\leq \mu_{A,\bB}(K \times I) \leq \liminf_{n \rightarrow \infty} \mu_{A_n,\bB_n}(S \times I) 
    \\
    &\leq \limsup_{n \rightarrow \infty} \mu_{A_n,\bB_n}(S \times I) \leq \mu_{A,\bB}(G \times I) \leq \mu_{A,\bB}(S \times I) +\epsilon\,.
  \end{align*}  
  Taking $\epsilon \rightarrow 0$, we conclude \eqref{eq:rect-conv}.
  
  We now show \eqref{eq:rect-conv} implies the lemma.
  Fix $\epsilon > 0$.
  Let $M$ be such that $\P(\bB_n \in [-M,M]^k) > 1 - \epsilon$ for all $n$ and $\P(\bB \in [-M,M]^k) > 1 - \epsilon$, which we may do by tightness.
  For each $a$, let $\delta(a,\epsilon) = \sup\{ 0 < \Delta \leq M \mid \|\bb - \bb'\|_\infty < \Delta \Rightarrow |f(a,\bb) - f(a,\bb')| < \epsilon \}$.
  Because continuous functions are uniformly continuous on compact sets, the supremum is over a non-empty, bounded set.
  Thus, $\delta(a,\epsilon)$ is positive and bounded above by $M$ for all $a$.
  Further, $\delta(a,\epsilon)$ is measurable and non-decreasing in $\epsilon$.
  Pick $\delta_*$ such that $\P( \delta(A,\epsilon) < \delta_* ) < \epsilon$, which we may do because $\delta(a,\epsilon)$ is positive for all $a$.
  We can partition $[-M,M]^k$ into rectangles with side-widths smaller than $\delta_*$ such that the probability that $\bB$ lies on the boundary of one of the partitioning rectangles is 0. 
  Define $f_-(a,\bb) := \sum_{\iota} \indic{\bb \in I_\iota} \inf_{\bb' \in I_\iota} f(a,\bb')$ and $f_+(a,\bb) := \sum_{\iota} \indic{\bb \in I_\iota} \sup_{\bb' \in I_\iota} f(a,\bb')$, and note that on $\{\delta(a,\epsilon) < \delta^*\} \times [-M,M]^k$, we have $f_-(a,\bb) \leq f(a,\bb) \leq f_+(a,\bb)$ and $|f(a,\bb) - f_-(a,\bb)| < \epsilon$ and $|f(a,\bb) - f_+(a,\bb)| < \epsilon$.
  Thus, by the boundedness of $f$ and the high-probability bound on $\{\delta(a,\epsilon) < \delta^*\} \times [-M,M]^k$ 
  \begin{equation}\label{eq:f-pm-bound}
    \begin{gathered}
      \E[f_-(A_n,\bB_n)] - 2\epsilon < \E[f(A_n,\bB_n)] < \E[f_+(A_n,\bB_n)] + 2\epsilon\,,\\
      \E[f_-(A,\bB)] - 2\epsilon < \E[f(A,\bB)] < \E[f_+(A,\bB)] + 2\epsilon\,.
    \end{gathered}
  \end{equation}

  We show that $\E[f_-(A_n,\bB_n)] \rightarrow \E[f_-(A,\bB)]$.
  Fix $\xi > 0$.
  Take $0 = x_0 \leq \ldots \leq x_N = 1$ such that $x_{j+1} - x_j < \xi$ for all $j$.
  Let $S_{j\iota} = \{a \mid \inf_{\bb' \in I_\iota} f(a,\bb') \in [x_j,x_{j+1})\}$.
  Then
  \begin{equation*}
    \sum_{\iota,j} x_j\indic{a \in S_{j\iota},\bb \in I_\iota}+ \xi \geq f_-(a,\bb) \geq \sum_{\iota,j} x_j\indic{a \in S_{j\iota},\bb \in I_\iota}\,.
  \end{equation*}
  By \eqref{eq:rect-conv}, we conclude $\E[\sum_{\iota,j} x_j\indic{A_n \in S_{j\iota},\bB_n \in I_\iota}] \rightarrow \E[\sum_{\iota,j} x_j\indic{A \in S_{j\iota},\bB \in I_\iota}]$.
  Combined with the previous display and taking $\xi \rightarrow 0$, 
  we conclude that $\E[f_-(A_n,\bB_n)] \rightarrow \E[f_-(A,\bB)]$.
  Similarly, we may argue that $\E[f_+(A_n,\bB_n)] \rightarrow \E[f_+(A,\bB)]$.
  The first statment in the lemma now follows from taking $\epsilon \rightarrow 0$ after $n \rightarrow \infty$ in \eqref{eq:f-pm-bound}.

  The second statement in the lemma follows by observing that for any bounded continuous function $f:\reals^{k'}\rightarrow \reals$, we have that $f \circ \phi$ is bounded and is continuous in all but the first coordinate, so that we may apply the first part of the lemma to conclude $\E[f(\phi(A_n,\bB_n))] \rightarrow \E[f(\phi(A,\bB))]$.
\end{proof}

We will sometimes use the following alternative form of recursion \eqref{bamp-se-hd-reg} defining the lower bound in the high-dimensional regression model.
\begin{lemma}\label{lem:posterior-to-score}
  Consider a family, indexed by $x \in \reals$, of bounded probability densities $p(\cdot |x,u)$ with respect to some base measure $\mu_Y$.
  Then for $\ttau > 0$ and $\sigma \geq 0$ we have that
  \begin{equation*}
    \frac1{\ttau^2 }\E[\E[G_1|Y,G_0,U]^2] = \E_{G_0,Y}\left[\left(\frac{\mathrm{d}\phantom{b}}{\mathrm{d}x} \log \E_{G_1}p(Y|x + \sigma G_0 + \ttau G_1,U) \Big|_{x = 0}\right)^2\right]\,,
  \end{equation*}
  where $G_0,G_1 \stackrel{\mathrm{iid}}\sim \normal(0,1)$ and $Y | G_0,G_1,U$ had density $p(\cdot| \sigma G_0 + \ttau G_1,U)$ with respect to $\mu_Y$.
  In particular, the derivatives exist.
  (In this case, we may equivalently generate $Y = h( \sigma G_0 + \ttau G_1 , \bW)$ for $(\bW,U) \sim \mu_{\bW,U}$).
\end{lemma}

The preceding lemma applies, in particular, for $p$ as in \textsf{R4}. 
It then provides an alternative form of the second equation in recursion \eqref{bamp-se-hd-reg}.

\begin{proof}[Lemma \ref{lem:posterior-to-score}]
We have
\begin{equation*}
  \E_{G_1}p(Y|x + \sigma G_0+ \ttau G_1,U) = \int p(Y| \sigma G_0 + s ,U ) \frac1{\sqrt{2\pi}\ttau}e^{-\frac1{2\ttau^2} (s-x)^2} \de g\,,
\end{equation*}
so that 
\begin{align*}
  \frac{\mathrm{d}\phantom{b}}{\mathrm{d}x} \E_{G_1}p(Y|x + \sigma G_0+ \ttau G_1,U) &= \frac1{\ttau^2}\int p(Y| \sigma G_0 + s , U) \frac{(s-x)}{\sqrt{2\pi}\ttau}e^{-\frac1{2\ttau^2} (s-x)^2} \de g\,,
\end{align*}
where the boundedness of of $p$ allows us to exchange integration and differentition.
Thus, 
\begin{equation*}
  \frac{\mathrm{d}\phantom{b}}{\mathrm{d}x} \log \E_{G_1}p(Y|x + \sigma G_0+ \ttau G_1 ,U ) = \frac1{\ttau}\E[G_1|Y,G_0,U]\,.
\end{equation*}
The result follows.
\end{proof}

Finally, we collect some results on the Bayes risk with respect to quadratically-bounded losses $\ell: \reals^k \times \reals^k \rightarrow \reals_{\geq 0}$.
Recall that quadratically-bounded means that $\ell$ is pseudo-Lipschitz of order 2 and also satisfies
\begin{equation}\label{eq:loss-growth-condition}
	|\ell(\bvartheta,\bd) - \ell(\bvartheta',\bd)| 
	\leq 
	C\left(1 + \sqrt{\ell(\bvartheta,\bd)} + \sqrt{\ell(\bvartheta',\bd)})\right)\|\bvartheta - \bvartheta'\|\,.
\end{equation}
We consider a setting $(\bTheta,\bV) \sim \mu_{\bTheta,\bV} \in \cuP_2(\reals^k\times \reals^k)$, $\bZ \sim \normal(0,\bI_k)$ independent and $\tau,K,M \geq 0$.
Define $\bTheta^{(K)}$ by $\Theta^{(K)}_i = \Theta_i \indic{|\Theta_i| \leq K}$.
Denote by $\mu_{\bTheta^{(K)},\bV}$ the joint distribution of $\bTheta^{(K)}$ and $\bV$,
and by $\mu_{\bTheta^{(K)}|\bV} : \reals^k \times \mathcal{B} \rightarrow [0,1]$ a regular conditional probability distribution for $\bTheta^{(K)}$ conditioned on $\bV$.
Define the posterior Bayes risk
\begin{equation}\label{eq:conditional-bayes-risk}
	R(\by,\tau,\bv,K, M) := \inf_{\|\bd\|_\infty \leq M} \int \frac1Z \ell(\bvartheta,\bd) e^{-\frac1{2\tau^2}\|\by - \bvartheta\|^2} \mu_{\bTheta^{(K)}|\bV}(\bv,\de \bvartheta)\,,
\end{equation}
where $Z = \int   e^{-\frac1{2\tau^2}\|\by - \bvartheta\|^2} \mu_{\bTheta^{(K)}|\bV}(\bv,\de \bvartheta)$ is a normalization constant.
It depends on $\by,\tau,\bv,K$. 
When required for clarity, we write $Z(\by,\tau,\bv,K)$.

\begin{lemma}\label{lem:properties-of-bayes-risk}
    The following properties hold for the Bayes risk with respect to pseudo-Lipschitz losses of order 2 satsifying \eqref{eq:loss-growth-condition}.
    \begin{enumerate}[(a)]
        \item 
        For any $\tau,K,M$, with $K,M$ possibly equal to infinity, the Bayes risk is equal to the expected posterior Bayes risk.
        That is,
        \begin{equation}\label{eq:bayes-is-exp-post}
            \inf_{\hat \btheta(\cdot)} \E[\ell(\bTheta^{(K)},\hat \btheta(\bTheta^{(K)} + \tau \bz,\bV)] = \E[R(\bY^{(K)},\tau,\bV,K,M)]\,,
        \end{equation}
        where $\bY^{(K)} = \bTheta^{(K)} + \tau \bZ$ with $\bZ \sim \normal(\bzero,\bI_k)$ independent of $\bTheta^{(K)}$ and the infimum is taken over all measurable functions $(\reals^k)^2 \rightarrow [-M,M]^k$.
        Moreover,
    	\begin{align}
    		\E[R(\bTheta^{(K)} + \tau \bZ,\tau,\bV,K,\infty)]
    		&= \lim_{M \rightarrow \infty} \E[R(\bTheta^{(K)} + \tau \bZ,\tau,\bV,K,M)]\,.\label{eq:risks-trunc}
    	\end{align}
        \item 
    	For a fixed $K < \infty$, 
    	the posterior Bayes risk is bounded: $R(\by,\tau,\bv,K,M) \leq \bar R(K)$ for some function $\bar R$ which does not depend on $\by,\tau,\bv,M$. 
    	Further, for $K < \infty$ the function $(\by,\tau) \mapsto R(\by,\tau,\bv,K,M)$ is continuous on $\reals^k \times \reals_{>0}$.
    	\item 
    	The Bayes risk is jointly continuous in truncation level $K$ and noise variance $\tau$.
    	This is true also at $K = \infty$:
    	\begin{align}\label{eq:risks-trunc-limit}
			\E[R(\bTheta^{(K)} + \tau \bZ,\tau,\bV,K,\infty)] &= \lim_{\substack{K \rightarrow \infty\\\tau'\rightarrow \tau}} \E[R(\bTheta^{(K)} + \tau' \bZ,\tau',\bV,K,\infty)]\,,
    	\end{align}
    	where the limit holds for any way of taking $K,\tau'$ to their limits (ie., sequentially or simultaneously).
	\end{enumerate}
\end{lemma}

\begin{proof}[Proof of Lemma \ref{lem:properties-of-bayes-risk}(a)]
	For any measurable $\hat \btheta: \reals^k \times \reals^k \rightarrow [-M,M]^k$,
	\begin{align}
		\E[\ell(\bTheta^{(K)},\hat \btheta(\bTheta^{(K)} + \tau \bZ,\bV))] &= \E[\E[\ell(\bTheta^{(K)},\hat \btheta(\bTheta^{(K)} + \tau \bZ,\bV))|\bTheta^{(K)} + \tau \bZ,\bV]] \nonumber \\
		&\geq \E[R(\bTheta^{(K)} + \tau \bZ,\tau,\bV,K,M)]\,.\label{eq:BayesLB}
	\end{align}
	For $M < \infty$, equality obstains.
	Indeed,
	we may define
	\begin{equation}\label{eq:achieve-bayes}
	  \hat \btheta^{(M)}(\by,\bv;\tau) = \arg\min_{\|\bd\|_\infty \leq M} \int \frac1{Z} \ell(\bvartheta,\bd) e^{-\frac1{2\tau^2}\|\by - \bvartheta\|^2}\mu_{\bTheta^{(K)}|\bV}(\bv,\de \bvartheta)\,, 
	\end{equation}
	because the integral is continuous in $\bd$ by dominated convergence.
	Then $\E[\ell({\bTheta^{(K)}},\hat \btheta^{(M)}( \bY ,\bV ; \tau))] = \E[R(\bY ,\tau,\bV,K,M)]$
	when $\bY = \bTheta^{(K)} + \tau \bZ$.
	Observe $R(\by,\tau,\bv,K,M) \downarrow R(\by,\tau,\bv,K,\infty)$ as $M \rightarrow \infty$ with the other arguments fixed.
	Thus, $\E[R(\bTheta^{(K)} + \tau \bZ,\tau,\bV,K,M)] \downarrow \E[R(\bTheta^{(K)} + \tau \bZ,\tau,\bV,K,\infty)]\,$ in this limit.
	Because $\E[R(\bTheta^{(K)} + \tau \bZ,\tau,\bV,K,\infty)]$ is a lower bound on the Bayes risk at $M = \infty$ by \eqref{eq:BayesLB} and we may achieve risk arbitrarily close to this lower bound by taking $M \rightarrow \infty$ in \eqref{eq:achieve-bayes},
	we conclude \eqref{eq:bayes-is-exp-post} at $M = \infty$ as well.
\end{proof}

\begin{proof}[Proof of Lemma \ref{lem:properties-of-bayes-risk}(b)]
	The quantity $R(\by,\tau,\bv,K,M)$ is non-negative.
	Define $\bar R(K) = \max_{\|\bvartheta\|_\infty \leq K} \ell(\bvartheta,\bzero)$. 
	Observe that $R(\by,\tau,\bv,K,M) \leq \bar R(K)$ for all $\by,\tau,\bv,K,M$.
	Let $p^*(\btheta|\by,\tau,\bv,K) = \frac1Z e^{-\frac1{2\tau^2}\|\by - \bvartheta\|^2}$.
	For any fixed $\bd$, we have
	\begin{align*}
		\left\|\nabla_{\by} \int \ell(\bvartheta,\bd) p^*(\bvartheta|\by,\tau,\bv,K) \mu_{\bTheta^{(K)}|\bV}(\bv,\mathrm{d}\bvartheta)\right\| &\leq \int \ell(\bvartheta,\bd) p^*(\bvartheta|\by,\tau,v) \left\|\nabla_{\by}\log p^*(\bvartheta|\by,\tau,\bv)\right\| \mu_{\bTheta^{(K)}|\bV}(\bv,\mathrm{d}\bvartheta) \\
		&\leq \frac{2K\sqrt{k}}{\tau^2} \int \ell(\bvartheta,\bd) p^*(\bvartheta|\by,\tau,\bv) \mu_{\bTheta^{(K)}|\bV}(\bv,\mathrm{d}\bvartheta)\,,
	\end{align*}
	where we have used that $\|\nabla_{\by}\log p^*(\bvartheta|\by,\tau,\bv)\| = \frac1{\tau^2}(\bvartheta - \E_{\bTheta^{(K)}}[\bTheta^{(K)}]) \leq 2K\sqrt{k}/\tau^2$,
	and the expectation is taken with respect to $\bTheta^{(K)}$ having density $p^*(\bvartheta|\by,\tau,\bv)$ with respect to $\mu_{\bTheta^{(K)}|V}(\bv,\cdot)$.
	Thus, for fixed $\tau,\bd,\bv$ satisfying $ \int \ell(\bvartheta,\bd) p^*(\bvartheta|\by,\tau,\bv) \mu_{\Theta|V}(\bv,\mathrm{d}\bvartheta) \leq \bar R$, the function $\by \mapsto \int \ell(\bvartheta,\bd) p^*(\bvartheta|\by,\tau,\bv) \mu_{\Theta|V}(\bv,\mathrm{d}\bvartheta)$ is $2K\sqrt{k}\bar R/\tau^2$-Lipschitz.
	Because the infimum defining $R$ can be taken over such $\bd$ and infima retain a uniform Lipschitz property, 
	$R(\by,\tau,\bv,K,M)$ is $2K\sqrt{k}\bar R/\tau^2$-Lipschitz in $\by$ for fixed $\tau,\bv,K,M$.
	By a similar argument, we can establish that $R(\by,\tau,\bv,K,M)$ is $2(K^2k + 2\|\by\|K\sqrt{k})/\bar \tau^3$-Lipschitz in $\tau$ on the set $\tau > \bar \tau$ for any fixed $\bar \tau > 0$ and any fixed $\by,\bv,K,M$.
	We conclude $(\by,\tau) \mapsto R(\by,\tau,\bv,K,M)$ is continuous on $\reals^k \times \reals_{>0}$.
	Lemma \ref{lem:properties-of-bayes-risk}(b) has been shown.
\end{proof}

\begin{proof}[Proof of Lemma \ref{lem:properties-of-bayes-risk}(c)]
	Finally, we prove \eqref{eq:risks-trunc-limit}.
	For any $K > 0$, we may write\footnote{Precisely, for any regular conditional probability distribution $\mu_{\bTheta|\bV}$ for $\bTheta$ given $\bV$, this formula gives a valid version of a regular conditional probability distribution for $\bTheta^{(K)}$ given $\bV$. 
	We assume we use this version throughout our proof.}
	\begin{equation}\label{eq:reg-cond-decomposition}
		\mu_{\bTheta^{(K)}|\bV}(\bv,\cdot) 
		= 
		\mu_{\bTheta|\bV}(\bv,\cdot)|_{[-K,K]^k} + \mu_{\bTheta|\bV}(\bv,([-K,K]^k)^c) \delta_{\bzero}(\cdot)\,.
	\end{equation}
	Choose $\bar K,\epsilon' > 0$ such that $|\tau' - \tau| < \epsilon'$ implies
	\begin{equation*}
		\int_{[-\bar K,\bar K]^k} \frac1{Z(\by,\tau',\bv,\infty)}e^{-\frac1{2{\tau'}^2}\|\by - \bvartheta\|^2}\mu_{\bTheta|\bV}(\bv,\de \bvartheta) \geq \frac12 \int \frac1{Z(\by,\tau,\bv,\infty)} e^{-\frac1{2\tau^2}\|\by - \bvartheta\|^2}\mu_{\bTheta|\bV}(\bv,\de \bvartheta)\,.
	\end{equation*}
	Fix $\epsilon > 0$ and $K' > K > 0$ with $K'$ possibly equal to infinity.
	By \eqref{eq:conditional-bayes-risk}, we may choose $\bd^*$ such that  
	\begin{equation}\label{eq:bayes-estimate-approx}
		\int \frac1{Z(\by,\tau,\bv,K)} \ell(\bvartheta,\bd^*) e^{-\frac1{2\tau^2}\|\by - \bvartheta\|^2}\mu_{\bTheta^{(K)}|\bV}(\bv,\de \bvartheta) \leq (1 + \epsilon)R(\by,\tau,\bv,K,\infty)\,.
	\end{equation}
	By the definition of $\bar K$, there exists $\bvartheta^* \in [-\bar K,\bar K]^k$ such that
	\begin{equation*}
		\ell(\bvartheta^*,\bd^*) \leq 2(1+\epsilon) R(\by,\tau,\bv,K,\infty)\,.
	\end{equation*}
	By \eqref{eq:loss-growth-condition}, we conclude that 
	\begin{align*}
		\ell(\bvartheta , \bd^*) 
		&\leq 
		C\left(1 +  \sqrt{2(1+\epsilon)R(\by,\tau,\bv,K,\infty)} + \sqrt{\ell(\bvartheta,\bd^*)}\right)\|\bvartheta - \bvartheta^*\|\,,
	\end{align*}
	whence 
	\begin{equation}\label{eq:ell-bayes-risk-bound}
		\ell(\bvartheta,\bd^*) \leq \left(1 +  \sqrt{2(1+\epsilon)R(\by,\tau,\bv,K,\infty)} + 3C\|\bvartheta - \bvartheta^*\|\right)^2\,.
	\end{equation}
	%
	Then
	\begin{align*}
		&\left|\int \ell(\bvartheta,\bd^*) e^{-\frac1{2{\tau'}^2}\|\by - \bvartheta\|^2}\mu_{\bTheta^{(K')}|\bV}(\bv,\de \bvartheta) - \int \ell(\bvartheta,\bd^*) e^{-\frac1{2\tau^2}\|\by - \bvartheta\|^2}\mu_{\bTheta^{(K)}|\bV}(\bv,\de \bvartheta)\right| \\
		&\qquad \leq \left|\int \ell(\bvartheta,\bd^*) e^{-\frac1{2{\tau'}^2}\|\by - \bvartheta\|^2}\mu_{\bTheta^{(K')}|\bV}(\bv,\de \bvartheta) - \int \ell(\bvartheta,\bd^*) e^{-\frac1{2{\tau'}^2}\|\by - \bvartheta\|^2}\mu_{\bTheta^{(K)}|\bV}(\bv,\de \bvartheta)\right| \\
    		&\qquad\qquad + \left|\int \ell(\bvartheta,\bd^*) e^{-\frac1{2{\tau'}^2}\|\by - \bvartheta\|^2}\mu_{\bTheta^{(K)}|\bV}(\bv,\de \bvartheta) - \int \ell(\bvartheta,\bd^*) e^{-\frac1{2{\tau}^2}\|\by - \bvartheta\|^2}\mu_{\bTheta^{(K)}|\bV}(\bv,\de \bvartheta)\right| \\
		&\qquad \leq \int_{([-K,K]^k)^c} \ell(\bvartheta , \bd^* ) e^{-\frac1{2{\tau'}^2}\|\by - \bvartheta\|^2} \mu_{\bTheta|\bV}(\bv , \de \bvartheta) + \ell(\bzero,\bd^*)e^{-\frac1{2{\tau'}^2}\|\by\|^2} \mu_{\bTheta|\bV}(\bv,([-K,K]^k)^c)\\
    		&\qquad\qquad + \left|\int \ell(\bvartheta,\bd^*) e^{-\frac1{2{\tau'}^2}\|\by - \bvartheta\|^2}\mu_{\bTheta^{(K)}|\bV}(\bv,\de \bvartheta) - \int \ell(\bvartheta,\bd^*) e^{-\frac1{2{\tau}^2}\|\by - \bvartheta\|^2}\mu_{\bTheta^{(K)}|\bV}(\bv,\de \bvartheta)\right| \\
		&\qquad\leq \xi(K,\tau')(1 + R(\by,\tau,\bv,K,\infty))\,,
	\end{align*}
	for some $\xi(K,\tau')\rightarrow 0$ as $K\rightarrow \infty$, $\tau' \rightarrow \tau$ because the conditional measure $\mu_{\bTheta|\bV}(\bv,\cdot)$ has finite second moment and $\ell$ is bounded by \eqref{eq:ell-bayes-risk-bound}. 
	Then, by \eqref{eq:bayes-estimate-approx},
	\begin{align*}
		Z(\by,\tau',\bv,K')R(\by,\tau',\bv,K',\infty) &\leq \int \ell(\bvartheta,\bd^*) e^{-\frac1{2\tau^2}\|\by - \bvartheta\|^2}\mu_{\bTheta^{(K')}|\bV}(\bv,\de \bvartheta) \\
		&\leq (1 + \epsilon)Z(\by,\tau,\bv,K)R(\by,\tau,\bv,K,\infty) +  \xi(K,\tau')(1 + R(\by,\tau,\bv,K,\infty))\,.
	\end{align*}
	By dominated convergence, we have that $Z(\by,\tau',\bv,K') \rightarrow Z(\by,\tau,\bv,\infty)$ as $\tau'\rightarrow \tau, K'\rightarrow \infty$.
	Also, $\bar R(K) = \max_{\|\bvartheta\|_\infty \leq K} \ell (\bvartheta,\bzero)$ cannot diverge at finite $K$.
	Thus, applying the previous display with $K,\epsilon$ fixed allows us to conclude that $R(\by,\tau,\bv,K',\infty)$ is uniformly bounded over $K' > K$ and $\tau'$ in a neighborhood of $\tau$.
	Then, taking $K' = \infty$ and $K \rightarrow \infty$, $\tau' \rightarrow \tau$ followed by $\epsilon \rightarrow 0$ allows us to conclude that
	\begin{align}\label{eq:conditional-bayes-risk-convergence}
		\lim_{\substack{K \rightarrow \infty\\ \tau'\rightarrow \tau}} R(\by,\tau',\bv,K,\infty) = R(\by,\tau,\bv,\infty,\infty)\,.
	\end{align}
	for every fixed $\by,\bv$.
	Moreover,
	\begin{align*}
		R(\by,\tau,\bv,K,M) &= \inf_{\|\bd\|_\infty \leq M} \int \frac1{Z} \ell(\bvartheta,\bd) e^{-\frac1{2\tau^2}\|\by - \bvartheta\|^2}\mu_{\bTheta^{(K)}|\bV}(\bv,\de \bvartheta) \\
		&\leq \int \frac1{Z} \ell(\bvartheta,\bzero) e^{-\frac1{2\tau^2}\|\by - \bvartheta\|^2} \mu_{\bTheta^{(K)}|\bV}(\bv,\de \bvartheta) \\
		&\leq \int \frac1Z C(1 + \|\bTheta^{(K)}\|^2) e^{-\frac1{\tau^2}(\by - \bvartheta)^2} \mu_{{\bTheta^{(K)}}|\bV}(\bv,\de \bvartheta)\\
		&= C(1 + \E[\|{\bTheta^{(K)}}\|^2 | {\bTheta^{(K)}} + \tau \bG = \by,\bV=\bv])\,.
	\end{align*}
	Thus, $R(\bTheta^{(K)} + \tau\bZ,\tau,\bV,K,M)$ is uniformly integrable as we vary $\tau,K,M$.	
	Because the total variation distance between $(\bTheta^{(K)} + \tau' \bZ,\bV)$ and $(\bTheta + \tau \bZ,\bV))$ goes to 0 as $K\rightarrow \infty$ and $\tau' \rightarrow \tau$, 
	for any discrete sequence $(K,\tau') \rightarrow (\infty,\tau)$,
    there exists a probability space containing variables $\tilde \bY^{(K,\tau')}, \tilde \bV, \tilde \bY$ such that $(\tilde \bY^{(K,\tau')},\tilde \bV) = (\tilde \bY,\tilde \bV)$ eventually.
    Thus, Eq.~\eqref{eq:conditional-bayes-risk-convergence} and uniform integrability imply \eqref{eq:risks-trunc-limit}.
\end{proof}

\section{Proof for reduction from GFOMs to AMP (Lemma \ref{lem:gfom-to-amp})}
\label{app:proofs-gfom-to-amp}

In this section, we prove Lemma \ref{lem:gfom-to-amp}.

\subsection{A general change of variables}

For any GFOM \eqref{gfom}, there is a collection of GFOMs to which it is, up to a change of variabes, equivalent.
In this section, we specify these GFOMs and the corresponding changes of variables.

The change of variables is determined by a collection of $r \times r$ matrices $(\bxi_{t,s})_{t\ge 1,1\le s \le t}$, $(\bzeta_{t,s})_{t \ge 1, 0 \leq s < t}$.
We will often omit subscripts outside of the parentheses.
Define recursively the functions $(f_t)_{t\geq 0}$, $(\phi_t)_{t\geq 1}$
\begin{subequations}\label{cov}
\begin{equation}\label{cov-f-phi}
\begin{aligned}
    f_t(\bb^1,\ldots,\bb^t;y,\bu) &= F_t^{(1)}(\phi_1(\bb^1;y,\bu),\ldots,\phi_t(\bb^1,\ldots,\bb^t;y,\bu);y,\bu)\\
    \phi_t(\bb^1,\ldots,\bb^t;y,\bu) &= \bb^t + \sum_{s=0}^{t-1}  f_s(\bb^1,\ldots,\bb^s;y,\bu) \bzeta_{t,s}^\sT \\
    &\qquad\qquad + G_t^{(2)}(\phi_1(\bb^1;y,\bu),\ldots,\phi_{t-1}(\bb^1,\ldots,\bb^{t-1};y,\bu);y,\bu),
\end{aligned}
\end{equation} 
initialized by $f_0(y,\bu) = F_0^{(1)}(y,\bu)$ (here $\bb^s,\bu \in \reals^r$), and define recursively the functions $(g_t)_{t \geq 1}$, $(\varphi_t)_{t \geq 1}$
\begin{equation}\label{cov-g-varphi}
\begin{aligned}
    \varphi_{t+1}(\ba^1,\ldots,\ba^{t+1};\bv) &= \ba^{t+1} + \sum_{s=1}^t g_s(\ba^1,\ldots,\ba^{t+1};\bv) \bxi_{t,s}^\sT \\
    &\qquad\qquad + F_t^{(2)}( \phi_1( \ba^1 ;\bv) , \ldots , \phi_t(\ba^1,\ldots, \ba^t ;\bv) ; \bv ),\\
    \qquad g_t(\ba^1,\ldots,\ba^t;\bv) &= G_t^{(1)}(\varphi_1(\ba^1;\bv),\ldots,\varphi_t(\ba^1,\ldots,\ba^t;\bv);\bv),
\end{aligned}
\end{equation}
initialized by $\varphi_1(\ba^1;\bv) = a^1 + F_0^{(2)}(\bv)$ (here $\ba^s,\bv \in \reals^r$).
\end{subequations}
Algebraic manipulation verifies that the iteration
\begin{equation}\label{gfom-post-cov}
\begin{aligned}
    \ba^{t+1} &= \bX^\mathsf{T} f_t(\bb^1,\ldots,\bb^t;y,\bu) - \sum_{s = 1}^t g_s(\ba^1,\ldots,\ba^s;\bv) \bxi_{t,s}^\sT ,\\
    \bb^t &= \bX g_t(\ba^1,\ldots,\ba^t;\bv) - \sum_{s = 0}^{t-1}  f_s(\bb^1,\ldots,\bb^s;y,\bu) \bzeta_{t,s}^\sT
\end{aligned}
\end{equation}
initialized by $\ba^1 = \bX^\mathsf{T} f_{0}(y,\bu)$ generates sequences $(\ba^t)_{t\geq 1}$, $(\bb^t)_{t\geq 1}$ which satisfy
\begin{align*}
    \bv^t = \varphi_t(\ba^1,\ldots,\ba^t;\bv),\quad t \geq 1,\\
    \bu^t = \phi_t(\bb^1,\ldots,\bb^t;y,\bu),\quad t \geq 1.
\end{align*}
Thus, $(\bxi_{t,s})$, $(\bzeta_{t,s})$ index a collection of GFOMs which, up to a change of variables, are equivalent. 

\subsection{Approximate message passing and state evolution}

We call the iteration \eqref{gfom-post-cov} an approximate message passing algorithm if the matrices $(\bxi_{t,s}),(\bzeta_{t,s})$ satisfy a certain model-specific recursion involving the functions $f_t,g_t$.
The state evolution characterization of the iterates (see Eq.~\eqref{eq:se}) holds whenever the matrices $\bxi_{t,s}$, $\bzeta_{t,s}$ satisfy this recursion.
In this section, we specify this recursion and the parameters $(\balpha_s),(\bT_{s,s'})$ in both the high-dimensional regression and low-rank matrix estimation models.

\subsubsection{High-dimensional regression AMP}

In the high-dimensional regression model, $r = 1$ and $\xi_{t,s}$, $\zeta_{t,s}$, $\alpha_t$, and $T_{s,s'}$ will be scalars (hence, written with non-bold font).
The recursion defining $\xi_{t,s}$, $\zeta_{t,s}$ also defines $(\alpha_t)$, $T_{s,s'}$ as well as a collection of scalars $(\Sigma_{s,t})_{s,t \geq 0}$ which did not appear in the statement of Lemma \ref{lem:gfom-to-amp}.
The recursion, whose lines are implemented in the order in which they appear, is
\begin{equation}\label{amp-scalars-hd-reg}
    \begin{split}
        \xi_{t,s} &= \E[\partial_{B^s} f_t(B^1,\ldots,B^t;h(B^0,W),U)],\quad 1 \leq s \leq t, \\
        \alpha_{t+1} &= \E[\partial_{B^0} f_t(B^1,\ldots,B^t;h(B^0,W),U)],\\
        T_{s+1,t+1} &= \E[f_s(B^1,\ldots,B^s;h(B^0,W),U)f_t(B^1,\ldots,B^t;h(B^0,W),U)], \quad 0 \leq s \leq t,\\        
        \zeta_{t,s} &= \frac1\delta \E[\partial_{Z^{s+1}} g_t( \alpha_1\Theta + Z^1,\ldots, \alpha_t\Theta + Z^t;V)],\quad 0 \leq s \leq t-1,\\
        \Sigma_{0,t} &= \frac1\delta \E[\Theta g_t(\alpha_1 \Theta + Z^1,\ldots, \alpha_t\Theta + Z^t;V)],\\
        \Sigma_{s,t} &= \frac1\delta \E[g_s(\alpha_1\Theta + Z^1,\ldots,\alpha_t\Theta + Z^s;V)g_t(\alpha_1\Theta + Z^1,\ldots,\alpha_t \Theta + Z^t;V)],\quad 1 \leq s \leq t,
    \end{split}
\end{equation}
where $\Theta \sim \mu_\Theta$, $U \sim \mu_U$, $V \sim \mu_V$, $W \sim \mu_W$,  $(B^0,\ldots,B^t) \sim \mathsf{N}(\bzero, \bSigma_{[0:t]})$, $(Z^1,\ldots,Z^t) \sim \mathsf{N}(\bzero,\bT_{[1:t]})$, all independent.
We initialize just before the second line with $\Sigma_{0,0} = \E[\Theta^2]$.

Eq.~\eqref{eq:se} for $(\alpha_s),(T_{s,s'})$ defined in this way is a special case of Proposition 5 of \cite{javanmard2013state}, as we now explain.
We fix iteration $t$ design an algorithm that agrees, after a change of variables, with iteration \eqref{eq:amp} up to iteration $t$ and to which we can apply the results of \cite{javanmard2013state}.
Because we take $n,p\rightarrow \infty$ before $t \rightarrow \infty$, this establishes the result.

We view the first $t$ iterations of \eqref{eq:amp} as acting on matrices $\tilde\ba^s\in \reals^{p\times (t+1)}$ and $\tilde\bb^s \in \reals^{n \times (t+1)}$ as follows.
Define $\tilde \ba^s$ to be the matrix whose first column is $\btheta$ and whose $i^\text{th}$ column is $\ba^{i-1}$ for $2 \leq i \leq s+1$ and is $\bzero$ for $i > s + 1$; 
define $\tilde \bb^s$ to be the matrix whose first column is $\bX \btheta$ and whose $i^\text{th}$ column is $\bb^{i=1}$ for $2 \leq i \leq s+1$ and is $\bzero$ for $i > s+1$.
The following change of variables transforms \eqref{eq:amp} into equations (28) and (29) of Proposition 5 in \cite{javanmard2013state}. Our notation is on the right and is separated from the notation of \cite{javanmard2013state} by the symbol ``$\leftarrow$''.
\begin{gather*}
  \tilde A \leftarrow X,\\
  u^s(i) \leftarrow
    \begin{cases}
      \bX \btheta & i = 1,\\
      \bb^{i-1} & 2 \leq i \leq s+1,\\
      \bzero & \text{otherwise},
    \end{cases}
  \qquad 
  \text{and}
  \qquad 
  v^s(i) \leftarrow 
    \begin{cases}
      \ba^{i-1} - \alpha_{i-1}\btheta & 2 \leq i \leq s+1,\\
      \bzero & \text{otherwise},
    \end{cases}\\
  y(i) \leftarrow 
    \begin{cases}
      \bv & i = 1,\\
      \btheta & i = 2, \\
      \bzero & \text{otherwise},
    \end{cases}
  \qquad
  \text{and}
  \qquad 
  w(i) \leftarrow 
    \begin{cases}
      \bu & i = 1, \\
      \bw & i = 2, \\
      \bzero & \text{otherwise},
    \end{cases}\\
  \widehat e(v,y;s)(i) \leftarrow 
    \begin{cases}
      y(2)  & i = 1, \\
      g_{i-1}(v(2) + \alpha_1 y(2) , \ldots ,  v(i+1) + \alpha_i y(2) ; y(1) )& 2 \leq i \leq s + 1,\\
      \bzero & \text{otherwise},
    \end{cases}\\
  \widehat h(u,w;s)(i) \leftarrow  
    \begin{cases}
      f_{i-1}(u(2),\ldots,u(i+1);h(u(1),w(2)),w(1)), & 1 \leq i \leq s + 1,\\
      \bzero & \text{otherwise},
    \end{cases}
\end{gather*}
where the ``$(i)$'' notation indexes columns of a matrix.
The Onsager correction coefficients $(\xi_{t,s})$ and $(\zeta_{t,s})$ correspond, after a change of variables, 
to entries in the matrices $\mathsf{D}_s$ and $\mathsf{B}_s$ in \cite{javanmard2013state}.
\begin{gather*}
  (\mathsf{D}_s)_{i,j} = \E[\partial_{u(j)} \widehat h(U,W;s)] \leftarrow 
  \begin{cases}
    \E[\partial_{B^{j-1}}f_{i-1}(B^1,\ldots,B^i;h(B^0,W),U)] & 1 \leq j - 1 \leq i \leq s+1, \\
    0 & \text{otherwise},
  \end{cases},\\
  (\mathsf{B}_s)_{i,j} = \frac1\delta \E[\partial_{v(j)}\widehat e(V,Y;i)] \leftarrow 
  \begin{cases}
    0 & i = 1 \text{ or } j = 1 ,\\
    \frac1\delta\E[\partial_{Z^{j-1}} g_i(\alpha_1 \Theta +  Z^1,\ldots,\alpha_i \Theta + Z^i ; V)] & 2 \leq j \leq i+1 \leq s+2,\\
    0 & \text{otherwise}.
  \end{cases}
\end{gather*}
The Onsager coefficients and state evolution coefficients are arrived at through the change of variables:
\begin{gather*}
   (\mathsf{B}_s)_{s+1,s'+2}\leftarrow \zeta_{s,s'} , \;\;\;\; (\mathsf{D}_s)_{s+1,s'+1} \leftarrow \xi_{s,s'} \;\;\;\;(\mathsf{D}_{s})_{s+1,1} \leftarrow  \alpha_s .
\end{gather*}
We remark that in \cite{javanmard2013state} the quantities $(\mathsf{B}_s)_{s+1,s'+2}$, $(\mathsf{D}_s)_{s+1,s'+1} $, and $(\mathsf{D}_{s})_{s+1,1}$ are empirical averages.
Because they concentration well on their population averages, we may replace them with their population averages, as we do here, without affecting the validity of state evolution. 
This observation is common in the AMP literature: see, for example, the relationship between Theorem 1 and Corollary 2 of \cite{Berthier2017StateFunctions}.
The state evolution matrices now correspond to 
\begin{align*}
  \E[V^{s+1}(s+1)V^{s+1}(s'+1)] &= \E[\widehat h(U,W;s)(s)\widehat h(U,W;s)(s')]\\
  &\leftarrow \E[f_{s-1}(B^1,\ldots,B^{s-1};h(B^0,W),U)f_{s'-1}(B^1,\ldots,B^{s'-1};h(B^0,W),U)]\\
  & = T_{s,s'},\\
  \E[U^{s+1}(s+1)U^{s+1}(s'+1)] &= \frac1\delta\E[\widehat e(V,Y;s+1)(s+1) \widehat e(V,Y;s+1)(s'+1)] \\
  &\leftarrow \frac1\delta\E[g_s(\alpha_1 \Theta + Z^1,\ldots,\alpha_s \Theta + Z^s;V)g_{s'}(\alpha_1 \Theta + Z^1,\ldots,\alpha_{s'} \Theta + Z^{s'};V)]\\
  &= \Sigma_{s,s'}.
\end{align*}
From these changes of variables, Eq.~\eqref{eq:se} holds in the high-dimensional regression model from Theorem 1 and Proposition 5 of \cite{javanmard2013state}.

\subsubsection{Low-rank matrix estimation AMP}

In the low-ank matrix estimation model, 
the recrusion defining $(\bx_{t,x})$, $(\bzeta_{t,s})$ also defines $(\balpha_t)$, $(\bT_{s,t})_{s,t \geq 1}$ as well as collections of $r\times r$ matrices $(\bgamma_t)_{t \geq 1}$, $(\bSigma_{s,t})_{s,t \geq 0}$ which did not appear in Lemma \ref{lem:gfom-to-amp}.
The recursion, whose lines are implemented in the order in which they appear, is
\begin{equation}\label{amp-scalars-lr-mat}
  \begin{split}
      \bxi_{t,s} &= \E[\nabla_{\tbZ^s} f_t(\bgamma_1 \bLambda + \tilde \bZ^1 , \ldots , \bgamma_t \bLambda + \tilde \bZ^t ; 0 , \bU)], \quad 1 \leq s \leq t,\\
      \balpha_{t+1} &= \E[f_t(\bgamma_1 \bLambda + \tilde \bZ^1 , \ldots , \bgamma_t \bLambda + \tilde \bZ^t ; 0 , \bU)\bLambda^\mathsf{T}], \\
      \bT_{s+1,t+1} &= \E[f_s(\bgamma_1 \bLambda + \tilde \bZ^1 , \ldots , \bgamma_t \bLambda + \tilde \bZ^s ; 0 , \bU)f_t(\bgamma_1 \bLambda + \tilde \bZ^1 , \ldots , \bgamma_t \bLambda + \tilde \bZ^t ; 0 , \bU)^\mathsf{T}], \; s \leq t,\\
      \bzeta_{t,s} &= \frac1\delta \E[\nabla_{\bZ^{s+1}} g_t(\balpha_1\bTheta + Z^1,\ldots,\balpha_t \bTheta + \bZ^t;\bV)], \quad 0 \leq s \leq t-1,\\
      \bgamma_t &= \frac1\delta \E[g_t(\balpha_1\bTheta + \bZ^1,\ldots,\balpha_t \bTheta + \bZ^t;\bV)\bTheta^\mathsf{T}],\\
      \bSigma_{s,t} &= \frac1\delta \E[g_s(\balpha_1\bTheta + \bZ^1,\ldots,\balpha_t \bTheta + \bZ^s;\bV)g_t(\balpha_1\bTheta + \bZ^1,\ldots,\balpha_t \bTheta + \bZ^t;\bV)^\mathsf{T}], \quad 1 \leq s \leq t,
  \end{split}
\end{equation}
where $\bLambda \sim \mu_{\bLambda}$ $\bU \sim \mu_{\bU}$, $\bTheta \sim \mu_{\bTheta}$, $\bV \sim \mu_{\bV}$, $(\tilde \bZ^1,\ldots,\tilde \bZ^t) \sim \normal(\bzero,\bSigma_{[1:t]})$, 
and $(\bZ^1,\ldots,\bZ^t) \sim \normal(\bzero,\bT_{[1:t]})$, 
all independent.
Here $\nabla$ denotes the Jacobian with respect to subscripted (vectorial) argument, which exists almost everywhere because the functions involved are Lipschitz and the random variables have density with respect to Lebesgue measure \cite[pg.~81]{Evans2015MeasureFunctions}.
As with $\bT_{[1:t]}$, we define $\bSigma_{[1:t]}$ to be the $rt \times rt$ block matrix with block $(s,t)$ given by $\bSigma_{s,t}$.
We initialize at the second line with $\balpha_1 = \E[f_0(0,\bU)\bLambda^\mathsf{T}]$.
In addition to \eqref{eq:se}, 
we have
\begin{equation*}
  \frac1n \sum_{i=1}^n \psi(\bb^1_i,\ldots,\bb^t_i,\bu_i,\blambda_i) \stackrel{\mathrm{p}}\rightarrow \E[\psi(\bgamma_1\bLambda + \tilde \bZ^1,\ldots,\bgamma_t\bLambda + \tilde \bZ^t,\bU,\bLambda)],
\end{equation*}
where we remind the reader that $\psi:\reals^{r(t+2)}\rightarrow \reals$ is any pseudo-Lipschitz function of order 2.

We now show Eq.~\eqref{eq:se} for $(\alpha_s),(T_{s,s'})$ defined in this way.
We consider the $r= 1$ case, as $r > 1$ is similar by requires more notational overhead.
Because $\bX = \frac{1}{n} \blambda \btheta^{\mathsf{T}} + \bZ$, we have
\begin{align*}
\ba^{t+1} - \frac{1}{n}\langle \blambda, f_t(\bb^1, \ldots, \bb^t, 0, \bu)\rangle\btheta &= \bZ^{\mathsf{T}} f_t(\bb^1, \ldots, \bb^t, 0, \bu) - \sum\limits_{s=1}^t \xi_{t,s}g_s(\ba^1,\ldots, \ba^s, \bv),\\
\bb^t - \frac{1}{n}\langle \btheta, g_t(\ba^1, \ldots, \ba^t, \bv) \rangle\blambda &= \bZ g_t(\ba^1, \ldots, \ba^t, \bv) - \sum\limits_{s=0}^{t-1}\zeta_{t,s}f_s(\bb^1, \ldots, \bb^s, \by, \bu).
\end{align*}
We introduce a change of variables:
\begin{eqnarray*}
\hat{f}_t(d^1,\ldots,d^t,u,\lambda) \overset{\Delta}{=} f_t(d^1 + \gamma_1\lambda, \ldots, d^t + \gamma_t\lambda, 0, u), &\qquad& \bd^t = \bb^t - \gamma_t \blambda \in \mathbb{R}^n,\\
\hat{g}_t(c^1, \ldots, c^t, v, \theta) \overset{\Delta}{=} g_t(c^1 + \alpha_1 \theta, \ldots, c^t + \alpha_t \theta, v), &\qquad& \bc^t = \ba^t - \alpha_t \btheta \in \mathbb{R}^p.
\end{eqnarray*}
Because $f_t$, $g_t$ are Lipschitz continuous, so too are $\hat{f}_t$, $\hat{g}_t$.
We have 
\begin{align*}
\ba^{t+1} - \frac{1}{n} \langle \blambda, \hat{f}_t(\bd^1, \ldots, \bd^t, \bu, \blambda)\rangle \btheta &= \bZ^{\mathsf{T}} \hat{f}_t(\bd^1, \ldots, \bd^t, \bu, \blambda) - \sum\limits_{s=1}^t\xi_{t,s}\hat{g}_s(\bc^1, \ldots, \bc^s, \bv, \btheta),\\
\bb^{t} - \frac{1}{n} \langle \btheta, \hat{g}_t(\bc^1, \ldots, \bc^t, \bv, \btheta) \rangle \blambda &= \bZ \hat{g}_t(\bc^1, \ldots, \bc^t, \bv, \btheta) - \sum\limits_{s=0}^{t-1}\zeta_{t,s}\hat{f}_s(\bb^1, \ldots, \bb^s, \bu, \blambda).
\end{align*}
Define
\begin{align*}
\hat{\bc}^{t+1} &= \bZ^{\mathsf{T}}\hat{f}_t(\hat{\bd}^1, \ldots, \hat{\bd}^t, \bu, \blambda) - \sum\limits_{s=1}^t \xi_{t,s}\hat{g}_s(\hat{\bc}^1, \ldots, \hat{\bc}^t,\bv, \btheta),\\
\hat{\bd}^t &= \bZ \hat{g}_t(\hat{\bc}^1, \ldots, \hat{\bc}^t, \bv, \btheta) - \sum\limits_{s=0}^{t-1} \zeta_{t,s} \hat{f}_s(\hat{\bd}^1, \ldots, \hat{\bd}^t,\bu, \blambda).
\end{align*}
We can analzye this iteration via the same techniques we used to analyze AMP
in the high-dimensional regression model in the previous section \cite{javanmard2013state}.
In particular, for any pseudo-Lipschitz function $\psi: \mathbb{R}^{t+2}\rightarrow \mathbb{R}$ of order 2, we have 
\begin{equation}
\begin{aligned}\label{eq:approx-amp-lr-mat-se}
\frac{1}{p}\sum\limits_{j=1}^p \psi(\hat{c}_j^1, \ldots, \hat{c}_j^t, v_j, \theta_j) &\stackrel{\mathrm{p}}\rightarrow \E[\psi(Z^1, \ldots, Z^t, V, \Theta)],\\
\frac{1}{n}\sum\limits_{i=1}^n \psi(\hat{d}^1_i, \ldots, \hat{d}^t_i, u_i, \lambda_i) &\stackrel{\mathrm{p}}\rightarrow \E[\psi(\tilde{Z}^1, \ldots, \tilde{Z}^t, U, \Lambda)].
\end{aligned}
\end{equation}
Now, t\o establish \eqref{eq:se}, it suffices to show
\begin{equation}\label{eq:approx-amp-lr-mat-iter}
    \frac{1}{n} \| \hat{\bc}^t - \bc^t\|_2^2 \stackrel{\mathrm{p}}\rightarrow 0, \qquad \frac{1}{n} \| \hat{\bd}^t - \bd^t\|_2^2 \stackrel{\mathrm{p}}\rightarrow 0.
\end{equation}
We proceed by induction.
By the weak law of large numbers, we have that $\frac{1}{n}\langle \blambda, \hat{f}_0(\blambda, \bu)\rangle = \frac{1}{n} \langle \blambda, f_0(0, \bu) \rangle \stackrel{\mathrm{p}}\rightarrow \alpha_1$. Therefore, $\bc^1 = \bZ^{\mathsf{T}}\hat{f}_0(\blambda, \bu) + o_p(1)\btheta = \hat{\bc}^1 + o_p(1) \btheta$. Since $\frac{1}{p}\|\btheta\|_2^2 \stackrel{\mathrm{p}}\rightarrow \E[\Theta^2]$, we have that $\frac{1}{n}\| \bc^1 - \hat{\bc}^1 \|_2^2 \stackrel{\mathrm{p}}\rightarrow 0$.

Because $\hat g_1$ is Lipschitz and $\frac1p\|\btheta\|^2 = O_p(1)$, we have $|\frac{1}{n}\langle \btheta, \hat{g}_1(\bc^1, \btheta, \bv)\rangle - \frac{1}{n}\langle \btheta, \hat{g}_1(\hat{\bc}^1, \btheta, \bv)\rangle |  \stackrel{\mathrm{p}}\rightarrow 0$. 
By \eqref{eq:approx-amp-lr-mat-se}, we have that $\frac{1}{n} \langle \btheta, \hat{g}_1(\hat{\bc}^1, \btheta, \bv)\rangle \stackrel{\mathrm{p}}\rightarrow \gamma_1$. 
We have
\begin{equation*}
    \frac{1}{n}\|\hat{g}_1(\bc^1, \bv, \btheta) - \hat{g}_1(\hat{\bc}^1, \bv, \btheta)\|_2^2 \leq \frac{1}{n}L^2 \| \bc^1 - \hat{\bc}^1\|_2^2 \stackrel{\mathrm{p}}\rightarrow 0,
\end{equation*}
wehre $L$ is a Lipschitz constant for $\hat g_1$.
By \cite{bai2008limit}, the maximal singular value of $\bZ^T\bZ$ is $O_p(1)$. 
Therefore, $\frac{1}{n}\|\bZ \hat{g}_1(\bc^1, \bv, \btheta) - \bZ\hat{g}_1(\hat{\bc}^1, \bv, \btheta)\|_2^2 \stackrel{\mathrm{p}}\rightarrow 0$. 
As a result, and using that $\frac{1}{n}\|\blambda\|_2^2$ converges almost surely to a constant,
\begin{equation*}
    \frac{1}{n} \|\hat{\bd}^1 - \bd^1 \|_2^2 = \frac{1}{n} \| \bZ \hat{g}_t(\hat{\bc}^1, \bv, \btheta) - \bZ \hat{g}_t(\bc^1, \bv, \btheta) + (\frac{1}{n}\langle \btheta, \hat{g}_1(\bc^1, \btheta, \bv)\rangle - \gamma_1)\blambda \|_2^2 \stackrel{\mathrm{p}}\rightarrow 0.
\end{equation*}
Now assume that \eqref{eq:approx-amp-lr-mat-iter} holds for $1,2, \ldots,t$. For the $(t+1)$-th iteration, we have
\begin{equation*}
    |\frac{1}{n}\langle \blambda, \hat{f}_t(\bd^1, \ldots, \bd^t, \bu, \blambda)\rangle - \frac{1}{n}\langle \blambda, \hat{f}_t(\hat{\bd}^1, \ldots, \hat{\bd}^t, \bu, \blambda)\rangle | \leq \frac{L}{n}\| \blambda \|_2\sum\limits_{s=1}^t \|\bd^s - \hat{\bd}^s \|_2 \stackrel{\mathrm{p}}\rightarrow 0.
\end{equation*}
where $L$ is a Lipschitz constant for $\hat f$.
By \eqref{eq:approx-amp-lr-mat-se}, we have $\frac{1}{n}\langle \blambda, \hat{f}_t(\hat{\bd}^1, \ldots, \hat{\bd}^t, \blambda, \bu)\rangle \stackrel{\mathrm{p}}\rightarrow \alpha_{t+1}$. 
As a result, we have $\frac{1}{n}\langle \blambda, \hat{f}_t(\bd^1, \ldots, \bd^t, \bu, \blambda)\rangle \stackrel{\mathrm{p}}\rightarrow \alpha_{t+1}$. Furthermore, for any $1 \leq s \leq t$, we have 
\begin{align*}
\frac{1}{n} \| \hat{f}_s(\bd^1, \ldots, \bd^s, \bu, \blambda) - \hat{f}_s(\hat{\bd}^1, \ldots, \hat{\bd}^s, \bu, \blambda) \|_2^2 &\leq \frac{\hat{L}_t^2}{n} \sum\limits_{i=1}^s \| \bd^i - \hat{\bd}^i \|_2^2 \stackrel{\mathrm{p}}\rightarrow 0,\\
\frac{1}{n} \| \hat{g}_s(\bc^1, \ldots, \bc^s, \bv, \btheta) - \hat{g}_s(\hat{\bc}^1, \ldots, \hat{\bc}^s, \bv, \btheta)\|_2^2 &\leq \frac{\hat{L}_t^2}{n}\sum\limits_{i=1}^s \| \bc^i - \hat{\bc}^i \|_2^2 \stackrel{\mathrm{p}}\rightarrow 0.
\end{align*}
Again using that the maximal singular value of $\bZ^\sT\bZ$ is $O_p(1)$,
we have 
\begin{equation*}
	\frac{1}{n}\|\bZ^{\mathsf{T}}\hat{f}_t(\hat{\bd}^1, \ldots, \hat{\bd}^t, \bu, \blambda) - \bZ^{\mathsf{T}}\hat{f}_t(\bd^1, \ldots, \bd^t, \bu, \blambda)\|_2^2 \stackrel{\mathrm{p}}\rightarrow 0\,.
\end{equation*}
As a result, we have
\begin{eqnarray*}
    &&\frac{1}{n}\|\hat{\bc}^{t+1} - \bc^{t+1}\|_2^2\\
    &=& \frac{1}{n} \|(\frac{1}{n}\langle \blambda, \hat{f}_t(\bd^1, \ldots, \bd^t, \bu, \blambda) - \alpha_{t+1}) \btheta + \bZ^{\mathsf{T}}(\hat{f}_t(\hat{\bd}^1, \ldots, \hat{\bd}^t, \bu, \blambda) - \hat{f}_t(\bd^1, \ldots, \bd^t, \bu, \blambda)) -\\
    &&\sum\limits_{s=1}^t\xi_{t,s}(\hat{g}_s(\hat{\bc}^1, \ldots, \hat{\bc}^s, \bv, \btheta) - \hat{g}_s(\bc^1, \ldots, \bc^s, \btheta, \bv))\|_2^2 \stackrel{\mathrm{p}}\rightarrow 0.
\end{eqnarray*}
Similarly, we have 
\begin{equation*}
    |\frac{1}{n}\langle \btheta, \hat{g}_{t+1}(\hat{\bc}^1, \ldots, \hat{\bc}^{t+1}, \bv, \btheta)\rangle - \frac{1}{n}\langle \btheta, \hat{g}_{t+1}(\bc^1, \ldots, \bc^{t+1}, \bv, \btheta)\rangle | \leq \frac{L}{n}\|\btheta\|_2\sum\limits_{s=1}^{t+1}\|\hat{\bc}^{t+1} - \bc^{t+1}\|_2 \stackrel{\mathrm{p}}\rightarrow 0,
\end{equation*}
where $L$ is a Lipschitz constant for $\hat g_{t+1}$.
By \eqref{eq:approx-amp-lr-mat-se}, 
we have that $\frac{1}{n}\langle \btheta, \hat{g}_{t+1}(\hat{\bc}^1, \ldots, \hat{\bc}^{t+1}, \bv, \btheta)\rangle \stackrel{\mathrm{p}}\rightarrow \gamma_{t+1}$. As a result, we have that $\frac{1}{n}\langle \btheta, \hat{g}_{t+1}(\bc^1, \ldots, \bc^{t+1}, \bv, \btheta)\rangle \stackrel{\mathrm{p}}\rightarrow \gamma_{t+1}$. Furthermore, for any $1 \leq s \leq t$, we have 
\begin{equation*}
    \frac{1}{n} \| \hat{f}_s(\bd^1, \ldots, \bd^s, \bu, \blambda) - \hat{f}_s(\hat{\bd}^1, \ldots, \hat{\bd}^s, \bu, \blambda) \|_2^2 \leq \frac{L^2}{n} \sum\limits_{i=1}^s \| \bd^i - \hat{\bd}^i \|_2^2 \stackrel{\mathrm{p}}\rightarrow 0.
\end{equation*}
Also, for any $1 \leq s \leq t+1$, we have 
\begin{equation*}
    \frac{1}{n} \| \hat{g}_s(\bc^1, \ldots, \bc^s, \bv, \btheta) - \hat{g}_s(\hat{\bc}^1, \ldots, \hat{\bc}^s, \bv, \btheta)\|_2^2 \leq \frac{L^2}{n}\sum\limits_{i=1}^s \| \bc^i - \hat{\bc}^i \|_2^2 \stackrel{\mathrm{p}}\rightarrow 0.
\end{equation*}
Then $\frac{1}{n}\|\bZ\hat{g}_{t+1}(\hat{\bc}^1, \ldots, \hat{\bc}^{t+1}, \bv, \btheta) - \bZ \hat{g}_{t+1}(\bc^1, \ldots, \bc^{t+1}, \bv, \btheta) \|_2^2 \stackrel{\mathrm{p}}\rightarrow 0$. 
As a result, we have
\begin{align*}
&\frac{1}{n}\|\hat{\bd}^{t+1} - \bd^{t+1}\|_2^2 \\
&\qquad= \frac{1}{n}\| (\frac{1}{n}\langle \btheta, \hat{g}_{t+1}(\bc^1, \ldots, \bc^{t+1}, \bv, \btheta)\rangle - \gamma_{t+1}) \blambda \\
&\qquad+ \bZ(\hat{g}_{t+1}(\hat{\bc}^1, \ldots, \hat{\bc}^{t+1}, \bv, \btheta) - \hat{g}_{t+1}(\bc^1, \ldots, \bc^{t+1}, \bv, \btheta)) \\
& \qquad- \sum\limits_{s=0}^t\zeta_{t,s}(\hat{f}_s(\hat{\bd}^1, \ldots, \hat{\bd}^s, \bu, \blambda) - \hat{f}_s(\bd^1, \ldots, \bd^s, \bu, \blambda))\|_2^2 \stackrel{\mathrm{p}}\rightarrow 0.
\end{align*}
Thus, we have proved \eqref{eq:approx-amp-lr-mat-iter}. 
Therefore, for all pseudo-Lipschitz function $\psi$ of order 2, we have that
there exists a numerical constant $C$ such that
\begin{align*}
&\left|\frac{1}{p}\sum\limits_{j=1}^p\psi(c_j^1 + \alpha_1 \theta_j, \ldots, c_j^t + \alpha_t \theta_j, v_j, \theta_j) - \frac{1}{p}\sum\limits_{j=1}^p\psi(\hat{c}_j^1 + \alpha_1 \theta_j, \ldots, \hat{c}_j^t + \alpha_t \theta_j, v_j, \theta_j)\right| \\
&\qquad \qquad \leq L_{\psi}(1 + \sum\limits_{s=1}^t \| \ba^s\|_2 + \| \btheta\|_2 + \|\bv\|_2)\sum\limits_{s=1}^t\|\hat{\bc}^s - \bc^s\|_2 \stackrel{\mathrm{p}}\rightarrow 0.
\end{align*}
By \eqref{eq:approx-amp-lr-mat-se},
\begin{equation*}
    \dfrac{1}{p}\sum\limits_{j=1}^p\psi(\hat{c}_j^1 + \alpha_1 \theta_j, \ldots, \hat{c}_j^t + \alpha_t \theta_j, v_j, \theta_j) \stackrel{\mathrm{p}}\rightarrow \E[\psi(\balpha_1\bTheta + \bZ^1,\ldots,\balpha_t\bTheta + \bZ^t,\bV,\bTheta)].
\end{equation*}
Therefore, $\frac{1}{p}\sum\limits_{j=1}^p \psi(a_j^1, \ldots, a_j^t, v_j, \theta_j) \stackrel{\mathrm{p}}\rightarrow \E[\psi(\balpha_1\bTheta + \bZ^1,\ldots,\balpha_t\bTheta + \bZ^t,\bV,\bTheta)]$.
Similarly, we can show that $\frac1n \sum_{i=1}^n \psi(\bb^1_i,\ldots,\bb^t_i,\bu_i,\blambda_i) \stackrel{\mathrm{p}}\rightarrow \E[\psi(\bgamma_1\bLambda + \tilde \bZ^1,\ldots,\bgamma_t\bLambda + \tilde \bZ^t,\bU,\bLambda)]$. Thus we have finished the proof.

\subsection{The AMP change of variables}

To prove Lemma \ref{lem:gfom-to-amp}, all that remains is to show that for any GFOM \eqref{gfom}, at least one of the change-of-variables in Eqs.~\eqref{cov} generates an iteration \eqref{gfom-post-cov} which is an AMP iteration. That is, in addition to satisfying Eq.~\eqref{cov}, the matrices $(\bxi_{t,s})$, $(\bzeta_{t,s})$ and functions $(f_t)$, $(g_t)$ satisfy Eqs.~\eqref{amp-scalars-hd-reg} and \eqref{amp-scalars-lr-mat} in the high-dimensional regression and low-rank matrix estimation models respectively.

To construct such a choice of scalars, we may define $(\bxi_{t,s}),(\bzeta_{t,s}),(f_t),(g_t)$ in a single recursion by interlacing definition \eqref{cov} with either \eqref{amp-scalars-hd-reg} or \eqref{amp-scalars-lr-mat}.
Specifically, in the high-dimensional regression model, we place \eqref{cov-f-phi} before the first line of \eqref{amp-scalars-hd-reg} and \eqref{cov-g-varphi} before the fourth line of \eqref{amp-scalars-hd-reg}.
In the combined recursion, all quantities are defined in terms of previously defined quantities, yielding choices for $(\bxi_{t,s}),(\bzeta_{t,s}),(f_t),(g_t)$  which simultaneously satisfy \eqref{cov} and \eqref{amp-scalars-hd-reg}.
Thus, in the high-dimensional regression model every GFOM is equivalent, up to a change of variables, to a certain AMP algorithm.
The construction in the low-rank matrix estimation model is analogous: we place \eqref{cov-f-phi} before the first line of \eqref{amp-scalars-lr-mat} and \eqref{cov-g-varphi} before the fourth line of \eqref{amp-scalars-lr-mat}.

The proof of Lemma \ref{lem:gfom-to-amp} is complete.\textbf{}

\section{Proof of state evolution for message passing (Lemma \ref{lem:se-mp-tree})}
\label{app:mp-on-tree}

In this section, we prove Lemma \ref{lem:se-mp-tree}.
We restrict ourselves to the case $r = 1$ and $k = 1$ (with $k$ the dimensionality of $\bW$) 
because the proof for $r > 1$ or $k  > 1$ is completely analogous but would complicate notation.

Let $\calT_{v\rightarrow f} = (\calV_{v\rightarrow f}, \calF_{v\rightarrow f}, \calE_{v\rightarrow f})$ be the tree consisting of edges and nodes in $\calT$ which are separated from $f$ by $v$.
By convention, $\calT_{v\rightarrow f}$ will also contain the node $v$.
In particular, $f \not \in \calF_{v\rightarrow f}$ and $(f,v) \not \in \calE_{v \rightarrow f}$, but $v \in \calV_{v \rightarrow f}$, and $f' \in \calF_{v \rightarrow f}$ and $(v,f') \in \calE_{v \rightarrow f}$ for $f' \in \partial v \setminus f$.
We define $\calT_{f\rightarrow v}, \calV_{f\rightarrow v}, \calF_{f\rightarrow v}, \calE_{f\rightarrow v}$ similarly.
With some abuse of notation, we will sometimes use $\calT_{f\rightarrow v}, \calV_{f\rightarrow v}, \calF_{f\rightarrow v}, \calE_{f\rightarrow v}$ to denote either the collection of observations corresponding to nodes and edges in these sets or the $\sigma$-algebra generated by these obervations. 
No confusion should result.
Which random variables we consider to be ``observed'' will vary with the model, and will be explicitly described in each part of the proof to avoid potential ambiguity.

\subsection{Gaussian message passing}

We first introduce a message passing algorithm whose behavior is particularly easy to analyze.
We call this message passing algorithm a \emph{Gaussian message passing} algorithm.
We will see that in both the high-dimensional regression and low-rank matrix estimation models, the message passing algorithm \eqref{mp-on-tree} approximates a certain Gaussian message passing algorithm.

Gaussian message passing algorithms operate on a computation tree with associated random variables
$\{(\theta_v,v_v)\}_{v \in \calV} \stackrel{\mathrm{iid}}\sim \mu_{\Theta,V}$, 
$\{(w_f,u_f)\}_{f \in \calF} \stackrel{\mathrm{iid}}\sim \mu_{W,U}$, 
and $\{z_{fv}\}_{(f,v)\in \calE} \stackrel{\mathrm{iid}}\sim \normal(0,1/n)$, all independent, where $\mu_{\Theta,V},\mu_{W,U} \in \cuP_4(\reals^2)$.\footnote{We believe that only $\mu_{\Theta,V},\mu_{W,U} \in \cuP_2(\reals^2)$ is needed, but the analysis under this weaker assumption would be substantially more complicated, and the weaker assumptions are not necessary for our purposes.}
Gaussian message passing algorithms access all these random variables, so that all are considered to be ``observed.''
Thus, for example, $\calV_{f \rightarrow v}$ contains $\theta_{v'},v_{v'}$ for all nodes $v'$ separated from $f$ by $v$ (including, by convention, $v$).

Gaussian message passing algorithms are defined by sequences of Lipschitz functions $(\tf_t:\reals^{t+3}\rightarrow \reals)_{t \geq 0}$, $(\tg_t:\reals^{t+2}\rightarrow \reals)_{t \geq 0}$.
We initialize the indexing differently than with Gaussian message passing algorithms than with the message passing algorithms in Section \ref{sec:proof-of-main-results} in anticipation of notational simplifications that will occur later.
For every pair of neighboring nodes $v,f$, we generate sequences of messages $(\tilde a_{v \rightarrow f}^t)_{t \geq 1}$, $(\tilde q_{v\rightarrow f}^t)_{t \geq 0}$, $(\tilde b_{f\rightarrow v}^t)_{t \geq 0}$, $(\tilde r_{f\rightarrow v}^t)_{t \geq 0}$
according to the iteration
\begin{subequations}\label{gauss-mp-tree}
\begin{gather}
    \tilde a_{v\rightarrow f}^{t+1} = \sum_{f' \in \partial v \setminus f} z_{f'v} \tilde r_{f'\rightarrow v}^t, \qquad \tilde r_{f\rightarrow v}^t = \tf_t(\tilde b_{f\rightarrow v}^0,\ldots,\tilde b_{f\rightarrow v}^t;w_f,u_f),\label{gauss-mp-tree-transpose} \\
    \tilde b_{f\rightarrow v}^t = \sum_{v' \in \partial f \setminus v} z_{fv'} \tilde q_{v'\rightarrow f}^t, \qquad \tilde q_{v\rightarrow f}^t = \tg_t(\tilde a_{v\rightarrow f}^1,\ldots,\tilde a_{v\rightarrow f}^t;\theta_v,v_v),\label{gauss-mp-tree-non-transpose}
\end{gather}
\end{subequations}
with initialization $\tq_{v\rightarrow f}^0 = g_0(\theta_v,v_v)$.
For $t \geq 0$, define the node beliefs
\begin{equation}\label{mp-node-beliefs}
    \tilde a_v^{t+1} = \sum_{f \in \partial v} z_{fv} \tilde r_{f\rightarrow v}^t,\qquad \tilde b_f^t = \sum_{v \in \partial f} z_{fv} \tilde q_{v\rightarrow f}^t.
\end{equation}
To compactify notation, denote $\tilde \ba_v^t = (\tilde a_v^1,\ldots,\tilde a_v^t)^\mathsf{T}$,
and likewise for $\tilde \ba_{v\rightarrow f}^t$, $\tilde \bq_{v \rightarrow f}^t$, $\tilde \bb_f^t$, $\tilde \bb_{f\rightarrow v}^t,\tilde \br_{f\rightarrow v}^t$ (where the first two of these are $t$-dimensional, and the last three are $(t+1)$-dimensional). 
We will often write 
$\tf_t(\tbb_{f \rightarrow v}^t;w_f,u_f)$ 
in place of $\tf_t(\tb_{f \rightarrow v}^0,\ldots,b_{f\rightarrow v}^t;w_f,u_f)$, and similarly for $\tg_t$.
The reader should not confuse the bold font here with that in Section \ref{sec:proof-of-main-results}, in which, for example, $\ba_{v \rightarrow f}^t$ denotes the vectorial message at time $t$ rather than the collection of scalar messages prior to and including time $t$.

Gaussian message passing obeys a Gaussian state evolution, defined by covariance matrices
\begin{equation}\label{eq:gauss-se}
  \Sigma_{s,s'} = \E[\tg_s(\tbA^s;\Theta,V)\tg_{s'}(\tbA^{s'};\Theta,V)], \;\;\;
  T_{s+1,s'+1} = \E[\tf_{s}(\tbB^s;W,U)\tf_{s'}(\tbB^{s'};W,U)],
\end{equation}
where $s,s' \geq 0$, $\tbA^s \sim \normal(\bzero_s,\bT_{[1{:}s]})$, $\tbB^s \sim \normal(\bzero_{s+1},\bSigma_{[0{:}s]})$, and $(\Theta,V) \sim \mu_{\Theta,V}$, $(W,U) \sim \mu_{W,U}$ independent of $\tbA^s,\tbB^s$. 
The iteration is initialized by $\Sigma_{0,0} = \E[\tg_0(\Theta,V)^2]$.

\begin{lemma}\label{lem:se-gauss-mp-on-tree}
    If we choose a variable node $v$ and factor node $f$ independently of the randomness in our model, then for fixed $t$ and for $n,p \rightarrow \infty$, $n/p \rightarrow \delta$ we have
    \begin{subequations}\label{mp-on-tree-belief-convergence}
        \begin{gather}
            (\tilde \ba_v^t,\theta_v,v_v) \stackrel{\mathrm{W}}\rightarrow \mathsf{N}(\bzero_t,\bT_{[1{:}t]}) \otimes \mu_{\Theta,V} \;\; \text{and} \;\; (\tilde \ba_{v\rightarrow f}^t,\theta_v,v_v) \stackrel{\mathrm{W}}\rightarrow \mathsf{N}(\bzero_t,\bT_{[1{:}t]}) \otimes \mu_{\Theta,V},\label{mp-on-tree-a-belief-convergence}\\
            (\tilde \bb_f^t,w_f,u_f) \stackrel{\mathrm{W}}\rightarrow \mathsf{N}(\bzero_{t+1},\bSigma_{[0{:}t]}) \otimes \mu_{W,U} \;\; \text{and} \;\; (\tilde \bb_{f\rightarrow v}^t,w_f,u_f) \stackrel{\mathrm{W}}\rightarrow \mathsf{N}(\bzero_{t+1},\bSigma_{[0{:}t]}) \otimes \mu_{W,U}.\label{mp-on-tree-b-belief-convergence}
        \end{gather}      
    \end{subequations}
    Further, all the random variables in the preceding displays have bounded fourth moments and $\E[\|\tba_v^t - \tba_{v \rightarrow f}^t\|^2] \rightarrow 0$ and $\E[\|\tbb_f^t - \tbb_{f\rightarrow v}^t\|^2] \rightarrow 0$.
\end{lemma}

The analysis of message passing on the tree is facilitated by the many independence relationships between messages, which follow from the following lemma.

\begin{lemma}\label{lem:Gfv-measurability-of-messages}
    For all $(f,v) \in \calE$ and all $t$, the messages $\tilde r_{f \rightarrow v}^t, \tb_{f \rightarrow v}^t$ are $\calT_{f \rightarrow v}$-measurable, and the messages $\tilde q_{v \rightarrow f}^t, \ta_{a \rightarrow f}^t$ is $\calT_{v \rightarrow f}$-measurable.
\end{lemma}

\begin{proof}[Lemma \ref{lem:Gfv-measurability-of-messages}]
    The proof is by induction. The base case is that $\tilde q_{v\rightarrow f}^0 = g_0(\theta_v,v_v)$ is $\calT_{v \rightarrow f}$-measurable.
    Then, if $\tq_{v \rightarrow f}^s$ are $\calT_{v\rightarrow f}$-measurable and $\tb_{f \rightarrow v}^s$ are $\calT_{f \rightarrow v}$-measurable for $0 \leq s \leq t$ and all $(f,v) \in \calE$, then $\tb_{f\rightarrow v}^t,\tr_{f \rightarrow v}^t$ are $\calT_{f \rightarrow v}$-measurable by \eqref{gauss-mp-tree}.
    Similarly, if $\tr_{f \rightarrow v}^s$ are $\calT_{f\rightarrow v}$-measurable and $\ta_{v \rightarrow f}^s$ are $\calT_{v \rightarrow r}$-measurable for $0 \leq s \leq t$ and all $(f,v) \in \calE$, then $\ta_{f\rightarrow v}^{t+1},\tr_{f \rightarrow v}^{t+1}$ are $\calT_{v \rightarrow f}$-measurable by \eqref{gauss-mp-tree}.
    The induction is complete.
\end{proof}

We now prove Lemma \ref{lem:se-gauss-mp-on-tree}.

\begin{proof}[Lemma \ref{lem:se-gauss-mp-on-tree}]
    The proof is by induction.
    
    \noindent \textit{Base case: $(\theta_v, v_f) \stackrel{\mathrm{W}}\rightarrow \mu_{\Theta,V}$.}

    This is the exact distribution in finite samples by assumption.\\
    
    \noindent \textit{Inductive step 1: Eq.~\eqref{mp-on-tree-a-belief-convergence} at $t$, bounded fourth moments of $\tba_v^t,\tba_{v\rightarrow f}^t$, and $\E[\|\tba_v^t - \tba_{v\rightarrow f}^t\|^2] \rightarrow 0$ imply Eq.~\eqref{mp-on-tree-b-belief-convergence} at $t$, bounded fourth moments of $\tbb_f^t,\tbb_{f\rightarrow v}^t$, and $\E[\|\tbb_f^t - \tbb_{f\rightarrow v}^t\|^2] \rightarrow 0$.}

    The $\sigma$-algebras $(\calT_{v \rightarrow f})_{v \in \partial f}$ are independent of $(z_{fv})_{v \in \partial f}$, which are mutually independent of each other. 
    Thus, by \eqref{mp-node-beliefs}, conditional on $\sigma((\calT_{v \rightarrow f})_{v \in \partial f})$ the beliefs $\tbb_f^t$ are jointly normal with covariance $\widehat \bSigma_{[0{:}t]} := \frac1n \sum_{v \in \partial f} \tbq_{v \rightarrow f}^t (\tbq_{v \rightarrow f}^t)^\mathsf{T}$.
    That is,
    \begin{equation*}
        \tbb_f^t \bigm| \sigma((\calT_{v \rightarrow f})_{v \in \partial f}) \sim \mathsf{N}(\bzero_{t+1},\widehat \bSigma_{[0{:}t]}).
    \end{equation*}
    Because $(\tba_{v\rightarrow f}^t,\theta_v,v_v) \mapsto \tg_s(\tba_{v\rightarrow f}^s;\theta_v,v_v)\tg_{s'}(\tba_{v\rightarrow f}^{s'};\theta_v,v_v)$ is uniformly pseudo-Lipschitz of order 2 by Lemma \ref{lem:pseudo-lipschitz}, we have $\E[\widehat \Sigma_{s,s'}] = \E[\tq_{v\rightarrow f}^s\tq_{v \rightarrow f}^{s'}] = \E[\tg_s(\tba_{v\rightarrow f}^s;\theta_v,v_v)\tg_{s'}(\tba_{v\rightarrow f}^{s'};\theta_v,v_v)] \rightarrow \Sigma_{s,s'}$ by the inductive hypothesis, Lemma \ref{lem:wass-pseudo-lipschitz-convergence}, and \eqref{eq:gauss-se}.
    The terms in the sum defining $\widehat \bSigma_{[0{:}t]}$ are mutually independent by Lemma \ref{lem:Gfv-measurability-of-messages} and have bounded second moments by the inductive hypothesis and the Lipschitz continuity of the functions $(\tg_s)_{0\leq s \leq t}$.
    By the weak law of large numbers, $\widehat \bSigma_{[0{:}t]} \stackrel{L_1}\rightarrow \bSigma_{[0{:}t]}$, whence by Slutsky's theorem, $\tbb_f^t \stackrel{\mathrm{d}}\rightarrow \normal(\bzero_{t+1},\bSigma_{[0{:}t]})$.
    Further, $\E[\tbb_f^t(\tbb_f^t)^\mathsf{T}] = \E[\widehat \bSigma_{[0{:}t]}] \rightarrow \bSigma_{[0{:}t]}$.
    Convergence in distribution and in second moment implies convergence in the Wasserstein space of order 2 \cite[Theorem 6.9]{Villani2008}, so $\tbb_f^t \stackrel{\mathrm{W}}\rightarrow \normal(\bzero_{t+1},\bSigma_{[0{:}t]})$.

    To bound the fourth moments of $\tb_f^t$, we compute 
    $$
    \E[(\tb_f^t)^4] = \E[\widehat \Sigma_{t,t}^2] = \frac1{n^2} \sum_{v \in \partial f} \E[(\tq_{v \rightarrow f}^t)^4] + \frac1{n^2} \sum_{v \neq v' \in \partial f} \E[(\tq_{v \rightarrow f}^t)^2]\E[(\tq_{v' \rightarrow f}^t)^2] \rightarrow \Sigma_{t,t},
    $$
    where the first term goes to 0 because the fourth moments of $\tq_{v \rightarrow f}^t$ are bounded by the inductive hypothesis and Lipschitz continuity of $\tg_t$, and the second term goes to $ \E[(\tq_{v \rightarrow f}^t)^2]$ by the same argument in the preceding paragraph.
    The boundedness of the fourth moments of $\tb_f^s$ holds similarly (and, anyway, will have been established earlier in the induction).

    Finally, observe $\tb_f^t - \tb_{f\rightarrow v}^t = z_{fv} \tq_{v \rightarrow f}^t$ and $\E[(z_{fv} \tq_{v \rightarrow f}^t)^2] = \E[\tq_{v \rightarrow f}^t)^2] / n \rightarrow 0$, where $\E[\tq_{v \rightarrow f}^t)^2]$ is bounded by the inductive hypothesis and Lipschitz continuity of $\tg_t$.
    The convergence $\E[(\tb_f^t - \tb_{f\rightarrow v}^s)^2] \rightarrow 0$ for $s < t$ holds similarly (and, anyway, will have been established earlier in the induction).
  The Wasserstein convergence of $(\tbb_{v \rightarrow f}^t,\theta_v,v_v)$ now follows.
  The bounded fourth moments of $\tbb_{v \rightarrow f}^t$ hold similarly.
    \\

    \noindent \textit{Inductive step 2: Eq.~\eqref{mp-on-tree-belief-convergence} at $t$, bounded fourth moments of $\tbb_f^t,\tbb_{f\rightarrow v}^t$, and $\E[\|\tbb_f^t - \tbb_{f\rightarrow v}^t\|^2] \rightarrow 0$ imply Eq.~\eqref{mp-on-tree-belief-convergence} at $t+1$, bounded fourth moments of $\tba_v^t,\tba_{v\rightarrow f}^{t+1}$, and $\E[\|\tba_v^{t+1} - \tba_{v\rightarrow f}^{t+1}\|^2] \rightarrow 0$.} 

    This follows by exactly the same argument as in inductive step 1.
    
    The induction is complete, and Lemma \ref{lem:se-gauss-mp-on-tree} follows.
\end{proof}

\subsection{Message passing in the high-dimensional regression model}

We prove Lemma \ref{lem:se-mp-tree} for the high-dimensional regression model by showing that the iteration \eqref{mp-on-tree} is well approximated by a Gaussian message passing algorithm after a change of variables.
The functions $\tf_t,\tg_t$ in the Gaussian message passing algorithm are defined in terms of the functions $f_t,g_t$ of the original message passing algorithm \eqref{mp-on-tree} and the function $h$ used to define the high-dimensional regression model.
\begin{gather*}
    \tf_t(\tb^0, \cdots, \tb^t, w, u ) := f_t(\tb^1, \cdots, \tb^t; h(\tb^0, w) , u), \;\; t \geq 0,\\
    \tg_0(\theta,v) = \theta,\;\;\tg_t(\ta^1, \cdots, \ta^t; \theta , v ) := g_t(\alpha_1 \theta + \ta^1, \cdots, \alpha_1 \theta +  \ta^t; v), \;\; t \geq 1.
\end{gather*}
Define $(\tilde a_{v \rightarrow f}^t)_{t \geq 1}$, $(\tilde a_v^t)_{t \geq 1}$, $(\tilde q_{v\rightarrow f}^t)_{t \geq 0}$, $(\tilde b_{f\rightarrow v}^t)_{t \geq 0}$, $(\tilde b_f^t)_{t \geq 0}$, $(\tilde r_{f\rightarrow v}^t)_{t \geq 0}$ via the Gaussian message passing algorithm \eqref{gauss-mp-tree} with initial data $\theta_v,v_v,w_f,u_f$ and with $z_{fv} = x_{fv}$.
Because $f_t$, $g_t$, and $h$ are Lipschitz, so too are $\tf_t$ and $\tg_t$. 
Under the function definitions $\tf_t,\tg_t$ given above, the definitions of $\Sigma_{s,s}$ and $T_{s,s'}$ in \eqref{eq:gauss-se} and \eqref{amp-scalars-hd-reg} are equivalent.
Thus, Lemma \ref{lem:se-gauss-mp-on-tree} holds for the iterates of this Gaussian message passing algorithm with the $\bT_{[1{:}t]}$, $\bSigma_{[0{:}t]}$ defined by \eqref{amp-scalars-hd-reg}.

We claim that for fixed $s \geq 1$, as $n \rightarrow \infty$ we have 
\begin{subequations}\label{hd-reg-mp-properties}
\begin{equation}\label{hd-reg-mp-properties-1}
    \E[(\alpha_s \theta_v + \mah{s} - a_{v\rightarrow f}^s)^2] \rightarrow 0 \;\; \text{and} \;\; \E[(\mbh{s} - b_{f \rightarrow v}^s)^2] \rightarrow 0,
\end{equation}
and
\begin{equation}\label{hd-reg-mp-properties-2}
    \E[(a_{v\rightarrow f}^s)^4] \;\; \text{and} \;\; \E[(b_{f \rightarrow v}^s)^4] \;\; \mbox{are uniformly bounded with respect to $n$},
\end{equation}
\end{subequations}
where $(\alpha_s)$ are defined by \eqref{amp-scalars-hd-reg}.
These are the same coefficients appearing in the AMP state evolution (Lemma \ref{lem:gfom-to-amp}), as claimed.
We show \eqref{hd-reg-mp-properties} by induction. 
There is no base case because the inductive steps work for $t = 0$ as written.
\\

\noindent \textit{Inductive step 1: If \eqref{hd-reg-mp-properties} holds for $1\leq s \leq t$, then \eqref{hd-reg-mp-properties-1} holds for $s = t+1$.}

We expand
\begin{align*}
    \alpha_{t+1} \theta_v + \ta_{v\rightarrow f}^{t+1} - a_{v \rightarrow f}^{t+1} &= \alpha_{t+1} \theta_v + \sum_{f' \in \partial v \setminus f} z_{f'v} ( \tf_t(\tbb_{f'\rightarrow v}^t;w_{f'},u_{f'}) - f_t(\bb_{f' \rightarrow v}^t;y_{f'},u_{f'}) )\\
    &= \alpha_{t+1} \theta_v + \sum_{f' \in \partial v \setminus f} z_{f'v} ( \tf_t(\tbb_{f'\rightarrow v}^t;w_{f'},u_{f'}) - \tf_t(\tb_{f'\rightarrow v}^0,\bb_{f'\rightarrow v}^t;w_{f'},u_{f'}) )\\
    &\quad\quad + \sum_{f' \in \partial v \setminus f} z_{f'v} ( \tf_t(\tb_{f'\rightarrow v}^0,\bb_{f'\rightarrow v}^t;w_{f'},u_{f'}) - \tf_t(\tb_{f'}^0,\bb_{f' \rightarrow v}^t;w_{f'},u_{f'}) )\\
    &=: \alpha_{t+1}\theta_v + \mathsf{I} + \mathsf{II}.
\end{align*}
(Note that $\tbb_{f'\rightarrow v}^t$ is $(t+1)$-dimensional and $\bb_{f'\rightarrow v}^t$ is $t$-dimensional).
First we analyze $\mathsf{I}$.
We have
\begin{align*}
    &|\tf_t(\tbb_{f'\rightarrow v}^t;w_{f'},u_{f'}) - \tf_t(\tb_{f'\rightarrow v}^0,\bb_{f'\rightarrow v}^t;w_{f'},u_{f'})| \leq L \sum_{s = 1}^t |\tb_{f' \rightarrow v}^s - b_{f'\rightarrow v}^s|,
\end{align*}
where $L$ is a Lipschitz constant of $\tf_t$.
The terms in the sum defining $\mathsf{I}$ are mutually independent, and $\tb_{f' \rightarrow v}^s,b_{f'\rightarrow v}^s$ are independent of $z_{f'v}$.
Thus,
\begin{align*}
    \E[\mathsf{I}^2] &= \frac{n-1}n \E[(\tf_t(\tbb_{f'\rightarrow v}^t;w_{f'},u_{f'}) - \tf_t(\tb_{f'\rightarrow v}^0,\bb_{f'\rightarrow v}^t;w_{f'},u_{f'}))^2]\\
    &\leq \frac{L^2(n-1)t}{n} \sum_{s=1}^t \E[(\tb_{f'\rightarrow v}^s - b_{f' \rightarrow v}^s)^2] \rightarrow 0,
\end{align*}
by the inductive hypothesis.

Next we analyze $\mathsf{II}$.
Note that all arguments to the functions in the sum defining $\mathsf{II}$ are independent of $z_{f'v}$ and $\theta_v$ except for $\tb_{f'}^0 = z_{f'v} \theta_v + \sum_{v' \in \partial f' \setminus v} z_{f'v'} \theta_{v'}$.
Because $\tf_t$ is Lipschitz, we may apply Stein's lemma (ie., Gaussian integration by parts) \cite{Stein1981EstimationDistribution} to get
\begin{align*}
    & \E[\alpha_{t+1} \theta_v + \mathsf{II} \bigm| \theta_v, \sigma((\calT_{v'' \rightarrow f'})_{v'' \in \partial f' \backslash v})] \\
    &= 
    \alpha_{t+1} \theta_v  + (n-1)\E\big[z_{f'v}( \tf_t(\tb_{f'\rightarrow v}^0,\bb_{f'\rightarrow v}^t;w_{f'},u_{f'}) - \tf_t(\tb_{f'}^0,\bb_{f' \rightarrow v}^t;w_{f'},u_{f'})) \bigm| \theta_v \big] \\
    & = \theta_v \left(\alpha_{t+1} - \frac{n-1}{n}\E[ \partial_{\tb^0} \tf_t(\tb_{f'}^0,\bb_{f'\rightarrow v}^t;w_{f'},u_{f'}) \bigm| \theta_v]\right),
\end{align*}
where $\partial_{\tb^0} \tf_t$ is the weak-derivative of $\tf_t$ with respect to its first argument, which is defined almost everywhere with respect to Lebesgue measure because $\tf_t$ is Lipschitz \cite[pg.~81]{Evans2015MeasureFunctions}.

We claim the right-hand side of the preceding display converges in $L_2$ to 0, as we now show.
The random variable $\E[ \partial_{\tb^0} \tf_t(\tb_{f'}^0,\bb_{f'\rightarrow v}^t;w_{f'},u_{f'}) | \theta_v, (\calT_{v'' \rightarrow f'})_{v'' \in \partial f' \backslash  v}]$
is almost-surely bounded because $\tf_t$ is Lipschitz.
It converges in probability to $\alpha_{t+1}$. 
The random vector $(\tilde b_{f'}^0,\bb_{f'\rightarrow v}^t)$ has a Gaussian distribution
conditional on $\sigma((\calT_{v'' \rightarrow f'})_{v'' \in \partial f' \backslash  v})$ and $\theta_v$; 
in particular,
\begin{align*}
    (\tb^0_{f'\rightarrow v} + z_{f'v} \theta_v, \bb_{f'\rightarrow v}^t) | \theta_v, \sigma((\calT_{v'' \rightarrow f'})_{v'' \in \partial f' \backslash  v}) \stackrel{\mathrm{d}}= \normal(\mathbf{0}, \widehat \bSigma),
\end{align*}
where we define $\widehat \bSigma \in \reals^{(t+1) \times (t+1)}$ by
\begin{align*}
    \widehat \Sigma_{0,0} = \dfrac{1}{n} \sum\limits_{v' \in \partial f'} \theta_{v'}^2 \;\; \text{and} \;\; \widehat \Sigma_{s,s'} = \dfrac{1}{n} \sum\limits_{v' \in \partial f' \setminus v} q_{v' \rightarrow f'}^s q_{v'\rightarrow f'}^{s'}  \;\; \text{for} \; s \geq 1 \mbox{ or } s' \geq 1,
\end{align*}
where for the purposes of the preceding display we set $q_{v'\rightarrow f'}^0 = \theta_{v'}$.
By the Lipschitz continuity of the functions $(g_s)$, Lemmas \ref{lem:pseudo-lipschitz} and \ref{lem:wass-pseudo-lipschitz-convergence}, and the inductive hypothesis,
we have $\E[\widehat \bSigma] \rightarrow \bSigma_{[0{:}t]}$.
The terms in the sums in the previous display have bounded second moments by the inductive hypthesis \eqref{hd-reg-mp-properties-2} and the Lipschitz continuity of the functions $(g_s)$.
By the weak law of large numbers, we conclude $\widehat \bSigma \stackrel{\mathrm{p}}\rightarrow \bSigma_{[0:t+1]}$.

Observe that $\E[ \partial_{\tb^0} \tf_t(\tb_{f'}^0,\bb_{f'\rightarrow v}^t;w_{f'},u_{f'}) | \theta_v, (\calT_{v'' \rightarrow f'})_{v'' \in \partial f' \backslash  v}] = \E[\partial_{\tb^0} \tf_t(\widehat \bSigma^{1/2}\bZ;W,U)]$, where on the right-hand side the expectation is with respect to $(W,U) \sim \mu_{W,U}$ and $\bZ \sim \normal(\bzero_{t+1},\bI_{t+1})$ independent.
Because $\partial_{\tb^0} \tf_t$ is almost surely bounded, by the dominated convergence theorem, the right-hand side is continuous in $\widehat \bSigma$.
By the continuous mapping theorem and \eqref{amp-scalars-hd-reg}, we conclude $\E[ \partial_{\tb^0} \tf_t(\tb_{f'}^0,\bb_{f'\rightarrow v}^t;w_{f'},u_{f'}) | \theta_v, (\calT_{v'' \rightarrow f'})_{v'' \in \partial f' \backslash  v}] \stackrel{\mathrm{p}}\rightarrow \alpha_{t+1}$.
Then, by dominated convergence, $\E[\alpha_{t+1} \theta_v + \mathsf{II} \bigm| \theta_v ] \stackrel{L_2}\rightarrow 0$.
Moreover, because the terms in the sum defining $\mathsf{II}$ are mutually independent given $\theta_v$
\begin{align*}
    \Var( \alpha_{t+1}\theta_v + \mathsf{II} \mid \theta_v ) &\leq (n-1)\E\left[z_{f'v}^2 ( \tf_t(\tb_{f'\rightarrow v}^0,\bb_{f'\rightarrow v}^t;w_{f'},u_{f'}) - \tf_t(\tb_{f'}^0,\bb_{f' \rightarrow v}^t;w_{f'},u_{f'}))^2 \mid \theta_v \right] \\
    &\leq L^2(n-1) \E[z_{f'v}^4 \theta_v^2 \mid \theta_v ] \leq 3\theta_v^2 / n,
\end{align*}
where $L$ is the Lipschitz constant of $\tf_t$.
We conclude that $\E[\Var(\alpha_{t+1}\theta_v + \mathsf{II} \mid \theta_v)] \rightarrow 0$.
Combined with $\E[\alpha_{t+1} \theta_v + \mathsf{II} \bigm| \theta_v ] \stackrel{L_2}\rightarrow 0$, we get $\Var(\alpha_{t+1} \theta_v + \mathsf{II}) = \Var(\E[\alpha_{t+1} \theta_v + \mathsf{II}|\theta_v]) + \E[\Var(\alpha_{t+1} \theta_v + \mathsf{II}|\theta_v)] \rightarrow 0$, so that $\alpha_{t+1} \theta_v + \mathsf{II} \stackrel{L_2}\rightarrow 0$.
Combining $\mathsf{I}\stackrel{L_2}\rightarrow 0$ and $\alpha_{t+1}\theta_v + \mathsf{II}\stackrel{L_2}\rightarrow 0$ gives $\E[(\alpha_{t+1}\theta_v + \ta_{v\rightarrow f}^{t+1} - a_{v \rightarrow f}^{t+1})^2]\rightarrow 0$, as desired.

We now expand
\begin{equation*}
    \tb_{f \rightarrow v}^{t+1} - b_{f\rightarrow v}^{t+1} = \sum_{v'\in\partial f \setminus v} z_{fv'}(g_t(\balpha_{t+1} \theta_{v'} + \tba_{v' \rightarrow f}^{t+1}; v_{v'}) - g_t(\ba_{v'\rightarrow f}^{t+1}; v_{v'})).
\end{equation*}
The terms in this sum are mutually independent, and $\tba_{v'\rightarrow f}^{t+1},\ba_{v'\rightarrow f}^{t+1},\theta_{v'}$ are independent of $z_{f'v}$.
Thus,
\begin{align*}
    \E[(\tb_{f \rightarrow v}^{t+1} - b_{f\rightarrow v}^{t+1})^2] &= \frac{p-1}n\E[(g_t(\balpha_{t+1} \theta_{v'} + \tba_{v' \rightarrow f}^{t+1}; v_{v'}) - g_t(\ba_{v'\rightarrow f}^{t+1}; v_{v'}))^2]\\
    &\leq \frac{L^2(p-1)(t+1)}{n} \sum_{s=1}^{t+1}\E[(\alpha_s \theta_{v'} + \ta_{v'\rightarrow f}^s - a_{v'\rightarrow f}^s)^2] \rightarrow 0.
\end{align*}
This completes the proof of \eqref{hd-reg-mp-properties-1} at $s = t+1$.
\\

\noindent \textit{Inductive step 2: If \eqref{hd-reg-mp-properties} holds for $1\leq s \leq t$, then \eqref{hd-reg-mp-properties-2} holds for $s = t+1$.}

By Lipschitz continuity,
\begin{align*}
  \left|a_{v \rightarrow f}^{t+1} - \sum_{f' \in \partial v \setminus f} z_{f'v} \tf_t(\tb_{f'\rightarrow v}^0,\bb_{f'\rightarrow v}^t,u_{f'},w_{f'})\right| \leq L|\theta_v|\sum_{f' \in \partial v \setminus f} |z_{f'v}|,
\end{align*}
where $L$ is a Lipschitz constant for $\tf_t$.
The right-hand side has bounded fourth moment, so we must only show that the sum in the previous display has 
bounded fourth moment.
The quantity $\tf_t(\tb_{f'\rightarrow v}^0,\bb_{f'\rightarrow v}^t,u_{f'},w_{f'})$ has bounded fourth moment by the inductive hypothesis and Lipschitz continuity of $\tf_t$.
Because $z_{f'v}$ is independent of the argument to $\tf_t$ and has fourth moment $3/n^2$, the product $z_{f'v}\tf_t(\tb_{f'\rightarrow v}^0,\bb_{f'\rightarrow v}^t,u_{f'},w_{f'})$ has mean 0 and fourth moment $O(1/n^2)$.
Because these products are mean zero and independent across $f'$, their sum has bounded fourth moment.
We conclude $a_{v \rightarrow f}^{t+1} $ has bounded fourth moment as well.

Recall $b_{f \rightarrow v}^{t+1} = \sum_{v' \in \partial f \setminus v} z_{fv'} g_t(\ba_{v'\rightarrow f}^{t+1};v_v')$.
The terms in the sum are independent, and $z_{fv'}$ is independent of $\ba_{v'\rightarrow f}^{t+1};v_v'$.
Using the Lipschitz continuity of $g_t$ and the inductive hypothesis, we conclude $b_{f \rightarrow v}^{t+1}$ has bounded fourth moment by the same argument as in the preceding paragraph.

We conclude \eqref{hd-reg-mp-properties-2} at $s = t+1$.

The induction is complete, and \eqref{hd-reg-mp-properties-1} holds for all $s\geq1$.
Lemma \ref{lem:se-mp-tree} follows by combining Lemma \ref{lem:se-gauss-mp-on-tree} and Eq.~\eqref{hd-reg-mp-properties-1}.

\subsection{Message passing in the low-rank matrix estimation model}

Like in the preceding section, we prove Lemma \ref{lem:se-mp-tree} for the low-rank matrix estimation model by showing that the iteration \eqref{mp-on-tree} is well approximated by a Gaussian message passing algorithm after a change of variables.
The functions in the Gaussian message passing algorithm are defined in terms of the functions $f_t,g_t$ of the original message passing algorithm \eqref{mp-on-tree}.
\begin{gather*}
    \tf_t(\tb^0, \cdots, \tb^t, w, u ) := f_t(\tb^1 + \gamma_1 w, \cdots, \tb^t + \gamma_t w; 0 , u), \\
    \tg_t(\ta^1, \cdots, \ta^t; \theta , v ) := g_t(\ta^1 + \alpha_1 \theta, \cdots, \ta^t + \alpha_t \theta; v).
\end{gather*}
Note that here $\tf_t$ does not depend on $\tb^0$ is never used, and we may define $\tg_0$ arbitrarily without affecting later iterates.\footnote{The iterate $\tb^0$ only played a role in approximating the high-dimensional regression message passing algorithm by a Gaussian message passing algorithm.} 
Define $(\tilde a_{v \rightarrow f}^t)_{t \geq 1}$, $(\tilde a_v^t)_{t \geq 1}$, $(\tilde q_{v\rightarrow f}^t)_{t \geq 0}$, $(\tilde b_{f\rightarrow v}^t)_{t \geq 0}$, $(\tilde b_f^t)_{t \geq 0}$, $(\tilde r_{f\rightarrow v}^t)_{t \geq 0}$ via the Gaussian message passing algorithm \eqref{gauss-mp-tree} with initial data $\theta_v,v_v,u_f,z_{fv}$ and $w_f = \lambda_f$.
Because $f_t$, $g_t$, and $h$ are Lipschitz, so too are $\tf_t$ and $\tg_t$.
Under the function definitions $\tf_t,\tg_t$ given above and the change of variables $w_f = \lambda_f$, the definitions of $\Sigma_{s,s}$ and $T_{s,s'}$ in \eqref{eq:gauss-se} and \eqref{amp-scalars-lr-mat} are equivalent.
Thus, Lemma \ref{lem:se-gauss-mp-on-tree} holds for the iterates of this Gaussian message passing algorithm with the $\bT_{[1{:}t]}$, $\bSigma_{[0{:}t]}$ defined by \eqref{amp-scalars-lr-mat}.

We claim that for fixed $s \geq 1$, as $n \rightarrow \infty$ we have
\begin{subequations}\label{lr-mat-mp-properties}
\begin{equation}\label{lr-mat-mp-properties-1 }
  \E[(\alpha_s \theta_v + \ta_{v\rightarrow f}^s - a_{v\rightarrow f}^s)^2] \rightarrow 0\;\; \text{and} \;\; \E[(\gamma_s \lambda_f + \tb_{f\rightarrow v}^s - b_{f\rightarrow v}^s)^2] \rightarrow 0,
\end{equation}
and
\begin{equation}\label{lr-mat-mp-properties-2}
  \E[\theta_v^2 (a_{v\rightarrow f}^s)^2] \;\; \text{and} \;\; \E[\lambda_f^2 (b_{f\rightarrow v}^s)^2] \;\; \text{are bounded for fixed $s$.}
\end{equation}
\end{subequations}
We show this by induction. There is no base case because the inductive step works for $t = 0$ as written.
\\

\noindent \textit{Inductive step: If \eqref{lr-mat-mp-properties} holds for $1\leq s \leq t$, then \eqref{lr-mat-mp-properties} holds for $s = t+1$.} 

We expand
\begin{align*}
  \alpha_{t+1} \theta_v + \ta_{v \rightarrow f}^{t+1} - a_{v \rightarrow f}^{t+1} &= \alpha_{t+1} \theta_v + \sum_{f' \in \partial v \setminus f} z_{f'v} ( f_t( \tbb_{f'\rightarrow v}^t + \bgamma_t \lambda_{f'} ; 0, u_{f'}) - f_t( \bb_{f'\rightarrow v}^t  ; 0, u_{f'}) ) \\
  &\quad\quad - \frac1n \theta_v \sum_{f' \in \partial v \setminus f} \lambda_{f'} f_t( \bb_{f'\rightarrow v}^t  ; 0, u_{f'})\\
  &=: \alpha_{t+1} \theta_v + \mathsf{I} + \mathsf{II},
\end{align*}
where $\tbb_{f'\rightarrow v}^t = (\tb_{f'\rightarrow v}^1,\ldots,\tb_{f'\rightarrow v}^t)$ and $\bgamma_t = (\gamma_1,\ldots,\gamma_t)$ (note that $\tb_{f'\rightarrow v}^0$ is excluded, which differs from the notation used in the proof of Lemma \ref{lem:se-mp-tree}).

First we analyze $\mathsf{I}$.
The terms in the sum defining $\mathsf{I}$ are mutually independent, and $\tb_{f'\rightarrow v}^s$, $b_{f'\rightarrow v}^s$, $\lambda_{f'}$, $u_{f'}$ are independent of $z_{f'v}$.
Thus,
\begin{align*}
  \E[\mathsf{I}^2] &= \frac{n-1}n \E[( f_t( \tbb_{f'\rightarrow v}^t + \bgamma_t \lambda_{f'} ; 0, u_{f'}) - f_t( \bb_{f'\rightarrow v}^t  ; 0, u_{f'})^2]\\
  &\leq \frac{L^2(n-1)t}{n} \sum_{s=1}^t \E[(\tb_{f'\rightarrow v}^s + \gamma_s \lambda_{f'} - b_{f'\rightarrow v}^s)^2] \rightarrow 0,
\end{align*}
by the inductive hypothesis, where $L$ is a Lipschitz constant of $f_t$.
Moreover, because $\theta_v$ is independent of $\mathsf{I}$ and has bounded fourth moment, $\E[\theta_v^2 \mathsf{I}^2] \rightarrow 0$ as well.

Next we analyze $\mathsf{II}$.
By the inductive hypothesis and Lemma \ref{lem:se-gauss-mp-on-tree}, 
\begin{equation*}
  (\bb_{f'\rightarrow v}^t,\lambda_{f'},u_{f'}) \stackrel{\mathrm{W}}\rightarrow (\bgamma_t \Lambda + \tB^t,\Lambda,U),
\end{equation*}
where $(\Lambda,U) \sim \mu_{\Lambda,U}$ and $\tB^t \sim \normal(\bzero_t,\bSigma_{[1{:}t]})$ independent.
Because $(\bb^t,\lambda,u)\mapsto \lambda f_t(\bb^t;0,u)$ is uniformly pseudo-Lipschitz of order 2 by Lemma \ref{lem:pseudo-lipschitz}, we have $\E[\lambda_{f'}f_t(\bb_{f'\rightarrow v}^t;0,u_{f'})] \rightarrow \alpha_{t+1}$ by Lemma \ref{lem:wass-pseudo-lipschitz-convergence} and the state evolution recursion \eqref{amp-scalars-lr-mat}.
Moreover, because $f_t$ is Lipschitz, for some constant $C$
\begin{align*}
  \E[\lambda_{f'}^2f_t(\bb_{f'\rightarrow v}^t;0,u_{f'})^2] &\leq C\E\left[\lambda_{f'}^2\left(1 + \sum_{s=1}^t (b_{f'\rightarrow v}^s)^2 + u_{f'}^2\right)\right] \\
  &= C\left(\E[\lambda_{f'}^2] + \sum_{s=1}^t \E[\lambda_{f'}^2 (b_{f'\rightarrow v}^s)^2] + \E[\lambda_{f'}^2u_{f'}^2]\right),
\end{align*}
which bounded by the inductive hypothesis and the fourth moment assumption on $\mu_{\Lambda,U}$.
Because the terms in the sum defining $\mathsf{II}$ are mutually independent, by the weak law of large numbers the preceding observations imply 
\begin{equation*}
  \frac1n\sum_{f' \in \partial v \setminus f} \lambda_{f'} f_t( \bb_{f'\rightarrow v}^t  ; 0, u_{f'}) \stackrel{L_2}\rightarrow \alpha_{t+1}.
\end{equation*}
Because $\theta_v$ is independent of this sum and has bounded second moment, we conclude that 
\begin{equation*}
  \alpha_{t+1} \theta_v + \mathsf{II} = \theta_v \left(\alpha_{t+1} - \frac1n\sum_{f' \in \partial v \setminus f} \lambda_{f'} f_t( \bb_{f'\rightarrow v}^t  ; 0, u_{f'}) \right) \stackrel{L_2}\rightarrow 0.
\end{equation*}
Moreover, because $\theta_v$ is independent of the term in parentheses and has bounded fourth moment, $\E[\theta_v^2(\alpha_{t+1}\theta_v + \mathsf{II})^2]\rightarrow 0$.

Combining the preceding results, we have that $\E[(\alpha_{t+1} \theta_v + \ta_{v\rightarrow f}^{t+1} - a_{v\rightarrow f}^{t+1})^2] \rightarrow 0$ and $\E[\theta_v^2(\alpha_{t+1} \theta_v + \ta_{v\rightarrow f}^{t+1} - a_{v\rightarrow f}^{t+1})^2]$ is bounded.
Because $\theta_v$ is independent of $\ta_{v \rightarrow f}^{t+1}$, the term $\E[\theta_v^2 (\ta_{v\rightarrow f}^{t+1})^2]$ is bounded, so also $\E[\theta_v^2 (a_{v\rightarrow f}^{t+1})^2]$ is bounded, as desired.

The argument establishing that $\E[(\gamma_{t+1}\lambda_f + \tb_{f\rightarrow v}^{t+1} - b_{f\rightarrow v}^{t+1})^2] \rightarrow 0$ and that $\E[\lambda_f^2(b_{f\rightarrow v}^{t+1})^2]$ is bounded is equivalent.
The induction is complete, and \eqref{lr-mat-mp-properties} holds for all $s$.

Lemma \ref{lem:se-mp-tree} follows by combining Lemma \ref{lem:se-gauss-mp-on-tree} and Eq.~\eqref{lr-mat-mp-properties}.

\section{Proof of information-theoretic lower bounds on the computation tree (Lemma \ref{lem:local-info-theory-lb})}
\label{app:info-theory-lb}

In this section, we prove Lemma \ref{lem:local-info-theory-lb} in both the high-dimensional regression and low-rank matrix estimation models.
We restrict ourselves to the case $r = 1$ and $k = 1$ (with $k$ the dimensionality of $\bW$) 
because the proof for $r > 1$ or $k  > 1$ is completely analogous but would complicate notation.

For any pair of nodes $u,u'$ in the tree $\calT$, let $d(u,u')$ denote the length (number of edges) of the shortest path between nodes $u$ and $u'$ in the tree.
Let $\calT_{u,k} = (\calV_{u,k}, \calF_{u,k}, \calE_{u,k})$ be the radius-$k$ neighborhood of node $u$; that is,
\begin{gather*}
    \calV_{u,k} = \{ v \in \calV \mid d(u,v) \leq k \},\\
    \calF_{u,k} = \{ f \in \calF \mid d(u,f) \leq k \},\\
    \calE_{u,k} = \{ (f,v) \in \calE \mid \max\{d(u,f),d(u,v)\} \leq k\}.
\end{gather*}
With some abuse of notation, we will often use $\calT_{u,k}, \calV_{u,k}, \calF_{u,k}, \calE_{u,k}$ to denote either the collection of observations corresponding to nodes and edges in these sets or the $\sigma$-algebra generated by these obervations. 
No confusion should result.
Note, our convention is that when used to denote a $\sigma$-algebra or collection of random variables, only observed random variables are in include.
Thus, in the high-dimensional regression model, $\calT_{u,k}$ is the $\sigma$-algebra generated by the local observations $x_{fv}$, $y_{f}$, $v_v$, and $u_f$; 
in the low-rank matrix estimation, it is the $\sigma$-algebra genreated by the local observations $x_{fv}$, $v_v$, and $u_f$.
We also denote by $\calT_{v \rightarrow f}^{t,k}$ the collection of observations associated to edges or nodes of $\calT$ which are separated from $f$ by $v$ by at least $k$ intervening edges and at most $t$ intervening edges.
For example, $\calT_{v \rightarrow f}^{1,1}$ contains only $(y_{f'})_{f' \in \partial v \setminus f}$, and $\calT_{v \rightarrow f}^{2,1}$ contains additional the observations $v_{v'}$ and $x_{f'v'}$ for $v' \in \partial f' \setminus v$ for some some $f' \in \partial v \setminus f$.
The collections (or $\sigma$-algebras) $\calV_{v \rightarrow f}^{t,k}$, $\calF_{v \rightarrow f}^{t,k}$, $\calE_{v \rightarrow f}^{t,k}$ are defined similarly, as are the versions of these where the roles of $v$ and $f$ are reversed.

\subsection{Information-theoretic lower bound in the high-dimensional regression model}
\label{app:info-lb-hd-reg}

In this section, we prove Lemma \ref{lem:local-info-theory-lb} in the high-dimensional regression model.

Note that conditions on the conditional density in assumption \textsf{R4} are equivalent positivity, boundedness, and the existence finite, non-negative constants $q_k'$ such that $\frac{|\partial_x^k p(y|x)|}{p(y|x)} \leq q_k'$ for $1 \leq k \leq 5$.
We will often use this form of the assumption without further comment.
This implies that for any random variable $A$ 
\begin{equation}\label{eq:score-bound}
  \frac{|\partial_x^k \E[p(y|x + A)]|}{\E[p(y|x+A)]} \leq \int \frac{|\partial_x^k p(y|x +a)|}{p(y|x+a)} \frac{p(y|x+a)}{\E[p(y|x+A)]} \mu_A(\de a)  \leq q_k',
\end{equation}
because $p(y|x+a)/\E[p(y|x+A)]$ is a probability density with respect to $\mu_A$, the distribution of $A$.

Denote the regular conditional probability of $\Theta$ conditional on $V$ for the measure $\mu_{\Theta,V}$ by $\mu_{\Theta|V} : \reals \times \calB \rightarrow [0,1]$, where $\calB$ denotes the Borel $\sigma$-algebra on $\reals$.
The posterior of $\theta_v$ given $\calT_{v,2t}$ has density with respect to $\mu_{\Theta|V}(v_v,\cdot)$ given by
\begin{equation*}
    p_v(\vartheta | \calT_{v,2t}) \propto \int \prod_{f \in \calF_{v,2t}} p(y_f \mid \sum_{v' \in \partial f} \vartheta_{v'} X_{v'f},u_f) \prod_{v' \in \calV_{v,2t}\setminus v} \mu_{\Theta|V}(v_{v'},\mathrm{d}\vartheta_{v'}).
\end{equation*}
Asymptotically, the posterior density with respect to $\mu_{\Theta|V}(v_v,\cdot)$ behaves like that produced by a Gaussian observation of $\theta_v$ with variance $\tau_t^2$, where $\tau_t$ is defined by \eqref{bamp-se-hd-reg}.
\begin{lemma}\label{lem:hd-reg-posterior}
  In the high-dimensional regression model,
  there exist $\calT_{v,2t}$-measurable random variables $\tau_{v,t},\chi_{v,t}$ such that 
  \begin{equation*}
   p_v( \vartheta | \calT_{v,2t} ) \propto \exp\left(-\frac1{2\tau_{v,t}^2}(\chi_{v,t} - \vartheta)^2 + o_p(1)\right),
  \end{equation*}
  where $o_p(1)$ has no $\vartheta$ dependence.
  Moreover, $(\chi_{v,t},\tau_{v,t},\theta_v,v_v) \stackrel{\mathrm{d}}\rightarrow (\Theta + \tau_t G, \tau_t,\Theta,V)$ where $(\Theta,V) \sim \mu_{\Theta,V}$, $G \sim \normal(0,1)$ independent of $\Theta,V$, and $\tau_t$ is given by \eqref{bamp-se-hd-reg}.
\end{lemma}

\begin{proof}[Lemma \ref{lem:hd-reg-posterior}]
    We compute the posterior density $p_v(\vartheta|\calT_{v,2t})$ via an iteration called belief propagation.
    For each edge $(v,f) \in \calE$, belief propagation generates a pair of sequences of real-valued functions
    $(m_{v \rightarrow f}^t(\vartheta))_{t \geq 0}, (m_{f \rightarrow v}^t(\vartheta))_{t \geq 0}$.
    The iteration is 
    \begin{gather*}
        m_{v \rightarrow f}^0(\vartheta) = 1,\\
        m_{f \rightarrow v}^s(\vartheta) \propto \int p(y_f | X_{fv} \vartheta + \sum_{v' \in \partial f \setminus v} X_{fv'} \vartheta_{v'},u_f) \prod_{v' \in \partial f \setminus v} m_{v' \rightarrow f}^s(\vartheta_{v'}) \prod_{v' \in \partial f\setminus v} \mu_{\Theta|V}(v_{v'},\mathrm{d}\vartheta_{v'}),\\
        m_{v \rightarrow f}^{s+1}(\vartheta) \propto \prod_{f' \in \partial v \setminus f} m_{f' \rightarrow v}^s(\vartheta),
    \end{gather*}
    with normalization $\int m_{f \rightarrow v}^t(\vartheta) \mu_{\Theta|V}(v_v,\mathrm{d}\vartheta) = \int m_{v \rightarrow f}^t(\vartheta) \mu_{\Theta|V}(v_v,\mathrm{d}\vartheta) = 1$.
    For any variable node $v$,
    \begin{equation}\label{eq:hd-reg-message-to-posterior}
        p_v(\vartheta | \calT_{v,2t}) \propto \prod_{f \in \partial v} m_{f \rightarrow v}^{t-1}(\vartheta).
    \end{equation}
    This equation is exact.

    We define several quantities related to the belief propagation iteration.
    \begin{align*}
        \mu_{v \rightarrow f}^s &= \int \vartheta m_{v \rightarrow f}^s(\vartheta)\mu_{\Theta|V}(v_v,\mathrm{d}\vartheta), & (\tilde \tau_{v \rightarrow f}^s)^2 &= \int \vartheta^2 m_{v \rightarrow f}^s(\vartheta) \mu_{\Theta|V}(v_v,\mathrm{d}\vartheta) - (\mu_{v\rightarrow f}^s)^2,\\
        \mu_{f \rightarrow v}^s &= \sum_{v' \in \partial f \setminus v} x_{fv'} \mu_{v' \rightarrow f}^s, & (\tilde \tau_{f \rightarrow v}^s)^2 &= \sum_{v' \in \partial f \setminus v} x_{fv'}^2 (\tilde \tau_{v' \rightarrow f}^s)^2,\\
        a_{f \rightarrow v}^s &= \frac1{x_{fv}}\frac{\mathrm{d}\phantom{b}}{\mathrm{d}\vartheta} \log m_{f \rightarrow v}^s(\vartheta)\Big|_{\vartheta = 0},& b_{f \rightarrow v}^s &= - \frac1{x_{fv}^2}\frac{\mathrm{d}^2\phantom{b}}{\mathrm{d}\vartheta^2} \log m_{f \rightarrow v}^s(\vartheta) \Big|_{\vartheta = 0},\\
        a_{v \rightarrow f}^s &= \frac{\mathrm{d}\phantom{b}}{\mathrm{d}\vartheta} \log m_{v\rightarrow f}^s(\vartheta) \Big|_{\vartheta = 0},&  b_{v \rightarrow f}^s &= -\frac{\mathrm{d}^2\phantom{b}}{\mathrm{d}\vartheta^2} \log m_{v\rightarrow f}^s(\vartheta) \Big|_{\vartheta = 0},\\
        \chi_{v \rightarrow f}^s &= a_{v \rightarrow f}^s / b_{v \rightarrow f}^s, & (\tau_{v \rightarrow f}^s)^2 &= 1 / b_{v \rightarrow f}^s.
    \end{align*}
    Lemma \ref{lem:hd-reg-posterior} follows from the following asymptotic characterization of the quantities in the preceding display in the limit $n,p \rightarrow \infty$, $n/p \rightarrow \delta$:
    \begin{equation}\label{eq:bp-hd-reg-scalar-lim}
        \begin{gathered}
            \qquad\E[(\mu_{v \rightarrow f}^s)^2] \rightarrow \delta \sigma_s^2, \qquad \E[(\tilde \tau_{v \rightarrow f}^s)^2] \rightarrow \delta \ttau_s^2,\\
            (\mu_{f \rightarrow v}^s,u_f) \stackrel{\mathrm{d}}\rightarrow \normal(0,\sigma_s^2)\otimes \mu_U, \qquad (\tilde \tau_{f \rightarrow v}^s)^2 \stackrel{\mathrm{p}}\rightarrow \ttau_s^2,\\
            (\theta_v,v_v,a_{v \rightarrow f}^s/b_{v \rightarrow f}^s, b_{v \rightarrow f}^s) \stackrel{\mathrm{d}}\rightarrow (\Theta,V,\Theta + \tau_s G, 1/\tau_s^2),
        \end{gathered}
    \end{equation}
    where in the last line $\Theta \sim \mu_\Theta$, $G \sim \normal(0,1)$ independent, and $\sigma_s^2,\tau_s^2$ are defined in \eqref{bamp-se-hd-reg}.
    By symmetry, the distribution of these quantities does not depend upon $v$ or $f$, so that the limits holds for all $v,f$ once we establish them for any $v,f$.
    We establish the limits inductively in $s$.\\

    \noindent \textit{Base case: $\E[(\mu_{v \rightarrow f}^0)^2] \rightarrow \delta \sigma_0^2$ and $\E[(\ttau_{v \rightarrow f}^0)^2] \rightarrow \delta \ttau_0^2$.}

    Observe that $\mu_{v \rightarrow f}^s = \int \vartheta \mu_{\Theta|V}(v_v,\mathrm{d}\vartheta) = \E_{\Theta,V}[\Theta | V  = v_v]$. 
    Because $v_v \sim \mu_V$, we have $\E[(\mu_{v \rightarrow f}^1)^2] = \E_{\Theta,V}[\E_{\Theta,V}[\Theta|V]^2] = \E[\Theta^2] - \mmse_{\Theta,V}(\infty) = \delta \sigma_1^2$.
    Similarly,  $(\ttau_{v\rightarrow f}^1)^2 = \Var_{\Theta,V}(\Theta | V = v_v)$, so that $\E[(\ttau_{v \rightarrow f}^1)^2] = \mmse_{\Theta,V}(\infty) = \delta \ttau_0^2$.
    \\

    \noindent \textit{Inductive step 1: If $\E[(\mu_{v \rightarrow f}^s)^2] \rightarrow \delta \sigma_s^2$, then $(\mu_{f \rightarrow v}^s,u_f) \stackrel{\mathrm{d}}\rightarrow \normal(0,\sigma_s^2) \otimes \mu_U$.}

    The quantity $\mu_{v' \rightarrow f}^s$ is $\calT_{v' \rightarrow f}^{2s,0}$-measurable, whence it is independent of $x_{fv'}$ and $u_f$. 
    Moreover, $(\mu_{v'\rightarrow f},x_{fv})$ are independent as we vary $v' \in \partial f \setminus v$.
    Thus, $\mu_{f \rightarrow v}^s | \calT_{f \rightarrow v}^{2s+1,1} \sim \normal(0,\frac1n \sum_{v' \in \partial f \setminus v} (\mu_{v'\rightarrow f}^s)^2)$.
    Note that $\E[\frac1n \sum_{v' \in \partial f \setminus v} (\mu_{v'\rightarrow f}^s)^2] = (p-1) \E[(\mu_{v \rightarrow f}^s)^2]/n \rightarrow \sigma_s^2$ by the inductive hypothesis.
    Moreover, $\mu_{v\rightarrow f}^s$ has bounded fourth moments because it is bounded by $M$.
    By the weak law of large numbers, $\frac1n \sum_{v' \in \partial f \setminus v} (\mu_{v'\rightarrow f}^s)^2 \stackrel{\mathrm{p}}\rightarrow \sigma_s^2$.
    We conclude by Slutsky's theorem and independence that $(\mu_{f \rightarrow v}^s,u_f) \stackrel{\mathrm{d}}\rightarrow \normal(0,\sigma_s^2) \otimes \mu_U$.
    \\

    \noindent \textit{Inductive step 2: If $\E[(\tilde \tau_{v \rightarrow f}^s)^2] \rightarrow \delta \ttau_s^2$, then $(\tilde \tau_{f \rightarrow v}^s)^2 \stackrel{\mathrm{p}}\rightarrow \tilde \tau_s^2$.}

    The quantity $\tilde \tau_{v' \rightarrow f}^s$ is $\calT_{v' \rightarrow f}^{2s,0}$-measurable, whence it is independent of $x_{fv'}$. 
    Therefore, 
    \begin{equation*}
    \E[\sum_{v' \in \partial f \setminus v} x_{fv'}^2 (\tilde \tau_{v'\rightarrow f}^s)^2] = (p-1)\E[(\ttau_{v \rightarrow f}^s)^2]/n \rightarrow \ttau_s^2.
    \end{equation*}
    Moreover, $(\tilde \tau_{v'\rightarrow f},x_{fv})$ are mutually independent as we vary $v' \in \partial f \setminus v$, and because $\tilde \tau_{v \rightarrow f}^s$ is bounded by $M$, the terms $nx_{fv'}^2(\sigma_{v'\rightarrow f}^s)^2$ have bounded fourth moments.
    By the weak law of large numbers, $(\tilde \tau_{f \rightarrow v}^s)^2 \stackrel{\mathrm{p}}\rightarrow \tilde \tau_s^2$.\\

    \noindent \textit{Inductive step 3: If $(\mu_{f \rightarrow v}^s, u_f , \tilde \tau_{f \rightarrow v}^s) \stackrel{\mathrm{d}}\rightarrow \normal(0,\sigma_s^2) \otimes \mu_U \otimes \delta_{\tilde \tau_s}$, then $(\theta_v,v_v,a_{v \rightarrow f}^{s+1}/b_{v \rightarrow f}^{s+1},b_{v \rightarrow f}^{s+1}) \stackrel{\mathrm{d}}\rightarrow (\Theta,V,\Theta + \tau_{s+1}G,1/\tau_{s+1}^2)$ where $G\sim \normal(0,1)$ independent of $(\Theta,V) \sim \mu_{\Theta,V}$.}

    For all $(f,v) \in \calE$ and $s \geq 1$, define 
    $$
    p_{f\rightarrow v}^s(y;x) = \int p(y| x + \sum_{v' \in \partial f \setminus v} x_{fv'}\vartheta_{v'},u_f) \prod_{v' \in \partial f \setminus v} m_{v' \rightarrow f}^s(\vartheta_{v'}) \prod_{v' \in \partial f \setminus v} \mu_{\Theta|V}(v_{v'},\mathrm{d}\vartheta_{v'}).
    $$ 
    More compactly, we may write $p_{f\rightarrow v}^s(y;x,u_f) = \E_{\{\Theta_{v'}\}}[p(y| x + \sum_{v'\in \partial f \setminus v} x_{fv'} \Theta_{v'},u_f)]$, where it is understood that the expectation is taken over $\Theta_{v'}$ independent with densities $m_{v'\rightarrow f}^s$ with respect to $\mu_{\Theta|V}(v_{v'},\cdot)$.
    Note that for all $x$, we have
    $$
    \int p_{f\rightarrow v}^s(y;x) \de y = 1
    $$
    everywhere. 
    That is, $p_{f\rightarrow v}^s(\cdot;x)$ is a probability density with respect to Lebesgue measure.
    We will denote by $\dot p_{f\rightarrow v}^s(y;x) = \frac{\mathrm{d}\phantom{b}}{\mathrm{d}\xi}p_{f\rightarrow v}^s(y;x)\big|_{\xi = x}$, and likewise for higher derivatives.
    These derivatives exist and may be taken under the integral by \textsf{R4}.
    Define 
    $$
      a_{f \rightarrow v}^s(y) = \frac{\mathrm{d}\phantom{b}}{\mathrm{d}x} \log p_{f\rightarrow v}^s(y;x) \Big|_{x = 0} \qquad \text{ and } \qquad b_{f \rightarrow v}^s(y) = - \frac{\mathrm{d}^2\phantom{b}}{\mathrm{d}x^2} \log p_{f \rightarrow v}^s(y;x) \Big|_{x =0}.
    $$
    For fixed $y$, the quantity $a_{f' \rightarrow v}^s(y)$ is independent of $x_{f'v}$, and $(a_{f'\rightarrow v}^s(y),x_{f'v})$ are mutually independent for $f' \in \partial v \setminus f$.
    Observe that 
    \begin{gather*}
        a_{f \rightarrow v}^s = a_{f \rightarrow v}^s(y_f) \qquad \text{and} \qquad a_{v \rightarrow f}^{s+1} = \sum_{f' \in \partial v \setminus f} x_{f'v} a_{f' \rightarrow v}^s(y_{f'}),\\
        b_{f \rightarrow v}^s = b_{f \rightarrow v}^s(y_f) \qquad \text{and} \qquad b_{v \rightarrow f}^{s+1} = \sum_{f' \in \partial v \setminus f} x_{f'v}^2 b_{f' \rightarrow v}^s(y_{f'}).
    \end{gather*}

    We will study the distributions of $a_{f \rightarrow v}^s, a_{v \rightarrow f}^{s+1}, b_{f \rightarrow v}^s$, and $b_{v \rightarrow f}^{s+1}$ under several measures, which we now introduce.
    Define $P_{v,\vartheta}$ to be the distribution of the regression model with $\theta_v$ forced to be $\theta$ and $v_v$ forced to be 0.
    That is, under $P_{v,\theta}$, we have $(\theta_{v'},v_{v'}) \stackrel{\mathrm{iid}}\sim \mu_{\Theta,V}$ for $v' \neq v$, $v_v = 0$ and $\theta_v = \theta$, the features are distributed independently $x_{fv'} \stackrel{\mathrm{iid}}\sim \normal(0,1/n)$ for all $f,v'$, and the observations $y_f$ are drawn independently from $p(\cdot| \sum_{v' \in \partial f} x_{fv'} \theta_{v'})$ for all $f$.
    We will consider the distribution of $a_{f \rightarrow v}^s, a_{v \rightarrow f}^{s+1}, b_{f \rightarrow v}^s$, and $b_{v \rightarrow f}^{s+1}$ under $P_{v,\theta}$ for $\theta \in [-M,M]$.

    We require the following lemmas, whose proofs are deferred to Section \ref{sec:technical-tools-hd-reg}.
    \begin{lemma}\label{lem:lindeberg}
        Under $P_{v,\theta}$ for any $\theta \in [-M,M]$, we have for all fixed $y$ that
        \begin{gather*}
            p_{f \rightarrow v}^s(y;0) - \E_{G_1}[p(y| \mu_{f \rightarrow v}^s + \tilde \tau_{f \rightarrow
            v}^s G_1,u_f)] = o_p(1),\\
            \dot p_{f \rightarrow v}^s(y;0) - \E_{G_1}[\dot p(y| \mu_{f \rightarrow v}^s + \tilde \tau_{f \rightarrow v}^s G_1,u_f)] = o_p(1),\\
            \ddot p_{f \rightarrow v}^s(y;0) - \E_{G_1}[\ddot p(y|\mu_{f \rightarrow v}^s + \tilde \tau_{f \rightarrow v}^s G_1,u_f)] = o_p(1),
        \end{gather*}
        where the expectation is over $G_1 \sim \normal(0,1)$.
        Further, for any $u$, the functions $(\mu,\tilde \tau) \mapsto \E_{G_1}[p(y|\mu + \tilde \tau G_1,u)]$, $(\mu,\tilde \tau) \mapsto \E_{G_1}[\dot p(y|\mu + \tilde \tau G_1,u)]$, and $(\mu,\tilde \tau) \mapsto \E_{G_1}[\ddot p(y|\mu + \tilde \tau G_1,u)]$ are continuous.
    \end{lemma}

    \begin{lemma}\label{lem:LAN-expansion-hd-reg}
        Under $P_{v,\theta}$ for any $\theta \in [-M,M]$, we have for any fixed $s$
        \begin{equation*}
            \log \frac{m_{v\rightarrow f}^{s+1}(\vartheta)}{m_{v\rightarrow f}^{s+1}(0)} = \vartheta a_{v \rightarrow f}^{s+1} - \frac12 \vartheta^2 b_{v \rightarrow f}^{s+1} + O_p(n^{-1/2}),
        \end{equation*}
        where $O_p(n^{-1/2})$ has no $\vartheta$ dependence, and the statement holds for $\vartheta \in [-M,M]$.
    \end{lemma}

    First we study the distribution of $a_{v \rightarrow f}^{s+1}, b_{v \rightarrow f}^{s+1}$ under $P_{v,0}$.
    Because $\mu_{f' \rightarrow v}^s,\tilde \tau_{f' \rightarrow v}^s$ is independent of $\theta_v,v_v$ for all $f' \in \partial v$, its distribution is the same under $P_{v,\theta}$ for all $\theta \in [-M,M]$ and is equal to its distribution under the original model.
    Thus, the inductive hypothesis implies $(\mu_{f \rightarrow v}^s, \tilde \tau_{f \rightarrow v}^s) \xrightarrow[P_{v,0}]{\mathrm{d}} \normal(0,\sigma_s^2) \times \delta_{\tilde \tau_s}$.

    By Lemma \ref{lem:lindeberg}, the inductive hypothesis, and Lemma \ref{lem:strong-dist-convergence}, we have for fixed $y$ 
    \begin{equation*}
        \begin{pmatrix}
        \E_{G_1}[p(y|\mu_{f \rightarrow v}^s + \tilde \tau_{f \rightarrow v}^s G_1,u_f)]\\
        \E_{G_1}[\dot p(y|\mu_{f \rightarrow v}^s + \tilde \tau_{f \rightarrow v}^s G_1,u_f)]\\
        \E_{G_1}[\ddot p(y|\mu_{f \rightarrow v}^s + \tilde \tau_{f \rightarrow v}^s G_1,u_f)]
        \end{pmatrix}
        \xrightarrow[P_{v,0}]{\mathrm{d}} 
        \begin{pmatrix}
        \E_{G_1}[p(y|\sigma_s G_0 + \tilde \tau_s G_1,U]\\
        \E_{G_1}[\dot p(y|\sigma_s G_0 + \tilde \tau_s G_1,U)]\\
        \E_{G_1}[\ddot p(y|\sigma_s G_0 + \tilde \tau_s G_1,U)]
        \end{pmatrix},
    \end{equation*}
    where and $G_0,G_1 \sim \normal(0,1)$ and $U \sim \mu_U$ independent.
    Applying Lemma \ref{lem:lindeberg} and Slutsky's Theorem, 
    we have that 
    \begin{equation*}
        \begin{pmatrix}
            p_{f\rightarrow v}^s(y;0)\\
            \dot p_{f\rightarrow v}^s(y;0)\\
            \ddot p_{f\rightarrow v}^s(y;0)
        \end{pmatrix}
        \xrightarrow[P_{v,0}]{\mathrm{d}} 
        \begin{pmatrix}
        \E_{G_1}[p(y|\sigma_s G_0 + \tilde \tau_s G_1,U)]\\
        \E_{G_1}[\dot p(y|\sigma_s G_0 + \tilde \tau_s G_1,U)]\\
        \E_{G_1}[\ddot p(y|\sigma_s G_0 + \tilde \tau_s G_1,U)]
        \end{pmatrix}.
    \end{equation*}
    By the Continuous Mapping Theorem,
    \begin{gather*}
      p_{f\rightarrow v}^s(y;0) \xrightarrow[P_{v,0}]{\mathrm{d}} \E_{G_1}[p(y|\sigma_s G_0 + \ttau_s G_1,U)],\\
        a_{f\rightarrow v}^s(y) \xrightarrow[P_{v,0}]{\mathrm{d}} \frac{\mathrm{d}\phantom b}{\mathrm{d}x} \log \E_{G_1}[\dot p(y|\sigma_s G_0 + \ttau_s G_1,U)]\Big|_{x = 0},\\
        b_{f\rightarrow v}^s(y) \xrightarrow[P_{v,0}]{\mathrm{d}} -\frac{\mathrm{d}^2\phantom b}{\mathrm{d}x^2} \log \E_{G_1}[\dot p(y|\sigma_s G_0 + \ttau_s G_1,U)]\Big|_{x=0}.
    \end{gather*}
    Because the quantity $p(y|x)$ is bounded  (assumption \textsf{R4}) and the quantities $a_{f \rightarrow v}^s(y), b_{f\rightarrow v}^s(y)$ are bounded by \eqref{eq:score-bound}, 
    we have
    \begin{gather*}
      \E_{P_{v,0}}[p_{f\rightarrow v}^s(y|0)] \rightarrow \E_{G_0,G_1,U}[p(y|\sigma_s G_0 + \ttau_s G_1,U)],\\
    \E_{P_{v,0}}[a_{f\rightarrow v}^s(y)^2] \rightarrow \E_{G_0,U}\left[\left(\frac{\mathrm{d}\phantom b}{\mathrm{d}x} \log \E_{G_1}[\dot p(y|\sigma_s G_0 + \ttau_s G_1,U)]\Big|_{x = 0}\right)^2\right],\\
    \E_{P_{v,0}}[b_{f \rightarrow v}^s] \rightarrow -\E_{G_0,U}\left[\frac{\mathrm{d}^2\phantom b}{\mathrm{d}x^2} \log \E_{G_1}[\dot p(y|\sigma_s G_0 + \ttau_s G_1,U)]\Big|_{x = 0}\right].
    \end{gather*}

    Under $P_{v,0}$, we have for all $f' \in \partial v$ that the random variable $y_{f'}$ is independent of $x_{f'v}$.
    Thus, conditional on $\calT_{v\rightarrow f}^{2s+2,1}$, the random variable $\sum_{f' \in \partial v \setminus f} x_{f'v} a_{f' \rightarrow v}^s(y_{f'})$ is normally distributed. 
    Specifically,
    \begin{equation*}
        \sum_{f' \in \partial v \setminus f} x_{f'v} a_{f' \rightarrow v}^s(y_{f'}) \bigm| \calT_{v\rightarrow f}^{2s+2,1} \underset{P_{v,0}}\sim \normal \left(0,\frac1n\sum_{f' \in \partial v \setminus f} (a_{f' \rightarrow v}^s(y_{f'}))^2\right).
    \end{equation*}
    Because $(a_{f'\rightarrow v}^s(y_{f'}))^2$ is bounded by \eqref{eq:score-bound}, if we show $\E_{P_{v,0}}[(a_{f\rightarrow v}^s(y_f))^2] \rightarrow 1/\tau_{s+1}^2$, then the weak law of large numbers and Slutsky's theorem will imply that
    \begin{equation}\label{dist-convergence-under-null}
        a_{v \rightarrow f}^{s+1} = \sum_{f' \in \partial v \setminus f} x_{f'v} a_{f' \rightarrow v}^s( y_{f'}) \xrightarrow[P_{v,0}]{\mathrm{d}} \normal\left(0,1/\tau_{s+1}^2\right).
    \end{equation}
    We compute
    \begin{align*}
        \E_{P_{v,0}}[(a_{f \rightarrow v}^s(y_f))^2] &= \E_{P_{v,0}}[\E_{P_{v,0}}[(a_{f \rightarrow v}^s(y_f))^2| \sigma(\calT_{f\rightarrow v}^{2s+1,1},(x_{fv'})_{v'\in\partial f \setminus v}),u_f]] \\
        &= \E_{P_{v,0}}\left[\int a_{f\rightarrow v}^s(y)^2 p_{f\rightarrow v}^s(y;0) \de y \right] \\
        &= \int \E_{P_{v,0}}\left[ a_{f\rightarrow v}^s(y)^2 p_{f\rightarrow v}^s(y;0) \right] \de y.
    \end{align*}
    where the second equation holds because under $P_{v,0}$ we have $y_f \mid \sigma(\calT_{f\rightarrow v}^{2s+1,1},(x_{fv'})_{v'\in\partial f \setminus v},u_f)$ has density $p_{f \rightarrow v}^s(\cdot;0)$ with respect to Lebesgue measure, and the last equation follows by Fubini's theorem (using the non-negativity of the integrand).
    Because $a_{f\rightarrow v}^s(y)^2 \leq (q_1')^2$ and  $\E_{P_{v,0}}[p_{f\rightarrow v}^s(y;0)]$ are probability densities which converge pointwise to $\E_{G_1}[p(y|\sigma_s G_0 + \ttau_s G_1)]$, we conclude that 
    \begin{align*}
    \E_{P_{v,0}}[(a_{f \rightarrow v}^s(y_f))^2] &\rightarrow \int \E_{G_0,U}\left[\frac{\E_{G_1}[\dot p(y|\sigma_s G_0 + \tilde \tau_s G_1 , U )]^2}{\E_{G_1}[p(y|\sigma_sG_0 + \tilde \tau_s G_1 , U )]}\right] \de y \\
    &=
    \E_{G_0,U}\left[\int \frac{\E_{G_1}[\dot p(y|\sigma_s G_0 + \tilde \tau_s G_1 , U )]^2}{\E_{G_1}[p(y|\sigma_sG_0 + \tilde \tau_s G_1 , U )]}\de y\right] = \frac1{\tau_{s+1}^2},
    \end{align*}
    where we have used the alternative characterization of the recursion \eqref{bamp-se-hd-reg} from Lemma \ref{lem:posterior-to-score}.
    We conclude \eqref{dist-convergence-under-null}.

    Now we compute the asymptotic behavior of $b_{v \rightarrow f}^{s+1}$ under $P_{v,0}$.
    Under $P_{v,0}$, $x_{f'v}$ is independent of $y_{f'}$, and $(x_{f'v}, b_{f'\rightarrow v}^s(y_{f'}))$ are mutually independent for $f' \in \partial v \setminus f$. 
    Thus, $\E_{P_{v,0}}[x_{f'v}^2b_{f'\rightarrow v}^s(y_{f'})] = \E_{P_{v,0}}[b_{f'\rightarrow v}^s(y_{f'})]/n$.
    Because $b_{f'\rightarrow v}^s(y_{f'})$ is bounded by \eqref{eq:score-bound}, if we can show that $\E_{P_{v,0}}[b_{f'\rightarrow v}^s(y_{f'})] \rightarrow 1/\tau_{s+1}^2$, then $b_{v \rightarrow f}^{s+1} \xrightarrow[P_{v,0}]{\mathrm{p}} 1/\tau_{s+1}^2$ will follow by the weak law of large numbers.
    We compute 
    \begin{align*}
        \E_{P_{v,0}}[b_{f \rightarrow v}^s(y_f)] &= \E_{P_{v,0}}[\E_{P_{v,0}}[b_{f \rightarrow v}^s(y_f)| \sigma(\calT_{f\rightarrow v}^{2s+1,1},(x_{fv'})_{v'\in\partial f \setminus v},u_f)]] \\
        &= \E_{P_{v,0}}\left[\int b_{f\rightarrow v}^s(y) p_{f\rightarrow v}^s(y;0) \de y \right] \\
        &= \int \E_{P_{v,0}}\left[ b_{f\rightarrow v}^s(y) p_{f\rightarrow v}^s(y;0) \right] \de y
    \end{align*}
    where the last equation follows by Fubini's theorem (using that the integrand is bounded by the integrable function $q_2 \E_{P_{v,0}}[p_{f\rightarrow v}^s(y;0)]$). 
    The integrands converge point-wise, so that 
    \begin{align*}
        \E_{P_{v,0}}&[b_{f \rightarrow v}^s(y_f)] \\
        &\rightarrow \E_{G_0,U}\left[\int \frac{\E_{G_1}[\dot p(y|\sigma_s G_0 + \tilde \tau_s G_1 , U )]^2}{\E_{G_1}[p(y|\sigma_sG_0 + \tilde \tau_s G_1 , U )]}\de y\right] - \int \E_{G_0,G_1,U}[\ddot p(y|\sigma_s G_0 + \tilde \tau_s G_1 , U )] \de y \\
        &= \frac1{\tau_{s+1}^2},
    \end{align*}
    where we have concluded that the second integral is zero because $x \mapsto \E_{G_0,G_1,U}[p(y|\sigma_s G_0 + \ttau_s G_1,U)]$ parameterizes a statistical model whose scores up to order 3 are bounded by \eqref{eq:score-bound}.
    Thus, we conclude that $b_{v \rightarrow f}^{s+1} \xrightarrow[P_{v,0}]{\mathrm{p}} 1/\tau_{s+1}^2$.

    Now we compute the asymptotic distribution of $(a_{v\rightarrow f}^{s+1}, b_{v \rightarrow f}^{s+1})$ under $P_{v,\theta}$ for any $\theta \in [-M,M]$.
    The log-likelihood ratio between $P_{v,\theta}$ and $P_{v,0}$ is 
    \begin{align*}
        \sum_{f' \in \partial v} \log \frac{p_{f' \rightarrow v}^s(y_{f'}| x_{f'v}\theta )}{p_{f'\rightarrow v}^s(y_{f'}|0)} &= \log \frac{m_{v \rightarrow f}^{s+1}(\theta)}{m_{v\rightarrow f}^{s+1}(0)} + \log \frac{p_{f \rightarrow v}^s(y_f|x_{fv}\theta)}{p_{f\rightarrow v}^s(y_f|0)}\\
        &= \theta a_{v \rightarrow f}^{s+1} - \frac12 \theta^2 b_{v \rightarrow f}^{s+1} + O_p(n^{-1/2}) ,
    \end{align*}
    where we have used Lemma \ref{lem:LAN-expansion-hd-reg} and that $\left|\log \frac{p_{f \rightarrow v}^s(y_f|x_{fv}\theta)}{p_{f\rightarrow v}^s(y_f|0)}\right| \leq Mq_1|x_{fv}| = O_p(n^{-1/2})$.
    Thus,
    \begin{equation*}
        \left(a_{v \rightarrow f}^{s+1}, b_{v\rightarrow f}^{s+1}, \log \frac{P_{v,\theta}}{P_{v,0}}\right) \xrightarrow[P_{v,0}]{\mathrm{p}} \left(Z, \frac1{\tau_{s+1}^2} , \theta Z - \frac12 \frac{\theta^2}{\tau_{s+1}^2}\right),
    \end{equation*}
    where $Z \sim \normal(0,1/\tau_{s+1}^2)$.
    By Le Cam's third lemma \cite[Example 6.7]{vaart_1998}, we have 
    \begin{equation*}
        (a_{v \rightarrow f}^{s+1}, b_{v \rightarrow f}^{s+1}) \xrightarrow[P_{v,\theta}]{\mathrm{d}} \left(Z' , \frac1{\tau_{s+1}^2}\right).
    \end{equation*}
    where $Z' \sim \normal( \theta/ \tau_{s+1}^2 , 1 / \tau_{s+1}^2 )$.
    By the Continuous Mapping Theorem \cite[Theorem 2.3]{vaart_1998}, we conclude $(a_{v \rightarrow f}^{s+1}/ b_{v \rightarrow f}^{s+1} , b_{v \rightarrow f}^{s+1}) \xrightarrow[P_{v,\theta}]{\mathrm{d}} \normal(\theta, \tau_{s+1}^2 ) \otimes \delta_{1/\tau_{s+1}^2}$.

    Denote by $P^*$ the distribution of the the original model.
    Consider a continuous bounded function $f: (\theta , \nu , \chi , b ) \mapsto \reals$, and define $\hat f_n(\theta,\nu) = \E_{P_{v,\theta}}[f(\theta,\nu,a_{v\rightarrow f}^{s+1}/b_{v \rightarrow f}^{s+1}, b_{v \rightarrow f}^{s+1})]$.
    Under $P^*$, the random variables $a_{v\rightarrow f}^{s+1}, b_{v \rightarrow f}^{s+1}$ are functions are $\theta_v$ and random vectors $\bD := \calT_{v,2t} \setminus \{\theta_v,v_v\}$, which is independent of $\theta_v,v_v$.
    In particular, we may write
    \begin{equation*}
      \E_{P^*}[f(\theta_v,v_v, a_{v \rightarrow f}^{s+1}/ b_{v \rightarrow f}^{s+1} , b_{v \rightarrow f}^{s+1} )] = \E_{P^*}[f(\theta_v,v_v, \chi(\theta_v,\bD) , B(\theta_v,\bD))],
    \end{equation*}
    for some measurable functions $\chi,B$.
    We see that 
    \begin{equation*}
      \E_{P^*}[f(\theta_v,v_v, a_{v \rightarrow f}^{s+1}/ b_{v \rightarrow f}^{s+1} , b_{v \rightarrow f}^{s+1} )\mid \theta_v,v_v ] = \hat f_n(\theta_v,v_v)
    \end{equation*}
    where 
    \begin{equation*}
      \hat f_n(\theta,\nu) = \E_{\bD}[f(\theta,\nu,\chi(\theta,\bD),B(\theta,\bD))],
    \end{equation*}
    with $\bD$ distributed as it is under $P^*$ (see e.g., \cite[Example 5.1.5]{Durrett2010Probability:Examples}). 
    Because $\bD$ has the same distribution on $P^*$ as under $P_{v,\theta}$, we see that in fact $\hat f_n(\theta,\nu) = \E_{P_{v,\theta}}[f(\theta,\nu,a_{v \rightarrow f}^{s+1}/ b_{v \rightarrow f}^{s+1} , b_{v \rightarrow f}^{s+1})]$.
    Because $(a_{v \rightarrow f}^{s+1}/ b_{v \rightarrow f}^{s+1} , b_{v \rightarrow f}^{s+1}) \xrightarrow[P_{v,\theta}]{\mathrm{d}} \normal(\theta, \tau_{s+1}^2 ) \otimes \delta_{1/\tau_{s+1}^2}$, we conclude that $\hat f_n(\theta,\nu) \rightarrow \E_{G}[f(\theta,\nu,\theta + \tau_{s+1}G,\tau_{s+1}^{-2})]$ for all $\theta,\nu$.
    By bounded convergence and the tower property, $\E_{\Theta,V}[\hat f_n(\Theta,V)] \rightarrow \E_{\Theta,V,G}[f(\theta,\nu,\theta + \tau_{s+1}G,\tau_{s+1}^{-2})]$ where $(\Theta,V) \sim \mu_{\Theta,V}$ independent of $G \sim \normal(0,1)$.
    Also by the tower property, 
    we have 
    \begin{equation*}
      \E_{\Theta,V}[\hat f_n(\Theta,V)] = \E_{P^*}[f(\theta_v,v_v,\chi(\theta_v,\bD),B(\theta_v,\bD))] = \E_{P^*}[f(\theta_v,v_v,a_{v\rightarrow f}^{s+1}/b_{v \rightarrow f}^{s+1}, b_{v \rightarrow f}^{s+1})].
    \end{equation*}
    We conclude
    \begin{equation*}
      \E_{P^*}[f(\theta_v,v_v,a_{v\rightarrow f}^{s+1}/b_{v \rightarrow f}^{s+1}, b_{v \rightarrow f}^{s+1})] \rightarrow \E_{\Theta,V,G}[f(\Theta,V,\Theta + \tau_{s+1} G, \tau_{s+1}^{-2})].
    \end{equation*}
    Thus, we conclude that $(\theta_v, v_v, a_{v\rightarrow f}^{s+1}/b_{v \rightarrow f}^{s+1}, b_{v \rightarrow f}^{s+1}) \xrightarrow[P^*]{\mathrm{d}} (\Theta,V,\Theta + \tau_{s+1}G, 1/\tau_{s+1}^2)$, as desired.
    \\

    \noindent \textit{Inductive step 4: If $(\theta_v,v_v,a_{v \rightarrow f}^{s+1}/b_{v \rightarrow f}^{s+1},b_{v \rightarrow f}^{s+1}) \stackrel{\mathrm{d}}\rightarrow (\Theta,V,\Theta + \tau_{s+1}G,1/\tau_{s+1}^2)$ where $G\sim \normal(0,1)$ independent of $(\Theta,V) \sim \mu_{\Theta,V}$, then $\E[(\mu_{v \rightarrow f}^s)^2] \rightarrow \delta \sigma_s^2$ and $\E[(\tilde \tau_{v \rightarrow f}^s)^2] \rightarrow \mmse_{\Theta,V}(\tau_s^2)$.}

    Define
    \begin{equation*}
        \epsilon_{v \rightarrow f}^s = \sup_{\vartheta \in [-M,M]} \left|\log \frac{m_{v \rightarrow f}^s(\vartheta)}{m_{v\rightarrow f}^s(0)} - \left(\vartheta a_{v \rightarrow f}^s - \frac12 \vartheta^2 b_{v \rightarrow f}^s\right)\right|,
    \end{equation*}
    where because all the terms are continuous in $\vartheta$, the random variable $\epsilon_{v \rightarrow f}^s$ is measurable and finite.
    We have that 
    \begin{equation*}
        \mu_{v \rightarrow f}^s \geq \frac{\int \vartheta \exp(\vartheta a_{v \rightarrow f}^s - \vartheta^2 b_{v\rightarrow f}^s / 2 - \epsilon_{v \rightarrow f}^s) \mu_\Theta(v_v,\mathrm{d}\vartheta)}{\int \exp(\vartheta a_{v \rightarrow f}^s - \vartheta^2 b_{v\rightarrow f}^s / 2 + \epsilon_{v \rightarrow f}^s) \mu_\Theta(v_v,\mathrm{d}\vartheta)} \geq e^{-2\epsilon_{v \rightarrow f}^s} \eta_{\Theta,V}( a_{v\rightarrow f}^s / b_{v \rightarrow f}^s , v_v; 1/b_{v \rightarrow f}^s )
    \end{equation*}
    where $\eta_{\Theta,V}(y,v ; \tau^2 ) = \E_{\Theta,V,G}[\Theta | \Theta + \tau G = y; V=v]$ where $(\Theta,V) \sim \mu_{\Theta,V}$, $G \sim \normal(0,1)$ independent.
    Likewise, 
    $$
    \mu_{v \rightarrow f}^s \leq e^{2 \epsilon_{v \rightarrow f}^s} \eta_{\Theta,V}(a_{v\rightarrow f}^s / b_{v\rightarrow f}^s , v_v; 1 / b_{v \rightarrow f}^s).
    $$
    Because $\eta_{\Theta,V}$ takes values in the bounded interval $[-M,M]$ and $\epsilon_{v \rightarrow f} = o_p(1)$ by Lemma \ref{lem:LAN-expansion-hd-reg}, we conclude that 
    \begin{equation*}
      \mu_{v \rightarrow f}^s = \eta_{\Theta,V}(a_{v\rightarrow f}^s / b_{v\rightarrow f}^s , v_v; 1 / b_{v \rightarrow f}^s) + o_p(1).
    \end{equation*}
  For a fixed $v_v$, the Bayes estimator $\eta_{\Theta,V}$ is continuous in the observation and the noise variance on $\reals \times \reals_{>0}$.\footnote{This commonly known fact holds, for example, by \cite[Theorem 2.7.1]{Lehmann2005TestingHypotheses} because the posterior mean can be viewed as the mean in an exponential family paramterized by the observation and noise variance.} 
  Thus, by the inductive hypothesis and the fact that $v_v \sim \mu_V$ for all $n$, we have $\E[\eta_{\Theta,V}(a_{v\rightarrow f}^s / b_{v\rightarrow f}^s , v_v; 1 / b_{v \rightarrow f}^s)^2] = \E[\eta_{\Theta,V}(a_{v\rightarrow f}^s / b_{v\rightarrow f}^s , v_v; 1 / b_{v \rightarrow f}^s)^2 \vee M^2] \rightarrow \E_{\Theta,V,G}[\eta_{\Theta,V}(\Theta + \tau_s G,V;\tau_s^2)] = \E[\Theta^2] - \mmse_{\Theta,V}(\tau_s^2) = \delta \sigma_s^2$ by Lemma \ref{lem:strong-dist-convergence}.
  By the previous display and the boundedness of $\mu_{v \rightarrow f}^s$ and $\eta_{\Theta,V}$, we conclude $\E[(\mu_{v\rightarrow f}^s)^2] \rightarrow \delta \sigma_s^2$, as desired.

    Similarly, we may derive that
    \begin{align*}
        e^{-2 \epsilon_{v \rightarrow f}^s} s_{\Theta,V}^2( a_{v \rightarrow f}^s / b_{v \rightarrow f}^s , v_v ; 1 / b_{v \rightarrow f}^s) &\leq \int \vartheta^2 m_{v \rightarrow f}^s(\vartheta) \mu_\Theta(\mathrm{d}\vartheta) \\
        & \leq e^{2 \epsilon_{v \rightarrow f}^s} s_{\Theta,V}^2( a_{v \rightarrow f}^s / b_{v \rightarrow f}^s , v_v ; 1 / b_{v \rightarrow f}^s),
    \end{align*}
    where $s_{\Theta,V}^2(y,v;\tau^2) = \E_{\Theta,V,G}[\Theta^2 | \Theta + \tau G = y, V=v]$ where $(\Theta,V) \sim \mu_{\Theta,V}$, $G \sim \normal(0,1)$ independent.
    For fixed $v_v$, the the posterior second moment is continuous in the observation and the noise variance.
    Further, it is bounded by $M^2$.
    Thus, by exactly the same argument as in the previous paragraph, we have that $\E[(\ttau_{v\rightarrow f}^s)^2] \rightarrow \E_{\Theta,V,G}[s_{\Theta,V}^2(\Theta + \tau_s G,V; \tau_s^2) - \eta_{\Theta,V}(\Theta + \tau_s G,V; \tau_s^2)^2] = \mmse_{\Theta,V}(\tau_s^2)$, as desired.

    The inductive argument is complete, and \eqref{eq:bp-hd-reg-scalar-lim} is established.

    To complete the proof of Lemma \ref{lem:hd-reg-posterior}, first observe by \eqref{eq:hd-reg-message-to-posterior} that we may express $\log p_v(\vartheta|\calT_{v,2t})$ as, up to a constant, $\log \frac{m_{v \rightarrow f}^t(\vartheta)}{m_{v \rightarrow f}^t(0)} + \log \frac{m_{f \rightarrow v}^{t-1}(\vartheta)}{m_{f \rightarrow v}^{t-1}(0)}$.
    Note that 
    \begin{equation*}
        \left|\log \frac{m_{f \rightarrow v}^{t-1}(\vartheta_v)}{m_{f \rightarrow v}^{t-1}(0)}\right| \leq M|x_{fv}| \sup_{x \in \reals} \left|\frac{\dot p_{f \rightarrow v}^{t-1}(y_f;x)}{p_{f \rightarrow v}^{t-1}(y_f;x)}\right| \leq Mq_1|x_{fv}| = o_p(1).
    \end{equation*}
    By Lemma \ref{lem:LAN-expansion-hd-reg}, we have that, up to a constant, $\log \frac{m_{v \rightarrow f}^t(\vartheta)}{m_{v\rightarrow f}^t(0)} = -\frac12 b_{v \rightarrow f}^s\left(a_{v \rightarrow f}^t / b_{v \rightarrow f}^t - \vartheta \right)^2 + o_p(1)$.
    The lemma follows from \eqref{eq:bp-hd-reg-scalar-lim}.
\end{proof}

We complete the proof of Lemma \ref{lem:local-info-theory-lb} for the high-dimensional regression model.
Consider any estimator $\hat \theta: \calT_{v,2t} \mapsto [-M,M]$ on the computation tree.
We compute
\begin{align*}
  &\E[\ell(\theta_v,\hat \theta(\calT_{v,2t}))] = 
  \E[\E[\ell(\theta_v,\hat \theta(\calT_{v,2t}))|\calT_{v,2t}]] \\
  &\qquad= \E\left[\int \ell(\vartheta , \hat \theta(\calT_{v,2t})) \frac1{Z(\calT_{v,2t})}\exp\left(-\frac1{2\tau_{v,t}^2}(\chi_{v,t} - \vartheta)^2 + o_p(1)\right) \mu_{\Theta|V}(v_v,\mathrm{d}\vartheta)\right] \\
  &\qquad\geq \E\left[\exp(-2\epsilon_v)\int \ell(\vartheta , \hat \theta(\calT_{v,2t})) \frac1{Z(\chi_{v,t},\tau_{v,t},v_v)} \exp\left(-\frac1{2\tau_{v,t}^2}(\chi_{v,t} - \vartheta)^2\right)\mu_{\Theta|V}(v_v,\mathrm{d}\vartheta)\right] \\
  &\qquad\geq \E\left[\exp(-2\epsilon_v)R(\chi_{v,2t},\tau_{v,2t},v_v)\right],
\end{align*} 
where $Z(\calT_{v,2t}) = \int \exp\left(-\frac1{2\tau_{v,t}^2}(\chi_{v,t} - \vartheta)^2 + o_p(1)\right) \mu_{\Theta|V}(v_v,\mathrm{d}\vartheta)$, 
\begin{equation*}
	R(\chi,\tau,v) := \inf_{d \in \reals} \int \frac1Z \ell(\vartheta,d) e^{-\frac1{2\tau^2}(\chi - \vartheta)^2} \mu_{\Theta|V}(v,\de \vartheta)\,,
\end{equation*}
and 
$$
\epsilon_v = \sup_{\vartheta \in [-M,M]} \left|\log \frac{p(\vartheta|\calT_{v,2t})}{p(0|\calT_{v,2t})} + \vartheta \chi_{v,t}/\tau_{v,t}^2 - \vartheta^2 / (2 \tau_{v,t}^2)\right|.
$$ 
Because $\Theta$ is bounded support, by Lemma  \ref{lem:properties-of-bayes-risk}(b), $R(\chi,\tau,v)$ is continuous in $(\chi,\tau)$ on $\reals \times \reals_{> 0}$.
By Lemma \ref{lem:hd-reg-posterior}, $\epsilon_v = o_p(1)$.
The quantity on the right-hand side does not depend on $\hat \theta$, so provides a uniform lower bound over the performance of any estimator.
Because $(v_v,\chi_{v,2t},\tau_{v,2t},\epsilon_v) \stackrel{\mathrm{d}}\rightarrow (V,\Theta+\tau_t G, \tau_t,0)$, $v_v \stackrel{\mathrm{d}}= V$ for all $n$, and $\tau_t > 0$, we have $\E\left[\exp(-2\epsilon_v)R(\chi_{v,2t},\tau_{v,2t},v_v)\right] \rightarrow \E[R(\Theta + \tau_t G,\tau_t,V)] = \inf_{\hat \theta(\cdot)} \E[\ell(\Theta,\hat \theta(\Theta + \tau_t G,V))]$, where the convergence holds by Lemma \ref{lem:strong-dist-convergence} and the equality holds by Lemma \ref{lem:properties-of-bayes-risk}(a).
Thus, 
\begin{equation*}
  \liminf_{n \rightarrow \infty} \inf_{\hat \theta(\cdot)} \E[\ell(\theta_v,\hat \theta(\calT_{v,2t}))] \geq \inf_{\hat \theta(\cdot)} \E[\ell(\Theta,\hat \theta(\Theta + \tau_t G))].
\end{equation*}
The proof of Lemma \ref{lem:local-info-theory-lb} in the high-dimensional regression model is complete.

\subsubsection{Technical tools}\label{sec:technical-tools-hd-reg}

\begin{proof}[Lemma \ref{lem:lindeberg}]
    By Lindeberg's principle (see, e.g., \cite{chatterjee2006}) and using that $\mu_\Theta$ is supported on $[-M,M]$, we have
    \begin{gather*}
            |p_{f \rightarrow v}^s(y;0) - \E_{G_1}[p(y| \mu_{f \rightarrow v}^s + \tilde \tau_{f \rightarrow
            v}^s G_1,u_f)]| \leq \frac{M^3\sup_{x \in \reals} |\partial_x^3p(y|x,u_f)|}3 \sum_{v' \in \partial f \setminus v}|x_{fv'}|^3,\\
            |\dot p_{f \rightarrow v}^s(y;0) - \E_{G_1}[\dot p(y| \mu_{f \rightarrow v}^s + \tilde \tau_{f \rightarrow v}^s G_1,u_f)]| \leq \frac{M^3\sup_{x \in \reals} |\partial_x^4p(y|x,u_f)|}3 \sum_{v' \in \partial f \setminus v}|x_{fv'}|^3,\\
            |\ddot p_{f \rightarrow v}^s(y;0) - \E_{G_1}[\ddot p(y|\mu_{f \rightarrow v}^s + \tilde \tau_{f \rightarrow v}^s G_1,u_f)]| \leq \frac{M^3\sup_{x \in \reals} |\partial_x^5p(y|x,u_f)|}3 \sum_{v' \in \partial f \setminus v}|x_{fv'}|^3.
        \end{gather*}
        Using the $\sup_{x \in \reals}|\partial_x^k p(y|x,u)| \leq q_k' \sup_{x \in \reals} |p(y|x,u)| < \infty$ for $k = 3,4,5$ by \textsf{R4}, we have that for fixed $y$ the expectations on the right-hand side go to 0 as $n \rightarrow \infty$, whence the required expessions are $o_p(1)$.

        Further, $|\E_{G_1}[p(y|\mu + \ttau G_1,u)] - \E_{G_1}[p(y|\mu' + \ttau' G_1,u)]| \leq (|\mu - \mu'| + |\ttau -\ttau'|\sqrt{2/\pi})\sup_{x \in \reals}|\dot p(y|x,u)|$, whence $\E_{G_1}[p(y|\mu + \ttau G_1,u)]$ is continuous in $(\mu,\ttau)$ by \textsf{R4}.
        The remaining continuity results follow similarly.
\end{proof}

\begin{proof}[Lemma \ref{lem:LAN-expansion-hd-reg}]
    Fix any $\vartheta \in [-M,M]$.
    By Taylor's theorem, there exist $\vartheta_i \in [-M,M]$ (in fact, between $0$ and $\vartheta$) such that
    \begin{align*}
        \log &\frac{m_{v \rightarrow f}^{s+1}(\vartheta)}{m_{v \rightarrow f}^{s+1}(0)} = \sum_{f' \in \partial v \setminus f} \log \frac{m_{f' \rightarrow v}^s(\vartheta)}{m_{f' \rightarrow v}^s(0)} \\ 
        &= \vartheta a_{v \rightarrow f}^{s+1} - \frac12 \vartheta^2 b_{v \rightarrow f}^{s+1} + \frac16\vartheta^3 \sum_{f' \in \partial v \setminus f} \left(\frac{\mathrm{d}^3}{\mathrm{d}\vartheta^3} \log \E_{\hat G_{f'}}[p(y_{f'}|x_{f'v}\vartheta + \hat G_{f'},u_{f'})] \bigg|_{\vartheta=\vartheta_i} \right).
    \end{align*}
    where it is understood that expectation is taken with respect to $\hat G_{f'} \stackrel{\mathrm{d}}= \sum_{v' \in \partial f' \setminus v} x_{f'v'} \Theta_{v' \rightarrow f'}$ where $x_{f'v'}$ is considered fixed and $\Theta_{v'\rightarrow f'}$ are drawn independently with densities $m_{v'\rightarrow f'}^s$ with respect to $\mu_{\Theta|V}(v_{v'},\cdot)$.
    We bound the sum using assumption \textsf{R4}:
    \begin{align*}
      \left|\sum_{f' \in \partial v \setminus f} \left(\frac{\mathrm{d}^3}{\mathrm{d}\vartheta^3} \log \E_{\hat G_{f'}}[p(y_f|x_{fv}\vartheta + \hat G_{f'},u_{f'})] \bigg|_{\vartheta=\vartheta_i} \right)\right|
      &\leq q_3\sum_{f' \in \partial v \setminus f} |x_{f'v}|^3 = O_p(n^{-1/2}).
    \end{align*}
    The proof is complete.
\end{proof}

\subsection{Information-theoretic lower bound in the low-rank matrix estimation model}

In this section, we prove Lemma \ref{lem:local-info-theory-lb} in the low-rank matrix estimation model.

Recall that conditions on the conditional density in assumption \textsf{R4} are equivalent positivity, boundedness, and the existence finite, non-negative constants $q_k'$ such that $\frac{|\partial_x^k p(y|x)|}{p(y|x)} \leq q_k'$ for $1 \leq k \leq 5$.
In particular, we have \eqref{eq:score-bound} for any random variable $A$.

Denote the regular conditional probability of $\Theta$ conditional on $V$ for the measure $\mu_{\Theta,V}$ by $\mu_{\Theta|V} : \reals \times \calB \rightarrow [0,1]$, where $\calB$ denotes the Borel $\sigma$-algebra on $\reals$, similarly for $\mu_{\Lambda|U}$.
The posterior density of $\theta_v$ given $\calT_{v,2t-1}$
has density respect to $\mu_{\Theta|V}(v_v,\cdot)$ given by
\begin{equation*}
    p_v(\vartheta_v | \calT_{v,2t-1}) \propto \int \prod \exp\left(-\frac{n}{2} (x_{f'v'} - \frac1n \ell_{f'}\vartheta_{v'})^2\right)\prod \mu_\Lambda(u_f,\mathrm{d}\ell_f) \prod \mu_\Theta(v_{v'},\mathrm{d}\vartheta_{v'}),
\end{equation*}
where the produces are over $(f',v') \in \calE_{v,2t-1}$, $f \in  \calF_{v,2t-1}$, and $v' \in \calV_{v,2t-1}$, respectively.
Asymptotically, the posterior behaves like that produced by a Gaussian observation of $\theta_v$ with variance $\tau_t^2$.
\begin{lemma}\label{lem:lr-mat-posterior}
    In the low-rank matrix estimation model, there exist $\calT_{v,2t-1}$-measurable random variables $q_{v,t},\chi_{v,t}$ such that for fixed $t\geq1$
     \begin{equation*}
         p_v( \vartheta | \calT_{v,2t-1} ) \propto \exp\left(-\frac12(\chi_{v,t} - q_{v,t}^{1/2} \vartheta)^2 + o_p(1)\right),
     \end{equation*}
     where $o_p(1)$ has no $\vartheta$ dependence.
     Moreover, $(\theta_v,v_v,\chi_{v,t},q_{v,t}) \stackrel{\mathrm{d}}\rightarrow (\Theta,V,q_t^{1/2}\Theta + G, q_t)$ where $(\Theta,V) \sim \mu_{\Theta,V}, G \sim \normal(0,1)$ independent of $\Theta,V$, and $q_t$ is given by \eqref{eq:SE_Matrix}.
\end{lemma}

\begin{proof}[Lemma \ref{lem:lr-mat-posterior}]
  As in the proof of Lemma \ref{lem:hd-reg-posterior}, we compute the posterior density $p_v(\vartheta|\calT_{v,2t-1})$ via belief propogation.
  The belief propagation iteration is
  \begin{gather*}
      m_{f\rightarrow v}^0(\ell) = 1,\\
      m_{v\rightarrow f}^{s+1}(\vartheta) \propto \int \prod_{f' \in \partial v \setminus f} \left(\exp\left(-\frac{n}{2} (x_{f'v} - \frac1n \ell_{f'} \vartheta)^2\right) m_{f' \rightarrow v}^s(\ell_{f'})  \mu_{\Lambda|U}(u_{f'},\mathrm{d}\ell_{f'}) \right),\\
      m_{f\rightarrow v}^s(\ell) \propto \int \prod_{v' \in \partial f \setminus v} \left(\exp\left(-\frac{n}{2} (x_{fv'} - \frac1n \ell \vartheta_{v'})^2\right) m_{v' \rightarrow f}^s(\vartheta_{v'})  \mu_{\Theta|V}(v_{v'},\mathrm{d}\vartheta_{v'}) \right),
  \end{gather*}
  with normalization $\int m_{f \rightarrow v}^s(\ell) \mu_{\Lambda|U}(u_f,\mathrm{d}\ell) = \int m_{v \rightarrow f}^s(\vartheta) \mu_{\Theta|V}(v_v,\mathrm{d}\vartheta) = 1$.
  For $t \geq 1$
  \begin{gather*}
      p_v(\vartheta|\calT_{v,2t-1}) \propto \int \prod_{f \in \partial v} \left(\exp\left(-\frac{n}{2} (x_{fv} - \frac1n \ell_f \vartheta)^2\right) m_{f \rightarrow v}^{t-1}(\ell_f)  \mu_{\Lambda|U}(u_f,\mathrm{d}\ell_f) \right),
  \end{gather*}
  This equation is exact.

  We define several quantities related to the belief propagation iteration.
  \begin{align*}
      \mu_{f \rightarrow v}^s &= \int \ell m_{f \rightarrow v}^s(\ell) \mu_{\Lambda|U}(u_f,\mathrm{d}\ell), & s_{f \rightarrow v}^s &= \int \ell^2 m_{f \rightarrow v}^s(\ell) \mu_{\Lambda|U}(u_f,\mathrm{d}\ell),\\
      \alpha_{v \rightarrow f}^{s+1} &= \frac1n \sum_{f' \in \partial v \setminus f} \mu_{f' \rightarrow v}^s \lambda_{f'}, & (\tau_{v \rightarrow f}^{s+1})^2 &= \frac1n \sum_{f' \in \partial v \setminus f} (\mu_{f' \rightarrow v}^s)^2,\\
      a_{v \rightarrow f}^s &= \frac{\mathrm{d}\phantom{b}}{\mathrm{d}\vartheta} \log m_{v \rightarrow f}^s(\vartheta) \Big|_{\vartheta = 0}, & b_{v \rightarrow f}^s &= - \frac{\mathrm{d}^2\phantom{b}}{\mathrm{d}\vartheta^2} \log m_{v \rightarrow f}^s(\vartheta)\Big|_{\vartheta = 0},\\
      \mu_{v \rightarrow f}^s &= \int \vartheta m_{v \rightarrow f}^s(\vartheta) \mu_{\Theta|V}(v_v,\mathrm{d}\vartheta), & s_{v \rightarrow f}^s &= \int \vartheta^2 m_{v \rightarrow f}^s(\vartheta) \mu_{\Theta|V}(v_v,\mathrm{d}\vartheta),\\
      \alpha_{f \rightarrow v}^s &= \frac1n \sum_{v' \in \partial f \setminus v} \mu_{v' \rightarrow f}^s \theta_{v'} , & (\hat{\tau}_{f \rightarrow v}^s)^2 &= \frac1n\sum_{v' \in \partial f \setminus v} (\mu_{v'\rightarrow f}^s)^2,\\
      a_{f \rightarrow v}^s &= \frac{\mathrm{d}\phantom{b}}{\mathrm{d}\ell} \log m_{f \rightarrow v}^s(\ell)\Big|_{\ell = 0}, & b_{f \rightarrow v}^s &= -\frac{\mathrm{d}^2\phantom{b}}{\mathrm{d}\ell^2} \log m_{f \rightarrow v}^s(\ell)\Big|_{\ell = 0},\\
  \end{align*}
  Lemma \ref{lem:lr-mat-posterior} follows from the following asymptotic characterization of the quantities in the preceding display in the limit $n,p \rightarrow \infty$, $n/p \rightarrow \delta$:
  \begin{equation}\label{eq:bp-lr-mat-scalar-lim}
      \begin{gathered}
        \E[\mu_{f\rightarrow v}^s\lambda_f] \rightarrow q_{s+1}, \qquad \E[(\mu_{f \rightarrow v}^s)^2] \rightarrow q_{s+1},\\
        \alpha_{v \rightarrow f}^{s+1} \stackrel{\mathrm{p}}\rightarrow q_{s+1}, \qquad (\tau_{v\rightarrow f}^{s+1}) \stackrel{\mathrm{p}}\rightarrow q_{s+1},\\
        (\theta_v,v_v,a_{v\rightarrow f}^s,b_{v\rightarrow f}^s) \stackrel{\mathrm{d}}\rightarrow (\Theta,V,q_s \Theta + q_s^{1/2}G_2,q_s),\\
      \E[\mu_{v\rightarrow f}^s\theta_v] \rightarrow  \delta\hq_s, \qquad \E[(\mu_{v\rightarrow f}^s)^2] \rightarrow  \delta\hq_s,\\
      \alpha_{f \rightarrow v}^s \stackrel{\mathrm{p}}\rightarrow \hq_s, \qquad (\hat{\tau}_{f \rightarrow v}^{s+1})^2 \stackrel{\mathrm{p}}\rightarrow \hq_s,\\
      (\lambda_f , u_f , a_{f \rightarrow v}^s,b_{f\rightarrow v}^s) \stackrel{\mathrm{d}}\rightarrow (\Lambda,U,\hq_s\Lambda + \hq_s^{1/2} G, \hq_s).
      \end{gathered}
  \end{equation}
  As in the proof of Lemma \ref{lem:hd-reg-posterior}, the distribution of these quantities does not depend upon $v$ or $f$, so that the limits hold for all $v,f$ once we establish them for any $v,f$.
  We establish the limits inductively in $s$.
  \\

  \noindent \textit{Base case: $\E[\mu_{f \rightarrow v}^0\lambda_f] \rightarrow q_1$ and $\E[(\mu_{f\rightarrow v}^0)^2] \rightarrow q_1$.}

  Note $\mu_{f \rightarrow v}^0 = \E[\lambda_f | u_f]$. Thus $\E[\mu_{f \rightarrow v}^0\lambda_f] = \E[\E[\lambda_f|u_f]^2] = V_{\Lambda,U}(0) = q_1$ exactly in finite samples, so also asymptotically.
  The expectation $\E[(\mu_{f\rightarrow v}^0)^2]$ has the same value.
  \\

  \noindent \textit{Inductive step 1: If $\E[\mu_{f\rightarrow v}^s\lambda_f] \rightarrow  q_{s+1}$ and $\E[(\mu_{f\rightarrow v}^s)^2] \rightarrow  q_{s+1}$, then $\alpha_{v \rightarrow f}^{s+1} \stackrel{\mathrm{p}}\rightarrow q_{s+1}$ and $(\tau_{v \rightarrow f}^{s+1})^2 \stackrel{\mathrm{p}}\rightarrow q_{s+1}$.}

  By the inductive hypothesis, $\E[\alpha_{v \rightarrow f}^{s+1}] = (n-1)\E[\mu_{f\rightarrow v}^s\lambda_f]/n \rightarrow q_{s+1} $ and $\E[(\tau_{v\rightarrow f}^{s+1})^2] = (n-1)\E[(\mu_{f \rightarrow v}^s)^2] /n \rightarrow q_{s+1}$.
  Moreover, $\mu_{f' \rightarrow v}^s\lambda_{f'}$ are mutually independent as we vary $f' \in \partial v \setminus f$, and likewise for $\mu_{f'\rightarrow v}^s$. 
  We have $\E[(\mu_{f'\rightarrow v}^s\lambda_{f'})^2] \leq M^4$ and $\E[(\mu_{f'\rightarrow v}^s)^4] \leq M^4$ because the integrands are bounded by $M^4$.
  By the weak law of large numbers, $\alpha_{v \rightarrow f}^{s+1} \stackrel{\mathrm{p}}\rightarrow q_{s+1}$ and $(\tau_{v\rightarrow f}^{s+1})^2 \stackrel{\mathrm{p}}\rightarrow q_{s+1}$.
  \\

  \noindent \textit{Inductive step 2: If $\alpha_{v \rightarrow f}^{s+1} \stackrel{\mathrm{p}}\rightarrow q_{s+1}$ and $(\tau_{v \rightarrow f}^{s+1})^2 \stackrel{\mathrm{p}}\rightarrow q_{s+1}$, then $(\theta_v , v_v, a_{v \rightarrow f}^{s+1},b_{v\rightarrow f}^{s+1}) \stackrel{\mathrm{d}}\rightarrow (\Theta,V,q_{s+1}\Theta + q_{s+1}^{1/2} G, q_{s+1})$.}

  We may express
  $$
  \log m_{v \rightarrow f}^{s+1}(\vartheta) = \mathsf{const} +  \sum_{f' \in \partial v \setminus f} \log \E_{\Lambda_{f'}}\left[\exp\left(-\frac1{2n}\Lambda_{f'}^2 \vartheta^2 + x_{f'v} \Lambda_{f'} \vartheta\right)\right],
  $$
  where $\Lambda_{f'}$ has density $m_{f' \rightarrow v}^s$ with respect to $\mu_{\Lambda|U}(u_{f'},\cdot)$.
  We compute
  \begin{align*}
      \frac{\mathrm{d}\phantom{b}}{\mathrm{d}\vartheta} \E_{\Lambda_{f'}}\left[\exp\left(-\frac1{2n}\Lambda_{f'}^2 \vartheta^2 + x_{f'v} \Lambda_{f'} \vartheta\right)\right]\Big|_{\vartheta=0} &= \E_{\Lambda_{f'}}\left[x_{f'v} \Lambda_{f'}\right] = x_{f'v} \mu_{f'\rightarrow v}^s,\\
      \frac{\mathrm{d}^2\phantom{b}}{\mathrm{d}\vartheta^2}  \E_{\Lambda_{f'}}\left[\exp\left(-\frac1{2n}\Lambda_{f'}^2 \vartheta^2 + x_{f'v} \Lambda_{f'} \vartheta\right)\right]\Big|_{\vartheta=0} &= \E_{\Lambda_{f'}}\left[x_{f'v}^2 \Lambda_{f'}^2 - \frac1n \Lambda_{f'}^2\right] = \left(x_{f'v}^2 - \frac1n\right)s_{f'\rightarrow v}^s.
      %
      %
  \end{align*}
  Then
  \begin{align*}
      a_{v \rightarrow f}^{s+1} = \sum_{f' \in \partial v \setminus f} x_{f'v}\mu_{f'\rightarrow v}^s \;\; \text{and} \;\; b_{v \rightarrow f}^{s+1} = \sum_{f' \in \partial v \setminus f'} \left(x_{f'v}^2 (\mu_{f' \rightarrow v}^s)^2-\left(x_{f'v}^2 - \frac1n\right)s_{f'\rightarrow v}^s\right).
  \end{align*}
  We compute
  \begin{align*}
      a_{v \rightarrow f}^{s+1} = \left(\frac1n\sum_{f' \in \partial v \setminus f} \mu_{f'\rightarrow v}^s \lambda_{f'} \right)\theta_v + \sum_{f' \in \partial v \setminus f} z_{f'v} \mu_{f' \rightarrow v}^s.
  \end{align*}
  Because $(z_{f'v})_{f' \in \partial v \setminus f}$ are independent of $\mu_{f' \rightarrow v}^2$ and are mutually independent from each other, conditional on $\calT_{v \rightarrow f}^1$ the quantity $\sum_{f' \in \partial v \setminus f} z_{f'v} \mu_{f' \rightarrow v}^s$ is distributed $\normal(0,(\tau_{v \rightarrow f}^{s+1})^2)$.
  By the inductive hypothesis, $(\tau_{f \rightarrow v}^{s+1})^2 \stackrel{\mathrm{p}}\rightarrow q_{s+1}$, so that $\sum_{f' \in \partial v \setminus f} z_{f'v} \mu_{f' \rightarrow v}^s \stackrel{\mathrm{d}}\rightarrow \normal(0,q_{s+1})$.
  Further, $z_{f'v}$ and $\mu_{f' \rightarrow v}^s$ are independent of $\theta_v$, and by the inductive hypothesis, the coefficient of $\theta_v$ converges in probability to $q_{s+1}$.
  By the Continuous Mapping Theorem \cite[Theorem 2.3]{vaart_1998}, we conclude that $(\theta_v,v_v,a_{v \rightarrow f}^{s+1}) \stackrel{\mathrm{d}}\rightarrow (\Theta,V,q_{s+1}\Theta + q_{s+1}^{1/2} G)$ where $G \sim \normal(0,1)$ independent of $\Theta$, as desired.

  Now we show that $b_{v \rightarrow f}^{s+1} \stackrel{\mathrm{d}}\rightarrow q_{s+1}$.
  We expand $b_{v\rightarrow f}^{s+1} = A - B$ where $A = \sum_{f' \in \partial v \setminus f} x_{f'v}^2(\mu_{f' \rightarrow v}^s)^2$ and $B = \sum_{f' \in \partial v \setminus f} (x_{f'v}^2 - 1/n)s_{f' \rightarrow v}^s$.
  We have
  \begin{equation*}
      A = \frac1{n^2} \sum_{v' \in \partial f \setminus v} \lambda_{f'}^2 \theta_v^2 (\mu_{f' \rightarrow v}^s)^2  + \frac2n \sum_{f' \in \partial v \setminus f} \lambda_{f'} \theta_v z_{f'v} (\mu_{f'\rightarrow v}^s)^2 + \sum_{v' \in \partial f \setminus v} z_{f'v}^2 (\mu_{f'\rightarrow v}^s)^2.
  \end{equation*}
  Observe $\E[\lambda_{f'}^2 \theta_v^2(\mu_{f'\rightarrow v}^s)^2] \leq M^6$, so that the expectation of the first term is bounded by $M^6(p-1)/n^2 \rightarrow 0$. 
  Thus, the first term converges to 0 in probability.
  Because $z_{f'v}$ is independent of $\mu_{f'\rightarrow v}^s$,
  $\E[|\lambda_{f'} \theta_v z_{f'v}(\mu_{f'\rightarrow v}^s)^2|] \leq M^4 \sqrt{2/(\pi n)}$, so that the absolute value of the expectation of the second term is bounded by $2M^4 \sqrt{2/(\pi n)} \rightarrow 0$.
  Thus, the second term converges to 0 in probability.
  Because $\mu_{f' \rightarrow v}^s$ is independent of $z_{f'v}$, the expectation of the last term is $(n-1)\E[(\mu_{f'\rightarrow v})^2]/n \rightarrow q_{s+1}$ (we have used here the assumption of inductive step 1).
  The terms $(z_{f'v}^2 (\mu_{f'\rightarrow v}^s)^2)_{f' \in \partial v \setminus f}$ are mutually independent and $\E[z_{f'v}^4 (\mu_{f'\rightarrow v}^s)^4] \leq 3M^4/n^2$, so that by the weak law of large numbers we have that the last term converges to $q_{s+1}$ in probability. 
  Thus, $A \stackrel{\mathrm{p}}\rightarrow q_{s+1}$.

  We have
  \begin{equation*}
      B = \frac1{n^2} \sum_{v' \in \partial f \setminus v} \lambda_f^2 \theta_{v'}^2 s_{v' \rightarrow f}^s + \frac2n \sum_{v' \in \partial f \setminus v} \lambda_f \theta_{v'} s_{v'\rightarrow f}^s  + \sum_{v' \in \partial f\setminus v}(z_{f'v}^2 - 1/n) s_{v'\rightarrow f}^s.
  \end{equation*}
  As in the analysis of the first two terms of $A$, we may use that $s_{v' \rightarrow f}^s \leq M^2$ to argue that the first two terms of $B$ converge to 0 in probability.
  Further, because $z_{f'v}$ is independent of $s_{v' \rightarrow f}^s$, the expectation of the last term is 0.
  Further, $\E[(z_{f'v}^2 - 1/n)^2(s_{v' \rightarrow f}^s)^2] \leq 2\E[(z_{f'v}^4 + 1/n^2)]\E[(s_{v' \rightarrow f}^s)^2] \leq  8M^4/n^2$, so that by the weak law of large numbers, the final term converges to 0 in probability. Thus, $B \stackrel{\mathrm{p}}\rightarrow 0$.
  Because, as we have shown, $A \stackrel{\mathrm{p}}\rightarrow q_{s+1}$, we conclude $b_{v\rightarrow f}^{s+1} \stackrel{\mathrm{p}}\rightarrow q_{s+1}$.

  Combining with $(\theta_v,v_v,a_{v \rightarrow f}^{s+1}) \stackrel{\mathrm{d}}\rightarrow (\Theta,V,q_{s+1}\Theta + q_{s+1}^{1/2} G)$ and applying the Continuous Mapping Theorem \cite[Theorem 2.3]{vaart_1998}, we have $(\theta_v ,a_{v \rightarrow f}^{s+1},b_{v\rightarrow f}^{s+1}) \stackrel{\mathrm{d}}\rightarrow (\Theta,q_{s+1}\Theta + q_{s+1}^{1/2} G, q_{s+1})$.
  \\

  \noindent \textit{Inductive step 3: If $(\theta_v ,v_v,a_{v \rightarrow f}^s,b_{v\rightarrow f}^s) \stackrel{\mathrm{d}}\rightarrow (\Theta,V,q_s\Theta + q_s^{1/2} G_1, q_s)$, then $\E[\mu_{v\rightarrow f}^s\theta_v] \rightarrow  \delta\hq_s$ and $\E[(\mu_{v\rightarrow f}^s)^2] \rightarrow  \delta\hq_s$.}

  We will require the following lemma, whose proof is deferred to section \ref{sec:technical-tools-lr-mat}.
  \begin{lemma}\label{lem:LAN-expansion-lr-mat}
      For any fixed $s$, we have $\vartheta,\ell \in [-M,M]$
      \begin{gather*}
          \log \frac{m_{v\rightarrow f}^s(\vartheta)}{m_{v\rightarrow f}^s(0)} = \vartheta a_{v \rightarrow f}^s - \frac12 \vartheta^2 b_{v \rightarrow f}^s + O_p(n^{-1/2}),\\
          \log \frac{m_{f\rightarrow v}^s(\ell)}{m_{f\rightarrow v}^s(0)} = \ell a_{f \rightarrow v}^s - \frac12 \ell^2 b_{f \rightarrow v}^s + O_p(n^{-1/2}),
      \end{gather*}
      where $O_p(n^{-1/2})$ has no $\vartheta$ (or $\ell$) dependence.
  \end{lemma}
  Define 
  \begin{equation*}
      \epsilon_{f \rightarrow v}^s = \sup_{\vartheta \in [-M,M]}\left|\log \frac{m_{v\rightarrow f}^s(\vartheta)}{m_{v\rightarrow f}^s(0)} -\left( \vartheta a_{v \rightarrow f}^s - \frac12 \vartheta^2 b_{v \rightarrow f}^s\right)\right|.
  \end{equation*}
  By Lemma \ref{lem:LAN-expansion-lr-mat}, we have $\epsilon_{v \rightarrow f}^s = o_p(1)$.
  Moreover, using the same argument as in inductive step 4 of the proof of Theorem \ref{lem:hd-reg-posterior}, we have that
  \begin{align*}
      e^{-2\epsilon_{v \rightarrow f}^s} \eta_{\Theta,V}( a_{v \rightarrow f}^s(b_{v \rightarrow f}^s)^{-1/2} & , v_v ; b_{v \rightarrow f}^s) \leq \mu_{v \rightarrow f}^s \\
      &\leq e^{2\epsilon_{v \rightarrow f}^s} \eta_{\Theta,V}( a_{v \rightarrow f}^s(b_{v \rightarrow f}^s)^{1/2} , v_v; b_{v\rightarrow f}^s ),
  \end{align*}
  where $\eta_{\Theta,V}(y,v;q) = \E_{\Theta,V,G}[\Theta|q^{1/2}\Theta + \tau G = y;V=v]$.
  Because $\eta_{\Theta,V}$ takes values in the bounded interval $[-M,M]$ and $\epsilon_{v \rightarrow f}^s = o_p(1)$ by Lemma \ref{lem:LAN-expansion-lr-mat},
  we conclude that
  \begin{equation*}
    \mu_{v \rightarrow f}^s = \eta_{\Theta,V}(a_{v\rightarrow f}^s/b_{v\rightarrow f}^s,v_v;b_{v\rightarrow f}^s) + o_p(1).
  \end{equation*}
  For a fixed $v_v$, the Bayes estimator in the observation and coefficient $q$.
  Thus, by the inductive hypothesis and the fact that $v_v \sim \mu_V$ for all $n$, we have that $\E[\Theta\eta_{\Theta,V}( a_{v \rightarrow f}^s(b_{v \rightarrow f}^s)^{1/2} , v_v; b_{v\rightarrow f}^s )]$ has limit $\E[\Theta \eta_{\Theta,V}(q_s^{1/2} \Theta + G,V;q_s)] = \delta \hat q_s$ and $\E[\eta_{\Theta,V}(q_s^{1/2} \Theta + G,V;q_s)^2]$ has limit $\E_{\Theta,V,G}[\eta_{\Theta,V}(q_s^{1/2}\Theta + G,V;q_s)^2] = \delta \hq_s$.
  Because $|\theta_v|,|\mu_{v \rightarrow f}^s|,|\eta_{\Theta,V}(a_{v\rightarrow f}^s/b_{v\rightarrow f}^s,v_v;b_{v\rightarrow f}^s)| \leq M$, by bounded convergence, we conclude $\E[\mu_{v \rightarrow f}^s\theta_v] \rightarrow \delta \hq_s$ and $\E[(\mu_{f \rightarrow v}^s)^2] \rightarrow \delta \hq_s$. \\

  The remaining inductive steps are completely analagous to those already shown. 
  We list them here for completeness.

  \noindent \textit{Inductive step 4: If $\E[\mu_{v\rightarrow f}^s\theta_v] \rightarrow  \delta\hq_s$ and $\E[(\mu_{v\rightarrow f}^s)^2] \rightarrow  \delta\hq_s$, then $\alpha_{f \rightarrow v}^s \stackrel{\mathrm{p}}\rightarrow \hq_s$ and $(\hat{\tau}_{f \rightarrow v}^{s+1})^2 \stackrel{\mathrm{p}}\rightarrow \hq_s$.}

  \noindent \textit{Inductive step 5: If $\alpha_{f \rightarrow v}^s \stackrel{\mathrm{p}}\rightarrow \hq_s$ and $(\hat{\tau}_{f \rightarrow v}^s)^2 \stackrel{\mathrm{p}}\rightarrow \hq_s$, then $(\lambda_f , u_f , a_{f \rightarrow v}^s,b_{f\rightarrow v}^s) \stackrel{\mathrm{d}}\rightarrow (\Lambda,U,\hq_s\Lambda + \hq_s^{1/2} G, \hq_s)$.}

   \noindent \textit{Inductive step 6: If $(\lambda_f , u_f , a_{f \rightarrow v}^s,b_{f\rightarrow v}^s) \stackrel{\mathrm{d}}\rightarrow (\Lambda,U,\hq_s\Lambda + \hq_s^{1/2} G, \hq_s)$, then $\E[\mu_{f\rightarrow v}^s\lambda_f] \rightarrow  q_{s+1}$ and $\E[(\mu_{f\rightarrow v}^s)^2] \rightarrow  q_{s+1}$.}

  The induction is complete, and we conclude \eqref{eq:bp-lr-mat-scalar-lim}.

  To complete the proof of Lemma \ref{lem:lr-mat-posterior}, first observe that we may express $\log \frac{p_v(\vartheta|\calT_{v,2t-1})}{p_v(0|\calT_{v,2t-1})}$ as $\log \frac{m_{v\rightarrow f}^t(\vartheta)}{m_{v\rightarrow f}^t(\vartheta)} + \log \E_{\Lambda_f}[\exp( \vartheta x_{fv} \Lambda_f  - \vartheta^2 \Lambda_f^2/(2n) )]$.
  Note that 
  \begin{equation*}
  \left|\log \E_{\Lambda_f}[\exp( \vartheta x_{fv} \Lambda_f  - \vartheta^2 \Lambda_f^2/(2n) )]\right| \leq M^2 |x_{fv}| + M^4/2n = o_p(1).
  \end{equation*}
  By Lemma \ref{lem:LAN-expansion-lr-mat}, we have that, up to a constant, $\log \frac{m_{v\rightarrow f}^t(\vartheta)}{m_{v\rightarrow f}^t(\vartheta)} = -\frac12 ( ( a_{v\rightarrow f}^t(b_{ v \rightarrow f}^t)^{-1/2} - b_{ v \rightarrow f}^t)^{1/2} \vartheta )^2 + o_p(1)$.
  The lemma follows from \eqref{eq:bp-lr-mat-scalar-lim} and Slutsky's theorem.
\end{proof}

Lemma \ref{lem:local-info-theory-lb} in the low-rank matrix estimation model follows from Lemma \ref{lem:lr-mat-posterior} by exactly the same argument that derived Lemma \ref{lem:local-info-theory-lb} in the high-dimensional regression model from Lemma \ref{lem:hd-reg-posterior}.

\subsubsection{Technical tools}\label{sec:technical-tools-lr-mat}

\begin{proof}[Lemma \ref{lem:LAN-expansion-lr-mat}]
  Fix any $\vartheta \in [-M,M]$.
    By Taylor's theorem, there exist $\vartheta_{f'} \in [-M,M]$ (in fact, between $0$ and $\vartheta$) such that
    \begin{align*}
        \log \frac{m_{v \rightarrow f}^s(\vartheta)}{m_{v \rightarrow f}^s(0)} &= \sum_{f' \in \partial v \setminus f} \log \frac{\E_{\Lambda_{f'}}[\exp(-n(x_{f'v} - \Lambda_{f'}\vartheta/n)^2/2)]}{\E_{\Lambda_{f'}}[\exp(-nx_{f'v}^2/2)]} \\ 
        &= \vartheta a_{v \rightarrow f}^{s+1} - \frac12 \vartheta^2 b_{v \rightarrow f}^{s+1} + \frac16\vartheta^3 \sum_{f' \in \partial v \setminus f} \frac{\mathrm{d}^3}{\mathrm{d}\vartheta^3} \log \E_{\Lambda_{f'}}[\exp(-n(x_{f'v} - \Lambda_{f'}\vartheta/n)^2/2)] \Big|_{\vartheta=\vartheta_{f'}},
    \end{align*}
    where it is understood that $\Lambda_{f'} \sim \mu_{\Lambda|U}(u_{f'},\cdot)$.
    Denote $\psi(\vartheta,\ell,x) = -n(x_{f'v} - \ell \vartheta/n)^2/2$.
    By the same argument that allowed us to derive \eqref{eq:score-bound} from \textsf{R4} in the proof of Lemma \ref{lem:local-info-theory-lb}(a), we conclude
    \begin{align*}
      &\frac{\mathrm{d}^3}{\mathrm{d}\vartheta^3} \log \E_{\Lambda}[\exp(\psi(\vartheta,\Lambda,x))] \Big|_{\vartheta=\vartheta_{f'}}\\
      &\qquad\qquad \leq C \sup_{\ell,\vartheta\in[-M,M]} \max\{|\partial_\vartheta \psi(\vartheta,\ell,x)|^3,|\partial_\vartheta \psi(\vartheta,\ell,x)\partial_\vartheta^2  \psi(\vartheta,\ell,x)|, |\partial_\vartheta^3\psi(\vartheta,\ell,x)| \} \\
      &\qquad\qquad\leq C \max\left\{M^3|M^2/n + x_{f'v}|^3, (M^2/n) M|M^2/n + x_{f'v}|,0 \right\},
    \end{align*}
    where $C$ is a universal constant.
    The expectaton of the right-hand side is $O(n^{-3/2})$, whence we get
    \begin{equation*}
      \frac16\vartheta^3 \sum_{f' \in \partial v \setminus f} \frac{\mathrm{d}^3}{\mathrm{d}\vartheta^3} \log \E_{\Lambda_{f'}}[\exp(-n(x_{f'v} - \Lambda_{f'}\vartheta/n)^2/2)] \Big|_{\vartheta=\vartheta_{f'}} = O_p(n^{-1/2}),
    \end{equation*}
    where because $\vartheta \in [-M,M]$, we may take $O_p(n^{-1/2})$ to have no $\vartheta$-dependence.

    The expansion of $\log \frac{m_{f\rightarrow v}^s(\ell)}{m_{f\rightarrow v}^s(0)}$ is proved similarly.
\end{proof}

\section{Weakening the assumptions}
\label{app:strong-to-weak-ass}

Section \ref{sec:proof-of-main-results} and the preceding appendices establish under the assumptions \textsf{A1}, \textsf{A2} and either \textsf{R3}, \textsf{R4} or \textsf{M2} all claims in Theorems \ref{thm:hd-reg-lower-bound} and \ref{thm:lr-mat-lower-bound} except that the lower bound may be achieved.
In this section we show that if these claims hold under assumptions \textsf{A1}, \textsf{A2}, \textsf{R3}, \textsf{R4}, then they also hold under assumptions \textsf{A1}, \textsf{A2}, \textsf{R1}, \textsf{R2} in the high-dimensional regression model; 
and similarly for the low-rank matrix estimation model.
In the next section we prove we can achieve the lower bounds under the weaker assumptions \textsf{A1}, \textsf{A2} and either \textsf{R1}, \textsf{R2} or \textsf{M1}.

\subsection{From strong to weak assumptions in the high-dimensional regression model}

To prove the reduction from the stronger assumptions in the high-dimensional regression model, 
we need the following lemma, whose proof is given at the end of this section.

\begin{lemma}\label{lem:mmse-continuity}
    Consider on a single probability space random variables $A,B,(B_n)_{n \geq 1}$, and $Z \sim \normal(0,1)$ independent of the $A$'s and $B$'s, all with finite second moment.
    Assume $\E[(B - B_n)^2]  \rightarrow 0$.
    Let $Y = B + \tau Z$ and $Y_n = B_n + \tau Z$ for $\tau > 0$.
    Then
    \begin{equation*}
        \E[\E[A|Y_n]^2] \rightarrow \E[\E[A|Y]^2]\,.
    \end{equation*}
\end{lemma}
We now establish the reduction.

Consider $\mu_{W,U}$, $\mu_{\Theta,V}$, and $h$ satisfying \textsf{R1} and \textsf{R2}.
For any $\epsilon > 0$, we construct $\mu_{\tbW,\tU}$, $\mu_{\tTheta,\tV}$, and $\th$ satisfying \textsf{R3} and \textsf{R4} for $k = 3$ as well as data $\bX \in \reals^{n\times p}$, 
$\btheta,\tbtheta,\bv,\tbv \in \reals^p$, 
and $\by,\tby,\bw,\bu,\tbu \in \reals^n$ 
and  $\tbw \in \reals^{n\times 3}$ such that the following all hold.

\begin{enumerate}

\item $(\bX,\btheta,\bv,\bu,\bw,\by)$ and $(\bX,\tbtheta,\tbv,\tbu,\tbw,\tby)$ are generated according to their respective regression models: namely,
 $(\theta_j,v_j) \stackrel{\mathrm{iid}}\sim \mu_{\Theta,V}$ and $(w_i,u_i) \stackrel{\mathrm{iid}}\sim \mu_{W,U}$ independent; 
$(\ttheta_j,\tv_j) \stackrel{\mathrm{iid}}\sim \mu_{\tTheta,\tV}$ 
and $(\tbw_i,\tu_i)\stackrel{\mathrm{iid}}\sim \mu_{\tbW,\tU}$ independent;
$x_{ij} \stackrel{\mathrm{iid}}\sim \normal(0,1/n)$ independent of everything else;
and $\by = h(\bX\btheta,\bw)$ and $\tby = \tilde h( \bX \tbtheta,\tbv)$.
Here $\tbw_i^\sT$ is the $i^\text{th}$ row of $\tbw$.
We emphasize that the data from the two models are not independent.

\item We have
\begin{gather}\label{eq:hd-reg-models-close}
  \P\left(\frac1n\|\by - \tby \|^2 > \epsilon \right) \rightarrow 0, \;
  \P\left(\frac1p\|\bv - \tilde \bv \|^2 > \epsilon \right) \rightarrow 0, \; \P\left(\frac1n\|\bu - \tilde \bu \|^2 > \epsilon \right) \rightarrow 0\,.
\end{gather}
Note that because in any GFOM the functions $F_t^{(1)},F_t^{(2)},G_t^{(1)},G_t^{(2)},G_*$ are Lipschitz and $\|\bX\|_{\mathsf{op}} \stackrel{\mathrm{p}}\rightarrow C_\delta < \infty$ as $n,p\rightarrow \infty,n/p \rightarrow 0$ \cite[Theorem 5.31]{Vershynin2012IntroductionMatrices}, the previous display and the iteration \eqref{gfom} imply
\begin{equation}\label{eq:gfom-close}
  \P\left(\frac1p\|\hat \btheta^t - \tilde{\hat \btheta}^t\|^2 > c(\epsilon,t)\right) \rightarrow 0\,,
\end{equation}
for some $c(\epsilon,t) < \infty$ which goes to 0 as $\epsilon \rightarrow 0$ for fixed $t$.

\item We have
\begin{gather}
    |\mmse_{\Theta,V}(\tau_s^2) - \mmse_{\tilde \Theta, \tilde V}(\tau_s)^2| < \epsilon,\label{eq:mmse-close}\\
    \left|\E\left[\E[G_1|h(G,W) + \epsilon^{1/2}Z,G_0]^2\right] - \E\left[\E[G_1|\th(G,\tbW),G_0]^2\right]\right| < \ttau_s^2\epsilon\,,\label{eq:Finfo-close}
\end{gather}
for all $s \leq t$ where $G_0,G_1,Z \stackrel{\mathrm{iid}}\sim \normal(0,1)$, $W \sim \mu_W$, and $\tbW \sim \mu_{\tbW}$ independent,
and $G = \sigma_s G_0 + \ttau_s G_1$.

\end{enumerate}

We now describe the construction described and prove it has the desired properties.
Let $\mu_A$ be a smoothed Laplace distribution with mean zero and variance 1;
namely, $\mu_A$ has a $C_\infty$ positive density $p_A(\cdot)$ with respect to Lebesgue measure which satisfies $\partial_a \log p_A(a) = c\cdot \mathsf{sgn}(a)$ when $|x| > 1$ for some positive constant $c$.
This implies that $|\partial_a^k \log p_A(a)|\leq q_k$ for all $k$ and some constants $q_k$, and that $\mu_A$ has moments of all orders.

First we construct $\th$ and $\tbW$.
For a $\xi > 0$ to be chosen, let $\hat h$ be a Lipschitz function such that $\E[(\hat h(G,W) - h(G,W))^2] < \xi$ for $(G,W)$ as above, which is permitted by assumption \textsf{R2}.
Let $L > 0$ be a Lipschitz constant for $\hat h$.
Choose $M > 0$ such that $\E[W^2\indic{|W| > M}] < \xi/L^2$.
Define $\bar W = W\indic{|W| \leq  M}$.
Note that $\E[(h(G,W) - \hat h(G + \xi^{1/2} A, \bar W))^2] \leq 2\E[(h(G,W) - \hat h(G,W))^2] + 2\E[(\hat h(G,W) - \hat h(G + \xi^{1/2} A , \bar W))^2] < 4\xi$.
By Lemma \ref{lem:mmse-continuity}, 
we may pick $0 < \xi < \min\{\epsilon/4,\epsilon/L^2\}$ sufficiently small that
\begin{equation*}
  \left|\E\left[\E[G_1|h(G,W) + \epsilon^{1/2}Z,G_0]^2\right] - \E\left[\E[G_1|\hat h(G + \xi^{1/2}A,\bar W) + \epsilon^{1/2}Z,G_0]^2\right]\right| < \ttau_s^2 \epsilon\,.
\end{equation*} 
In fact, because $t$ is finite, we may choose $\xi > 0$ small enough that this holds for all $s \leq t$.
Define $\tbW = (\bar W,A,Z)$ and $\th(x,\tbw) = \hat h( x + \xi^{1/2} a , \bar w ) + \epsilon^{1/2}z $ where $\tbw = (\bar w,a,z)$.
Then $\th$ is Lipschitz, Eq.~\eqref{eq:Finfo-close} holds for all $s \leq t$, and $\E[(h(G,W) - \th(G,\tbW))^2] < \epsilon$ (the last because $\xi < \epsilon/4$).

Now choose $K>0$ large enough that 
\begin{gather}\label{eq:truncation}
  \E[\Theta^2 \indic{|\Theta| > K}] < \delta \epsilon/L^2, \;\; \E[U^2\indic{|U| > K}] < \epsilon/2, \;\; \E[V^2\indic{|V| > K}] < \epsilon/2\,.
\end{gather}
Define $\tTheta = \bar \Theta = \Theta \indic{|\Theta| \leq K}$, $\tV = \bar V =  V \indic{|V| \leq K}$, $\tU = \bar U = U \indic{|U| \leq K}$,
and let $\mu_{\tTheta,\tV},\mu_{\tbW,\tU}$ be the corresponding distributions;
namely, $\mu_{\tTheta,\tV}$ is the distribution of $(\Theta \indic{|\Theta| \leq K},V \indic{|V|\leq K})$ when $(\Theta,V) \sim \mu_{\Theta,V}$, and $\mu_{\tbW,\tU}$ is the distribution of $(W\indic{|W|\leq M},A,Z)$ when $(W,U) \sim \mu_{W,U}$ and $(A,Z) \sim \mu_A \otimes \normal(0,1)$ independent.
Because the Bayes risk converges as $K \rightarrow \infty$ to the Bayes risk with respect to the untruncated prior, we may choose $K$ large enough that also \eqref{eq:mmse-close} holds for these truncated distributions.

The distributions $\mu_{\tTheta,\tV},\mu_{\tbW,\tU}$ satisfy assumption \textsf{R3}.
We now show that $\th$ and $\tbW$ constructed in this way satisfy assumption \textsf{R4}.
The function $\th$ is Lipschitz because $\hat h$ is Lipschitz.
The random variable $\tY := \hat h(x + \xi^{1/2} A, \bar W) + \epsilon^{1/2} Z$ has density with respect to Lebesgue measure given by
\begin{equation*}
  p(y|x) = \int \int p_{\xi^{1/2}A}\left(s-x\right) p_{\normal(0,\epsilon)}(y - \hat h(s,\bar w)) \mu_{\bar W}(\de \bar w) \de s,
\end{equation*} 
where $p_{\normal(0,\epsilon)}$ is the density of $\normal(0,\epsilon)$ and $p_{\xi^{1/2}A}\left(s-x\right)$ the density of $\xi^{1/2}A$ with respect to Lebesgue measure.
We have $p(y|x) \leq \sup_y p_{\normal(0,\epsilon)}(y) = 1/\sqrt{2\pi \epsilon}$, so is bounded, as desired.
Moreover 
\begin{equation*}
  \left|\frac{\int \int \partial_x p_{\xi^{1/2}A}\left(s-x\right) p_{\normal(0,\epsilon)}(y - \hat h(s,\bar w)) \mu_{\bar W}(\de \bar w) \de s}{p(y|x)}\right| \leq \sup_{s} \left|\frac{\dot p_{\xi^{1/2}A}(s)}{p_{\xi^{1/2}A}(s)}\right|.
\end{equation*}
Because $A$ has a smoothed Laplace distribution, the right-hand side is finite.
Thus, by bounded convergence, we may exchange differentiation and integration and the preceding display is equal to $\partial_x \log p(y|x)$.
We conclude that $|\partial_x \log p(y|x)|$ is bounded.
The boundededness of all higher derivatives holds similarly. 
Thus, \textsf{R4} holds.

We now generate the appropriate joint distribution over $(\bX,\btheta,\bv,\bu,\bw,\by)$ and $(\bX,\tbtheta,\tbv,\tbu,\tbw,\tby)$.
First, generate $(\bX,\btheta,\bv,\bu,\bw,\by)$ from original the high-dimensional regression model.
Then generate $\ba,\bz$ independent and with entries $a_i \stackrel{\mathrm{iid}}\sim \mu_A$ and $z_i \stackrel{\mathrm{iid}}\sim \normal(0,1)$.
Define $\tbtheta,\tbv,\tbu$ by truncating $\btheta,\bv,\bu$ at threshold $K$; define $\tbw$ by truncating $\bw$ at threshold $M$ to form $\bar \bw$ and concatenating to it the vectors $\ba,\bz$ to form a matrix in $\reals^{n\times 3}$;
and define $\tby = \th(\bX\tbtheta,\tbw)$.

All that remains is to show \eqref{eq:hd-reg-models-close} holds for the model generated in this way.
The bounds on $\|\bv-\tbv\|^2$ and $\|\bu - \tbu\|^2$ hold by the weak law of large numbers and \eqref{eq:truncation}.
To control $\|\by - \tby\|$, we bound
\begin{align*}
  &\|\by - \tby\| = \|h(\bX\btheta,\bw) - \th(\bX \tilde \btheta , \tilde \bw) \| \\
  &\qquad  \leq \|h(\bX\btheta,\bw) - \hat h(\bX \btheta,\bw)\| + \|\hat h(\bX\btheta,\bw) - \hat h(\bX \tbtheta , \bw ) \| + \| \hat h (\bX \tbtheta , \bw ) - \th(\bX \tbtheta , \tbw ) \\
  &\qquad \leq \|h(\bX\btheta,\bw) - \hat h(\bX \btheta,\bw)\| +L \|\bX(\btheta - \tbtheta)\| + L\xi^{1/2}\|\ba\| + L \|\bw - \bar \bw\| + \epsilon^{1/2} \|\bz\|\,.
\end{align*}
Because $|h(x,w)| \leq C(1 + |x| + |w|)$ by \textsf{R2} and $\hat h$ is Lipschitz, there exist $C > 0$ such that 
$|h(x,w) - \hat h(x,w)| \leq C(1 + |x| + |w|)$.
Then, $\E[(h(\tau Z,w) - \hat h(\tau Z,w))^2] = \int (h(x,w) - \hat h(x,w))^2 \frac1{\sqrt{2\pi}\tau} e^{-\frac1{2\tau^2}x^2} \de x < C(1 + \tau^2 + w^2)$ and is continuous in $\tau^2$ for $\tau > 0$ by dominated convergence convergence, and is uniformly continuous for $\tau$ bounded away from 0 and infinity and $w_i$ restricted to a compact set.
Because $\bx_i^\sT \btheta | \btheta \sim \normal(0,\|\btheta\|^2/n)$ and $\|\btheta\|^2/n \stackrel{\mathrm{p}}\rightarrow \tau_\Theta^2/\delta$, we have that 
$$
  \E[(h(\bx_i^\sT\btheta,w_i) - \hat h( \bx_i^\sT \btheta , w_i))^2 |\btheta,w_i ] = \E[(h( \tau_\Theta \bx_i^\sT\btheta/\|\btheta\|,w_i) - \hat h( \tau_\Theta \bx_i^\sT \btheta /\|\btheta\|, w_i))^2 |\btheta,w_i ] + o_p(1)\,.
$$
The right-hand side is a  constant equal to $\E[(h( G, W) - \hat h( G, W))^2]$ and the left-hand side is uniformly integrable.
Thus,
\begin{equation*}
  \limsup_{n \rightarrow \infty} \E[(h(\bx_i^\sT\btheta,w_i) - \hat h( \bx_i^\sT \btheta , w_i))^2] \leq \E[(h(G,w_i) - \hat h( G, w_i))^2] < \xi\,.
\end{equation*}
Markov's inequality proves the the first convergence in \eqref{eq:hd-reg-models-close} because $\xi < \epsilon$.
Further, by the weak law of large numbers
\begin{equation*}
  \frac{L^2}{n} \|\bX(\btheta - \tilde \btheta)\|^2 \leq \frac{L^2\|\bX\|_{\mathsf{op}}^2}{n}\|\btheta - \tilde \btheta\|^2 \stackrel{\mathrm{p}}\rightarrow L^2C_\delta \delta^{-1} \E[\Theta^2\indic{|\Theta| > M}] < C_\delta\epsilon\,,
\end{equation*}
where $C_\delta$ is the constant satisfying $\|\bX\|_{\mathsf{op}}^2 \stackrel{\mathrm{p}}\rightarrow C_\delta$ \cite[Theorem 5.31]{Vershynin2012IntroductionMatrices}.
Similarly, by the weak law of large numbers
\begin{equation*}
  \frac{L^2\xi}{n}\|\ba\|^2 \stackrel{\mathrm{p}}\rightarrow L^2\xi < \epsilon, \;\; \frac{L^2}{n}\|\bw - \bar \bw\|^2 \stackrel{\mathrm{p}}\rightarrow L^2 \E[W^2 \indic{|W| > M}] < \xi < \epsilon, \;\; \frac\epsilon n\|\bz\|^2 \stackrel{\mathrm{p}}\rightarrow \epsilon\,.
\end{equation*}
We conclude that
\begin{equation*}
  \P\left(\frac1n \|\by - \tby\|^2 > 5(C_\delta +4)\epsilon\right) \rightarrow 0.
\end{equation*}
Becuse $\epsilon$ was arbitrary, we can in fact achieve \eqref{eq:hd-reg-models-close} by considering a smaller $\epsilon$ (without affecting the validity of \eqref{eq:mmse-close}).

This completes the construction. 
To summarize, we have two models: the first satisfying \textsf{R1} and \textsf{R2}, and the second satisfying \textsf{R3} and \textsf{R4}.

With the construction now complete, we explain why it establishes the reduction.
Let $\tau_s^{(\epsilon)},\ttau_s^{(\epsilon)}$ be the state evolution parameters generated by \eqref{bamp-se-hd-reg} with $\mu_{\tbW,\tU}$, $\mu_{\tTheta,\tV}$, and $\th$ in place of $\mu_{W,U},\mu_{\Theta,V}$, and $h$.
First, we claim that Eqs.~\eqref{eq:mmse-close} and \eqref{eq:Finfo-close} imply, by induction, that as $\epsilon \rightarrow 0$, we have 
\begin{equation*}
  \tau_t^{(\epsilon)} \rightarrow \tau_t.
\end{equation*}
Indeed, to show this, we must only establish that $\E\left[\E[G_1|h(G,W) + \epsilon^{1/2}Z,G_0]^2\right]$ converges to $\E[\E[G_1|h(G,W),G_0]^2]$ as $\epsilon \rightarrow 0$.
Without loss of generality, we may assume that on the same probability space there exists a Brownian motion $(B_\epsilon)_{\epsilon > 0}$ independent of everything else.
We see that $\E[G_1|h(G,W) + \epsilon^{1/2}Z,G_0]^2] \stackrel{\mathrm{d}}= \E[G_1|h(G,W) + B_\epsilon,G_0] = \E[G_1|(h(G,W) + B_s)_{s \geq \epsilon},G_0]$.
By L\'evy's upward theorem \cite[Theorem 5.5.7]{Durrett2010Probability:Examples}, we have that $\E[G_1|(h(G,W) + B_s)_{s \geq \epsilon},G_0]$ converges to $\E[G_1|(h(G,W) + B_s)_{s \geq 0},G_0] = \E[G_1|h(G,W),G_0]$ almost surely.
By uniform integrability, we conclude that $\E[\E[G_1|(h(G,W) + B_s)_{s \geq \epsilon},G_0]^2] \rightarrow \E[\E[G_1|h(G,W),G_0]^2]$, as claimed.
Thus, we conclude the previous display. 

We now show that as $\epsilon \rightarrow 0$, we have
\begin{equation*}
  \inf_{\hat \theta (\cdot) } \E[\ell(\tTheta,\hat \theta (\tTheta + \tau_t^{(\epsilon)}G,V))] \rightarrow \inf_{\hat \theta (\cdot) } \E[\ell(\Theta,\hat \theta (\Theta + \tau_tG,V))]\,.
\end{equation*}
Because the truncation level $K$ can be taken to $\infty$ as $\epsilon \rightarrow 0$,
this holds by combining Lemma \ref{lem:properties-of-bayes-risk}(a) and (c), and specifically, Eqs.~\eqref{eq:bayes-is-exp-post} and \eqref{eq:risks-trunc-limit}.

Because the lower bound of Theorem \ref{thm:hd-reg-lower-bound} holds under assumptions \textsf{R3} and \textsf{R4}, which are satisfied by $\mu_{\tbW,\tU}$, $\mu_{\tTheta,\tV}$, and $\th$, 
we conclude that 
\begin{equation*}
  \lim_{n \rightarrow \infty} \frac1p \sum_{j=1}^p \ell(\theta_j,\hat \theta_j^t) \geq \inf_{\hat \theta (\cdot) } \E[\ell(\tTheta,\hat \theta (\tTheta + \tau_t^{(\epsilon)}G,V))].
\end{equation*}
Taking $\epsilon \rightarrow 0$ and applying \eqref{eq:gfom-close}, we conclude that \eqref{eq:hd-reg-lb} holds for $\hat \btheta^t$, as desired.

The reduction in the high-dimensional regression model is complete.

\begin{proof}[Lemma \ref{lem:mmse-continuity}]
  It is enough to prove the result for $\tau = 1$.
  Note 
    \begin{equation*}
        \E[A|Y = y] = \frac{\int a e^{-(y-b)^2} \mu(\de a,\de b) }{ \int e^{-(y-b)^2} \mu(\de a , \de b) },\qquad \E[A|Y_n = y] = \frac{\int a e^{-(y-b)^2} \mu_n(\de a,\de b) }{ \int e^{-(y-b)^2} \mu_n(\de a , \de b) }.
    \end{equation*}
    Because $\mu_n \stackrel{\mathrm{W}}\rightarrow \mu$, we have 
    \begin{equation*}
        \frac{\int a e^{-(y-b)^2} \mu_n(\de a,\de b) }{ \int e^{-(y-b)^2} \mu_n(\de a , \de b) } \rightarrow \frac{\int a e^{-(y-b)^2} \mu(\de a,\de b) }{ \int e^{-(y-b)^2} \mu(\de a , \de b) },
    \end{equation*}
    for all $y$, and moreover, this convergence is uniform on compact sets.
    Moreover, one can check that the stated functions are Lipschitz (with uniform Lipschitz constant) in $y$ on compact sets.
    This implies that $\E[A|Y_n] \rightarrow \E[A|Y]$ almost surely.
    Because the $\E[A|Y_n]^2$ are uniformly integrable, the lemma follows.
\end{proof}

\subsection{From strong to weak assumptions in the low-rank matrix estimation model}

Consider $\mu_{\bLambda,\bU},\mu_{\bTheta,\bV}$ satisfying \textsf{M1}.
Fix $M > 0$.
For $(\bLambda,\bU) \sim \mu_{\bLambda,\bU}$, 
define $\tilde \Lambda$ by setting $\tilde \Lambda_i = \Lambda_i\indic{|\Lambda_i| \leq M}$ for $1 \leq i \leq k$.
Define $\tilde \bU$ similarly, and let $\mu_{\tilde \bLambda,\tilde \bU}$ be the distribution of $(\tilde \bLambda,\tilde \bU)$ so constructed.
Define $\mu_{\tilde \bTheta,\tilde \bV}$ similarly.

Consider $\{(\blambda_i,\bu_i)\}_{i\le n}\stackrel{\mathrm{iid}}\sim\mu_{\bLambda,\bU}$
and $\{(\btheta_j,\bv_j)\}_{j\le p}\stackrel{\mathrm{iid}}\sim\mu_{\bTheta,\bV}$ and $\bZ \in \reals^{n \times p}$ independent with $z_{ij} \stackrel{\mathrm{iid}}\sim \normal(0,1/n)$.
Constructe $\tilde \blambda_i,\tbu_i,\tilde \btheta_j,\tbv_j$ by truncated each coordinate at level $M$ as above. 
Define $\bX,\tilde \bX \in \reals^{n\times p}$ by $x_{ij} = \frac1n \blambda_i^\sT\btheta_j + z_{ij}$ and $\tilde z_{ij} = \frac1n \tilde \blambda_i^\sT \tilde \btheta + z_{ij}$.
As in the previous section, we have for any $\epsilon > 0$ that 
\begin{equation*}
  \P(\|\bX - \tilde \bX\|_{\mathsf{op}} > \epsilon) \rightarrow 0, \;\; \P\left( \frac1p \|\bv - \tilde \bv\|^2 > \epsilon \right) \rightarrow 0, \;\; \P\left(\frac 1p \|\bu - \tilde \bu\|^2 > \epsilon \right) \rightarrow 0.
\end{equation*}
As in the previous section, this implies that the iterates of the GFOMs before and after the truncation become arbitrarily close with high probability at a fixed iterate $t$ as we take $M \rightarrow \infty$.

Further, as $M \rightarrow \infty$ we have $\bV_{\tilde \bTheta,\tilde \bV}(\bQ) \rightarrow \bV_{\bTheta,\bV}(\bQ)$ for all $\bQ$, and likewise for $\tilde \bLambda, \tilde \bU$.
Further, $\bV_{\tilde \bTheta,\tilde \bV}(\bQ)$ is jointly continuous in $\bQ$ and $M$ (where $M$ is implicit in the truncation used to generate $\tilde \bTheta,\tilde \bV$).
Thus, as we take $M \rightarrow \infty$, the state evolution \eqref{eq:SE_Matrix} after the truncation converges to the state evolution with no truncation.

The reduction now occurs exactly as in the previous section.

\section{Achieving the bound}
\label{app:achieving-the-bound}

All that remains to prove Theorems \ref{thm:hd-reg-lower-bound} and \ref{thm:lr-mat-lower-bound} under assumptions \textsf{A1}, \textsf{A2} and either \textsf{R1}, \textsf{R2} or \textsf{M1}, respectively, is to show that the lower bounds in Eqs.~\eqref{eq:hd-reg-lb} and \eqref{eq:lr-mat-lb} can be achieved. 
In both cases, we can achieve the bound up to tolerance $\epsilon$ using a certain AMP algorithm.

\subsection{Achieving the bound in the high-dimensional regression model}

We first derive certain monotonicity properies of the parameters $\tau_s,\sigma_s,\ttau_s$ defined in the state evolution recursion \eqref{bamp-se-hd-reg}.
As we saw in Appendix \ref{app:info-lb-hd-reg}
and in particular, in Lemma \ref{lem:hd-reg-posterior},
the posterior of $\theta_v$ on the computation tree given observations in the local neighborhood $T_{v,2s}$ behaves like that from an observation under Gaussian noise with variance $\tau_s^2$.
This is made precise in Lemma \ref{lem:hd-reg-posterior}.
Moreover, we saw in the same section that a consequence of Lemma \ref{lem:hd-reg-posterior} is that the asymptotic limiting Bayes risk with respect to loss $\ell$ for estimation $\theta_v$ given observations in $\calT_{v,2s}$ is given by the corresponding risk for estimating $\Theta$ given $\Theta + \tau_s G$, $V$ with $(\Theta,V) \sim \mu_{\Theta,V}$ and $G \sim \normal(0,1)$ independent.
In particular, this applies to the minimum mean square error.
On the computation tree, minimum mean square error can only decrease as $s$ grows because as $s$ grows we receive strictly more information.
If $\E[\Var(\Theta | V)] > 0$, then $\mmse_{\Theta,V}(\tau^2)$ is strictly increasing in $\tau$,
so that we conclude that $\tau_s$ is non-increasing in $s$.
Thus, by \eqref{bamp-se-hd-reg}, we have also $\ttau_s$ is non-increasing in $s$ and $\sigma_s$ is non-decreasing in $s$.
In the complementary case that $\E[\Var(\Theta | V)] = 0$, we compute $\sigma_s^2 = \tau_\Theta^2/\delta$ and $\ttau_s^2 = 0$ for all $s \geq 0$, and $\tau_s^2 = 0$ for all $s \geq 1$.
Thus, the same monotoncity results hold in this case.
These monotonicity results will imply the needed structural properties of the state evolution matrices $(T_{s,s'}),(\Sigma_{s,s'})$ used below.

For all $s \leq t$, define 
\begin{equation*}
  \alpha_s = \frac1{\ttau_s} \E[\E[G_1|Y,G_0,U]^2], \;\; T_{s,t} = \E[\E[G_1|Y,G_0,U]^2], \;\; \Sigma_{s,t} = \sigma_t^2,
\end{equation*}
where $Y = h(\sigma_s G_0 + \ttau_s G_1,W)$ and $G_0,G_1 \stackrel{\mathrm{iid}}\sim \normal(0,1)$ and $W \sim \mu_W$ independent.
By the monotoncity properties stated, $(T_{s,t}),(\Sigma_{s,t})$ define positive definite arrays.
Define
\begin{gather*}
  f_t(b^t;y,u) = \E[B^0 - B^t | h(B^0,W) = y ,\, B^t = b^t ,\, U = u ] / \ttau_t,\\
  g_t(a^t;v) = \E[\Theta | V = v ,\, \alpha_t\Theta + Z^t = a^t],
\end{gather*}
where $(\Theta,V) \sim \mu_{\Theta,V})$, $(W,U) \sim \mu_{W,U}$, $(B^0,\ldots,B^t) \sim \normal(\bzero,\bSigma_{[0{:}t]})$, $(Z^1,\ldots,Z^t) \sim \normal(\bzero,\bT_{[1{:}t]})$, all independent.
With these definitions, $(B^t,B^0-B^t) \stackrel{\mathrm{d}}= (\sigma_tG_0,\ttau_t G_1)$ where $G_0,G_1 \stackrel{\mathrm{iid}}\sim \normal(0,1)$.
In particular, $(B^t)$ form a backwards Gaussian random walk.
We thus compute
\begin{align*}
  &\E[(B^0 - B^t) f_t(B^t;h(B^0,W),U)] / \ttau_t^2 = \E[(\E[B^0 - B^t|Y,B^t,U]/\ttau_t)^2] / \ttau_t = \alpha_t, \\
  &\E[f_s(B^s;h(B^0,W),U)f_t(B^t;h(B^0,W),U) ] \\
  &\qquad\qquad = \E[\E[B^0 - B^s|Y,B^s,U]\E[B^0 - B^t|Y,B^t,U]]/\ttau_t^2 \\
  &\qquad\qquad = \E[(B^0 - B^t)^2| Y , B^t, U] / \ttau_t^2 = T_{s,t},\\
  &\frac1\delta \E[\Theta g_t(\alpha_t\Theta + Z^t;V)] = \frac1\delta \E[\E[\Theta| \Theta + Z^t / \alpha_t,V]^2] = \sigma_t^2,\\
  &\frac1\delta \E[g_s(\alpha_s\Theta + Z^s;V)g_t(\alpha_t\Theta + Z^t;V)] = \frac1\delta \E[\E[\Theta|\Theta + Z^t/\alpha_t,V]^2].
\end{align*}
If $f_t,g_t$ are Lipschitz, then, because $h$ is also Lipschitz, Stein's lemma \cite{Stein1981EstimationDistribution} implies that the first line is equivalent to $\E[\partial_{B^0} f_t(B^t;h(B^0,W),U)] = \alpha_t$. (Here, we have used that $B^0 - B^t$ is independent of $B^t$).
Thus, $(\alpha_s),(T_{s,t}),(\Sigma_{s,t})$ are exactly the state evolution parameters determined by \eqref{amp-scalars-hd-reg}, 
and Lemma \ref{lem:gfom-to-amp} implies that AMP with these $(f_s),(g_s)$ achieves the lower bound.

If the $f_t,g_t$ are not Lipschitz, we proceed as follows.
Fix $\epsilon > 0$.
First, pick Lipschitz $\hat f_0$ such that $\E[(\hat f_0(B^0,W) - f_0(B^0,W))^2] < \epsilon$, which is possibly because Lipschitz functions are dense in $L_2$.
Define $\hat \alpha_0$ and $\hat T_{1,1}$ via \eqref{amp-scalars-hd-reg} with $\hat f_0$ in place of $f_0$.
Note that $\lim_{\epsilon \rightarrow 0} \hat \alpha_0 = \alpha_0$ and $\lim_{\epsilon \rightarrow 0} \hat T_{1,1} = T_{1,1}$.
Next, pick Lipschitz $\hat g_0$ such that $\E[(\hat g_0(\hat \alpha_0 \Theta + \hat T_{1,1}^{1/2} G;V) - \E[\Theta | \hat \alpha_0 + \Theta + \hat T_{1,1}^{1/2} G ; V)])^2] < \epsilon$, which is again possibly because Lipschitz functions are dense in $L_2$.
Define $\hat \Sigma_{0,1} = \frac1\delta \E[\Theta \hat g_t(\hat \alpha \Theta + \hat T_{1,1}^{1/2} G;V)]$ and $\hat \Sigma_{1,1} = \frac1\delta \E[ \hat g_t(\hat \alpha \Theta + \hat T_{1,1}^{1/2} G;V)^2 ]$.
Because as $\alpha \rightarrow \alpha_0$ and $\tau \rightarrow T_{0,0}^{1/2}$, we have $\E[\Theta | \alpha \Theta + \tau G ; V)] \stackrel{L_2}\rightarrow \E[\Theta | \alpha_0 \Theta + T_{0,0}^{1/2} G ; V)]$, 
we conclude that as $\epsilon \rightarrow 0$ that $\hat \Sigma_{0,1} \rightarrow \Sigma_{1,1}$ and $\hat \Sigma_{1,1} \rightarrow \Sigma_{1,1}$.
Continuing in this way, we are able to by taking $\epsilon$ sufficiently small construct Lipschitz functions $(\hat f_t),(\hat g_t)$ which track the state evolutoin of the previous paragraph arbitrarily closely up to a fixed time $t^*$.
Thus, we may come arbitrarily close to achieving the lower bound of Theorem \ref{thm:hd-reg-lower-bound}.

\subsection{Achieving the bound in the low-rank matrix estimation model}

Let $\bgamma_t = \hat \bQ_t$ for $t \geq 0$ and $\balpha_t = \bQ_t$, $\bSigma_{t,t} = \hat \bQ_t$, $\bT_{t,t} = \bQ_t$ for $t \geq 1$.
Define
\begin{gather*}
  f_t(\bb^t; \bu) = \E[\bLambda | \bgamma_t \bLambda + \bSigma_{t,t}^{1/2}\bG = \bb^t;\bU],\\
  g_t(\ba^t; \bv) = \E[\bTheta | \balpha_t \bTheta + \bT_{t,t}^{1/2}\bG = \ba^t;\bV].
\end{gather*}
We check that the parameters so defined satisfy the AMP state evolution \eqref{amp-scalars-lr-mat}.
Note that by \eqref{eq:SE_Matrix}, 
\begin{align*}
  \bT_{t+1,t+1} &= \bQ_{t+1} = \E[\E[\bLambda | \hat \bQ_t^{1/2} \bLambda +  \bG ;\bU]\E[\bLambda | \hat \bQ_t^{1/2}  \bLambda + \bG;\bU]^\sT]\\
  &= \E[\E[\bLambda | \hat \bQ_t \bLambda +  \hat \bQ_t^{1/2} \bG ;\bU]\E[\bLambda | \hat \bQ_t  \bLambda + \hat \bQ_t^{1/2} \bG ;\bU]^\sT]\\
  &= \E[\E[\bLambda | \bgamma_t \bLambda +  \bSigma_{t,t}^{1/2} \bG;\bU]\E[\bLambda | \bgamma_t  \bLambda + \bSigma_{t,t}^{1/2} \bG;\bU]^\sT],\\
  \balpha_{t+1} &= \E[\E[\bLambda | \hat \bQ_t^{1/2} \bLambda +  \bG ;\bU]\E[\bLambda | \hat \bQ_t^{1/2}  \bLambda + \bG;\bU]^\sT]\\
  &= \E[\E[\bLambda | \bgamma_t \bLambda +  \bSigma_{t,t}^{1/2}\bG ;\bU]\bLambda^\sT]\\
\end{align*}
where $(\bTheta,\bV) \sim \mu_{\bTheta,\bV}$ and $(\bLambda,\bU) \sim \mu_{\bLambda,\bU}$.
The state evolution equations \eqref{eq:SE_Matrix} for $\bSigma_{t,t}$ and $\bgamma_t$ hold similarly.

If $f_t,g_t$ so defined are Lipschitz, then $(\alpha_s),(\bT_{s,t}),(\bSigma_{s,t})$ are exactly the state evolution parameters determined by \eqref{amp-scalars-hd-reg}, 
and Lemma \ref{lem:gfom-to-amp} implies that AMP with these $(f_s),(g_s)$ achieves the lower bound.
If the $f_t,g_t$ so defined are not Lipschitz, then the same strategy used in the previous section allows us to achieve the lower bound within tolerance $\epsilon > 0$.

\section{Proofs for sparse phase retrieval and sparse PCA}

\subsection{Proof of Lemma \ref{lemma:SPCA}}
  Note that $\|\obtheta_0\|_2$ is tightly concentrated around $\mu^2\eps$. As a consequence, we can replace the side information
    $\obv$ by $\bv = \sqrt{\talpha}\btheta_0+\bg$.
    We apply Theorem \ref{thm:lr-mat-lower-bound} with $r=1$, and loss $\ell_{\lambda}(\theta,\htheta) = (\htheta-\theta_0/\lambda)^2$, where $\lambda\in \reals_{\ge 0}$ will be adjusted below.
  Setting $\bQ_t=q_t$, $\hbQ_t=\hq_t$, we obtain the iteration
  \begin{align}
  q_{t+1} = \frac{\hq_t}{1+\hq_t}\, ,\;\;\;\;
  \hq_{t} = \frac{1}{\delta}\E\big\{\E[\sqrt{\delta}\Theta_0|(\delta q_t)^{1/2}\Theta_0+G;V]^2\big\}\ ,
  \end{align}
  where $\Theta_0\sim \mu_{\theta}$,and $V= \sqrt{\delta\talpha}+G'$, $G'\sim\normal(0,1)$.
  Notice  that the additional factors $\sqrt{\delta}$ are due to the different normalization
  of the vector $\btheta_0$ with respect to the statement in Theorem \ref{thm:lr-mat-lower-bound}. Also note that
  the second moment of the conditional expectation bove is equal to $\E\big\{\E[\sqrt{\delta}\Theta_0|(\delta (q_t+\talpha))^{1/2}\Theta_0+G]^2\big\}$ and a simple calculation yields
  \begin{align}
  \hq_{t+1} = V_{\pm}(q_t+\talpha)\, ,\;\;\;\;
  q_t = \frac{\hq_t}{1+\hq_t}\, ,
  \end{align}
  which is equivalent to Eqs.~\eqref{eq:SE-SPCA-1}, \eqref{eq:SE-SPCA-2}.

  Let $Y = \sqrt{\delta (q_t+\talpha)}\Theta_0+G$, $G\sim\normal(0,1)$. 
  Theorem  \ref{thm:lr-mat-lower-bound} then yields 
  \begin{align}
    \frac{1}{p}\|\hbtheta^t-\btheta_0/\lambda\|_2^2&\ge \inf_{\htheta(\, \cdot\,)}
                                                     \E\big\{\big(\htheta(Y)-\Theta_0/\lambda\big)^2\big\}+o_p(1)\\
                                                   & = \frac{1}{\lambda^2}  \E\big\{\big(\E(\Theta_0|Y)-\Theta_0\big)^2\big\}+o_p(1)\,.
  \end{align}
    In order to prove the upper bound \eqref{eq:StatementPCA}, it is sufficient to consider $\|\hbtheta^t\|^2_2\le p$.
  Then, for any $\lambda\ge 0$, 
  \begin{align}
   \frac{1}{p}\<\hbtheta^t,\btheta_0\>& \le \frac{1}{p}\<\hbtheta^t,\btheta_0\> -\frac{\lambda}{2p} (\|\hbtheta^t\|^2_2-p) \\
                          &=\frac{\lambda}{2} +\frac{1}{2\lambda p}\|\btheta_0\|_2^2- \frac{\lambda}{2p}\|\hbtheta^t-\btheta_0/\lambda\|_2^2\\
    & \le \frac{\lambda}{2}     +\frac{1}{2\lambda}\E\{\Theta_0^2\}-\frac{1}{2\lambda}
      \E\big\{\big(\E(\Theta_0|Y)-\Theta_0\big)^2\big\}+o(1)\\
                           & \le  \frac{\lambda}{2}     +\frac{1}{2\lambda}V_{\pm}(q_t + \talpha) +o(1)\, .   
  \end{align}
  The claim follows by choosing $\lambda =V_{\pm}(q_t + \talpha) ^{1/2}$, and noting that $\|\btheta_0\|^2_2/p\to \mu^2\eps$, almost surely.
  
  \subsection{Proof of Corollary \ref{coro:SPCA}}
  
   Choose $\mu=R/\sqrt{\eps}$, and let $\mu'<\mu$, $\eps'<\eps$, $R'=\mu'\sqrt{\eps'}$.
    Draw the coordinates of $\btheta_0=\obtheta_0\sqrt{p}$ according to the three points distribution with parameters $\mu',\eps'$.
    Then, with probability one, we have  $\obtheta_0\in \cuT(\eps,R)$ for all $n$ large enough.
    Applying Lemma \ref{lemma:SPCA}, we get
      \begin{align}
      \lim_{n\to\infty}\inf_{\obtheta_0\in\cuT(\eps,R)}\E\left\{\frac{\<\obtheta_0,\hbtheta^t\>}{\|\obtheta_0\|_2\|\hbtheta^t\|_2}\right\}\le \sqrt{\frac{V_{\pm}(q'_t+\talpha')}{(\mu')^2\eps'}}\, ,\label{eq:FirstMinimaxSPCA}
      \end{align}
      ahere we used dominated convergence to pass from the limit in probability to limit in expectation, and $q'_t,\talpha'$ are computed with parameters $\mu'$, $\eps'$. By letting $\eps'\to\eps$, $\mu'\to\mu$, and since $\talpha', q'_t$ are continuous in these parameters by an induction argument, Eq.~\eqref{eq:FirstMinimaxSPCA} also holds with $\mu'$, $\eps'$, $q'_t$ replaced by  $\mu$, $\eps$, $q_t$:
      \begin{align}
      \lim_{n\to\infty}\inf_{\obtheta_0\in\cuT(\eps,R)}\E\left\{\frac{\<\obtheta_0,\hbtheta^t\>}{\|\obtheta_0\|_2\|\hbtheta^t\|_2}\right\}\le \sqrt{\frac{V_{\pm}(q_t+\talpha)}{\mu^2\eps}}\, , \label{eq:SecondMinimaxSPCA}
      \end{align}
      Claims $(a)$ and $(b)$ follow by upper bounding the right-hand side of the last equation.
      
      First notice that $V_{\pm}(q)= \mu^4\eps^2\delta\, q+O(q^2)$ and hence Eqs.~\eqref{eq:SE-SPCA-1}, \eqref{eq:SE-SPCA-2}
      imply that, for any $\eta>0$ there exists $q_*>0$ such that, if $q_t+\talpha\le q_*$, then
      \begin{align}
        q_{t+1} \le (\mu^4\eps^2\delta+\eta)(q_t+\talpha)\, .
        \end{align}
        If $\mu^4\eps^2\delta<1$, choosing $\eta=(1-\mu^4\eps^2\delta)/2$, this inequality implies
        $q_t\le 2\talpha/(1-\mu^4\eps^2\delta)$, which proves claim $(a)$.

        For the second claim, we use the  bounds $e^{-\delta q\mu^2/2}\cosh(\mu\sqrt{\delta q} G)\ge 0$
          and $x/(1+x)\le x$ in Eq.~\eqref{eq:SE-SPCA-2} to get $q_t\le \oq_t$ for all $t$,  where $\oq_0=0$ and
  \begin{align}
    \oq_{t+1} & = F_0(\oq_t+\talpha)\, ,\;\;\; \;\;\;\; F_0(q) := \frac{\mu^2\eps^2}{1-\eps}\sinh(\mu^2\delta q)\, .
  \end{align}
  Further Eq.~\eqref{eq:SecondMinimaxSPCA} implies
  \begin{align}
    \lim_{n\to\infty}\inf_{\obtheta_0\in\cuT(\eps,R)}\E\left\{\frac{\<\obtheta_0,\hbtheta^t\>}{\|\obtheta_0\|_2\|\hbtheta^t\|_2}\right\}\le \sqrt{\frac{\oq_{t+1}}{\mu^2\eps}}\, .\label{eq:ThirdMinimaxSPCA}
  \end{align}
  Define $x_t := \mu^2\delta\oq_t$,  $a:=\mu^4\eps^2\delta/(1-\eps)$, $b := \mu^2\delta\talpha =(\delta/\eps)(\alpha/(1-\alpha))$.
  Then $x_t$ obeys the recursion
  \begin{align}
    x_{t+1}= a\sinh(x_t+b)\, .
  \end{align}
  Since $a= R^4\delta/(1-\eps)$, we know that $a<1/4$. Using the fact that $\sinh(u)\le 2 u$ for $u\le 1$, this implies
  $x_t\le b$ for all $t$ provided $b<1/2$. Subsitiuting this bound in Eq.~\eqref{eq:ThirdMinimaxSPCA}, we obtain the desired claim.
 
   \subsection{Proof of Corollary \ref{coro:PhaseRetrieval}}

   Consider first the case of a random vector $\btheta_0$ with i.i.d. entries $\theta_{0,i}\sim \mu_{\theta}$. Define, for $\Theta_0\sim\mu_{\theta}$,
    \begin{align}
      F_{\eps}(q) & := \E\big\{\E[\Theta_0|\sqrt{q}\Theta_0+G]^2\big\}\\
      & = e^{-q\mu^2}\mu^2\eps^2\E\left\{\frac{\sinh(\mu\sqrt{ q} G)^2}{1-\eps+\eps e^{- q\mu^2/2}
        \cosh(\mu\sqrt{q} G)}\right\}\, .
    \end{align}
    Setting $q_t = \tau_t^{-2}$, $\hq_t=\sigma_t^2$, and $\talpha = \alpha/(1-\alpha)$, and referring to Lemma  \ref{lem:posterior-to-score}, the state evolution recursion \eqref{bamp-se-hd-reg}  takes
    the form 
    \begin{align}
      \hq_t & = F_{\eps}(q_t+\talpha) \, ,\;\;\;\;   q_{t+1}= \delta\, H(\hq_t)
      \, ,\label{eq:SE-Phase-Retrieval}\\
   H(q)& := \E_{G_0,Y}\left[\left(
                           \frac{\E_{G_1}\partial_xp(Y| \sqrt{q}\, G_0+\sqrt{1-q}G_1)}{\E_{G_1}p(Y|\sqrt{q}\, G_0+\sqrt{1-q}G_1}\right)^2\right]\,.
    \end{align}
    Notice the change in factors $\delta$ with respect to Eq.~\eqref{bamp-se-hd-reg}, which is due
    to the different normalization of the design matrix.

    By the same argument used in the proof of Lemma \ref{lemma:SPCA}, Theorem \ref{thm:hd-reg-lower-bound} implies
    that, for any GFOM, with output $\hbtheta_t$, we have
    \begin{align}
      \lim_{n,p\to\infty}\E
    \frac{\<\obtheta_0,\hbtheta^t\>}{\|\obtheta_0\|_2\|\hbtheta^t\|_2}\le \sqrt{\hq_t}\, .
    \end{align}
    We next compute the first order Taylor-expansion of the iteration \eqref{eq:SE-Phase-Retrieval}, and obtain 
    $F_{\eps}(q) = q+O(q^2)$, $H(q) = q/\delta_{\sp}+O(q^2)$ (the first order Taylor expanson of $H(q)$ was already computed in
    \cite{mondelli2019fundamental}). As a consequence, for any $\eta>0$, there exists $\alpha_0$
    such that, if $\talpha<\alpha_0$, $q_t<\alpha_0$, then
    \begin{align*}
      q_{t+1}\le (\frac{\delta}{\delta_{\sp}}+\eta)(q_t+\talpha)\, .
    \end{align*}
    The claim follows by taking $\eta= \eta(\delta) :=  (\delta_{\sp}-\delta)/(2\delta_{\sp})$, whence $q_t\le \talpha/\eta(\delta)$
    for all $t$, provided $\talpha<\alpha_*:=\alpha_0
    \eta(\delta)$. The deterministic argument follows in the same way as Corollary \ref{coro:SPCA}.

\end{appendices}  

\end{document}